\documentclass{ecai} 

\usepackage{latexsym}
\usepackage{amssymb}
\usepackage{amsmath}
\usepackage{amsthm}
\usepackage{booktabs}
\usepackage{enumitem}
\usepackage{graphicx}
\usepackage{color}
\usepackage{float}
\usepackage{subcaption} 
\usepackage{xr}
\usepackage{hyperref}
\usepackage[monochrome]{xcolor}
\usepackage{tikz}
\usepackage{cancel}
\usepackage{algorithm}
\usepackage{algpseudocode}
\usepackage{natbib}

\newtheorem{theorem}{Theorem}

\newcommand{\BibTeX}{B\kern-.05em{\sc i\kern-.025em b}\kern-.08em\TeX}
\newcommand{\Mycomb}[2]{^{#1}C_{#2}}

\begin{document}


\begin{frontmatter}


\paperid{123} 


\title{Cooperative SGD with Dynamic Mixing Matrices}


\author[A,B]{\fnms{Soumya}~\snm{Sarkar}\orcid{0009-0001-8410-3124}}
\author[A]{\fnms{Shweta}~\snm{Jain}\orcid{0000-0002-2666-9058} \thanks{Corresponding Author. Email: shweta.jain@iitrpr.ac.in}}

\address[A]{Indian Institute of Technology Ropar}
\address[B]{Indian Institute of Science Bangalore}


\begin{abstract}
One of the most common methods to train machine learning algorithms today is the stochastic gradient descent (SGD). In a distributed setting, SGD-based algorithms have been shown to converge theoretically under specific circumstances. A substantial number of works in the distributed SGD setting assume a fixed topology for the edge devices. These papers also assume that the contribution of nodes to the global model is uniform. However, experiments have shown that such assumptions are suboptimal and a non uniform aggregation strategy coupled with a dynamically shifting topology and client selection can significantly improve the performance of such models. This paper details a unified framework that covers several Local-Update SGD-based distributed algorithms with dynamic topologies and provides improved or matching theoretical guarantees on convergence compared to existing work.
\end{abstract}

\end{frontmatter}


\section{Introduction}\label{sec:intro}
Stochastic Gradient Descent (SGD) is one of the most widely used training algorithms for machine learning models across various domains, primarily due to its effectiveness and simplicity. Beyond its practical advantages, SGD also serves as a fundamental framework for establishing theoretical convergence guarantees, reinforcing the reliability and stability of algorithms that incorporate it. These guarantees ensure that, given appropriate conditions, SGD will eventually converge to a stable solution. This makes it a valuable tool for researchers and practitioners, enabling them to apply SGD-based algorithms to real-world problems with greater confidence while understanding their limitations and behavior under different conditions.

State-of-the-art commercial models like GPT-4 are scaling their parameter counts by an order of magnitude each year, necessitating large-scale distributed clusters and specialized training paradigms. To efficiently process vast amounts of data, two primary approaches have been explored in the literature. The first approach, minibatch SGD, involves each client computing stochastic gradients on a large minibatch of its local data, aggregating gradients from several smaller local minibatches before applying a single SGD update. The second approach, known as Local SGD or Federated SGD \citep{Mcmahan2017, Stich2019}, has recently gained popularity as an effective distributed training method. Compared to minibatch SGD, it offers several advantages, including improved communication efficiency and robustness.

In Local SGD, each client independently runs SGD to optimize its local objective. During each communication round, a global server aggregates the model parameters obtained from all participating clients. Local SGD has gained popularity due to its ability to continuously improve models on local datasets while enhancing communication efficiency. Given these advantages, ongoing research focuses on establishing theoretical guarantees for Local SGD in distributed settings \citep{Gower2019,
Lin2018, Li2019}. These guarantees primarily aim to prove the convergence properties of various models proposed in the literature. For instance, \citet{Yu2019, Khaled2020} provide convergence guarantees for FedAvg, where, at each communication round, the server selects all clients, and the global model is obtained through simple averaging of their updates. Beyond simple averaging, more sophisticated aggregation techniques have been explored. For example, Elastic Averaging SGD \citep{Zhang2015} employs the Alternating Direction Method of Multipliers to improve parameter synchronization.

There have been attempts to create a generalized framework that covers several different variants of local-SGD, such as in \citet{Wang2021, Pu2021, Koloskova2020}. For example, \citet{Wang2021} captures different averaging techniques along with different communication strategies between the client in a distributed setting. However, \citet{Wang2021} takes two important assumptions in their convergence guarantees. First, they assume that if a client $i$ contributes by a factor of $w_{ij}$ to $j^{th}$ client model, then client $j$ also contributes $w_{ij}$ to $i^{th}$ client model. The contribution factors are represented by a mixing matrix ($W$), which is symmetric in \citet{Wang2021}. Their theoretical guarantees also work for only a static mixing matrix which does not change over a period of time. This paper relaxes these two assumptions and extend the theoretical guarantees to the setting of asymmetric and dynamic mixing matrices. We next show the need for such relaxations. \\

\begin{figure}[t]
    \centering
    \begin{subfigure}[b]{.22\textwidth}
        \includegraphics[width=\textwidth]{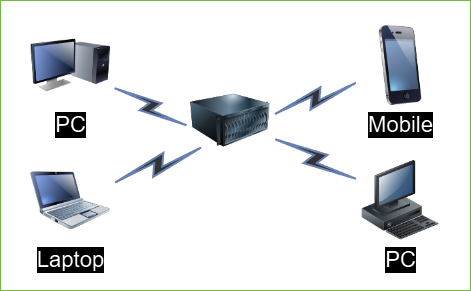}
        \subcaption{4-client FL Topology}\label{fig:W_ex_topo}
    \end{subfigure}\hfill
    \begin{subfigure}[b]{.22\textwidth}
        $\begin{bmatrix}
            \frac{\mathcal{D}_{1}}{\mathcal{D}} & \frac{\mathcal{D}_{1}}{\mathcal{D}} & \frac{\mathcal{D}_{1}}{\mathcal{D}} & \frac{\mathcal{D}_{1}}{\mathcal{D}}\\[0.2cm]
            \frac{\mathcal{D}_{2}}{\mathcal{D}} & \frac{\mathcal{D}_{2}}{\mathcal{D}} & \frac{\mathcal{D}_{2}}{\mathcal{D}} & \frac{\mathcal{D}_{2}}{\mathcal{D}}\\[0.2cm]
            \frac{\mathcal{D}_{3}}{\mathcal{D}} & \frac{\mathcal{D}_{3}}{\mathcal{D}} & \frac{\mathcal{D}_{3}}{\mathcal{D}} & \frac{\mathcal{D}_{3}}{\mathcal{D}}\\[0.2cm]
            \frac{\mathcal{D}_{4}}{\mathcal{D}} & \frac{\mathcal{D}_{4}}{\mathcal{D}} & \frac{\mathcal{D}_{4}}{\mathcal{D}} & \frac{\mathcal{D}_{4}}{\mathcal{D}}
        \end{bmatrix}$
        \subcaption{Corresponding W}
        \label{fig:W_ex_mat}
    \end{subfigure}
    \vspace{1em}
    \caption{An example along with it's corresponding mixing matrix.} 
    \vspace{2em}
\end{figure}


\noindent \textbf{Asymmetric Mixing Matrix}:\label{assym_MM}
The mixing matrix $W$ refers to a topology when aggregating the models. The element $w_{ij} \in W$ refers to the contribution of $i^{th}$ client on the model of client $j$. Usually, the models are aggregated using uniform averaging, where each client contributes equally, leading to a symmetric mixing matrix. However, consider FedAvg \citep{Mcmahan2017} aggregating rule:
\begin{align}\label{fedavg_mixing}
     F(x) = \sum_{i=1}^{m} \frac{|\mathcal{D}_{i}|}{|\mathcal{D}|}\mathbb{E}_{d\sim\mathcal{D}_{i}} \left[ f_{i}(x;d) \right]
\end{align}
Here, $\mathcal{D}_i$ and $\mathcal{D}$ represent the local dataset at client $i$ and global dataset respectively and $F_i(x) = \mathbb{E}_{d\sim\mathcal{D}_{i}}\left[ f_{i}(x;d) \right]$ denotes the local loss at client $i$. A simple example of the corresponding mixing matrix is shown in Figure (\ref{fig:W_ex_mat}) for the distributed topology shown in Figure (\ref{fig:W_ex_topo}). Here, $w_{ij} = \frac{|\mathcal{D}_i|}{|\mathcal{D}|} \ne w_{ji}$ when $|\mathcal{D}_i| \neq |\mathcal{D}_j|$.

There are many algorithms that utilize non-uniform aggregation strategies. These include \citet{Deng2021}, which considers training data quality; \citet{Hong2022} reduces information loss by manipulating the weight of each client; \citet{Bai2022} introduces parameter importance estimation; and FedDisco \citep{Tang2021} that rectify FedAvg aggregation strategy by non-uniform aggregation.\\ 

\noindent \textbf{Dynamic Mixing Matrices}:
Client selection is an effective way to reduce computational burden and communication overhead \citep{Li2022}, increase fairness \citep{Sultana2022}, increase model accuracy \citep{Lai2021}, as well as improve robustness \citep{Nguyen2020}. Building the mixing matrix when incorporating client selection usually involves assigning a contribution of zero to clients who are not selected. For example, in the scenario described in Figure \ref{fig:W_ex_topo}, if client $2$ and $4$ are selected in some round using a client selection algorithm like DivFL \citep{Balakrishnan2022} (which uses uniform model aggregation), we get the  mixing matrix $W$ with $w_{2,2}=w_{2,4}=w_{4,2}=w_{4,4} = 1/2$ and $w_{ij}=0$ for all other values of $i$ and $j$. 

In real-world scenarios, it may be imperative to do client selection after every round. This could be because the selected device may not be available for the entirety of the federated learning period \citep{Ruan2021}. Furthermore, doing client selection at every round gives improved performance over a static approach, especially when using algorithms that utilise client's gradients. Even with randomized client selection, where we select a fixed number of clients randomly, we get marked improvement, as demonstrated in Figure \ref{fig:csa_v_cso_MM}. In such a scenario, the mixing matrix keeps changing after each round, thereby showing that dynamic matrices are necessary.\\

\begin{figure}[t]
    \centering
    \begin{subfigure}[t]{.22\textwidth}
        \includegraphics[width=\textwidth]{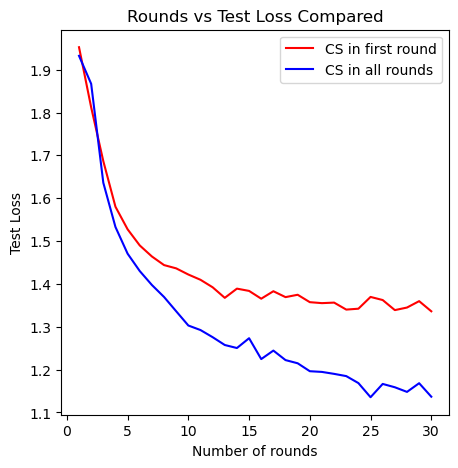}
        \subcaption{IID Scenario}\label{fig:csa_v_cso_MM_iid}
    \end{subfigure}\hfill
    \begin{subfigure}[t]{.22\textwidth}
        \includegraphics[width=\textwidth]{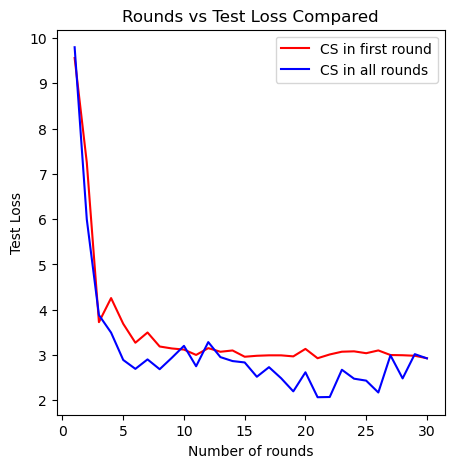}
        \subcaption{Non-IID Scenario}\label{fig:csa_v_cso_MM_niid}
    \end{subfigure}
    \vspace{1em}
    \caption{Comparison between two scenarios for 75 clients using FedAvg with CIFAR10 where 10 clients are randomly selected - one after every round (blue), and one only once at the start of federated learning training period (red). We see improvement in both the IID scenario (\ref{fig:csa_v_cso_MM_iid}) and the non-IID scenario (3-class distribution) (\ref{fig:csa_v_cso_MM_niid}), indicating using client selection at every round gives improved model.} \label{fig:csa_v_cso_MM}
    \vspace{2em}
\end{figure}

\noindent \textbf{Contributions}:
\begin{itemize}
    \item We provide a general framework for distributed SGD systems with non-convex loss functions, assyemtric and dynamic mixing matrices and client selection.
    \item We provide convergence guarantee of distributed systems in the most general framework with a time-varying topology.
    \item We compare with previous works and show them to be special cases of our framework. We also show that the convergence rates from previous works is matched or bettered by our work.
\end{itemize}


\section{Related Work}
There has been much work done when it comes to convergence analysis of distributed SGD systems. These are primarily based on two policies - asynchronous SGD or synchronous SGD. In asynchronous SGD policies such as the ones adopted by \citet{Dean2012, Chilimbi2014} each client independently communicates their calculated gradients to the central aggregating server. \citet{Lian2018} provides convergence guarantees for asynchronous SGD for non-convex functions but relies on specific topologies. \citet{Tosi2023}, on the other hand, gives a more general framework but only for the IID case. Even though asynchronous SGD is quite fast, it suffers from the stale gradient problem \citep{Tan2024} and does not scale very well \citep{Chen2017}. Therefore synchronous SGD \citep{Chen2017} is usually preferred over asynchronous SGD, and our analysis is, therefore, based on it.

Furthermore, there are many communication primitives that are used to synchronize model updation among clients. ALLREDUCE-SGD used in synchronous implementations \citep{Goyal2018} is a popular exact averaging technique for model weights to synchronize. PUSHSUM \citep{Kempe2003} has seen a recent rise in popularity as it offers some advantages in communication-constrained settings, although it tends to inject additional noise in the average gradient estimate \citep{Assran2019}. One framework which provides a novel analysis of SGD with PUSHSUM is by \citet{Assran2019}. They utilise dynamic column-stochastic mixing matrices, but use only some particular sequence of mixing matrices. Our work generalises to any SGD-based scenario (centralized and decentralized) as well as any sequence of mixing matrices. We also match their results as explained in section \ref{sec:Analysis}. Gossip algorithms such as \citet{Hu2019} also utilise the PUSHSUM primitive. In this work, we limit our analysis to the ALLREDUCE primitive
.

\citet{Pu2021} in their work utilize another primitive known as PUSH-PULL which uses asymmetric mixing matrices for achieving average consensus. However, their analysis is limited to convex loss functions. Furthermore, not only do they not consider dynamic mixing matrices, but the matrices themselves are constructed in a particular manner, reducing the generalizability of their framework.

In a federated learning context, \citet{Li2019} provide some convergence guarantees for vanilla FedAvg, but they assume convexity of the loss function.  Since then, there have been various methods that have tried to improve on FedAvg such as \citet{Machler2021,Wang2020,Li2020}, with each work providing convergence guarantees for their methods. \citet{Wang2022} generalizes centralized FedAvg convergence guarantees while incorporating client selection as well. On the other hand, there have been several works which have tried to argue that decentralized periodic SGD (D-PSGD) gives better results than their centralized counterparts. \citet{Lian2017}, for example, provide a framework for D-PSGD , showing it's superiority over centralized algorithms. They, however, assume static symmetric doubly-stochastic mixing matrices with a uniform averaging rule which, as seen in Section \ref{assym_MM}, is not a realistic or optimal choice. \citet{Koloskova2020} expands on this with a framework for decentralized algorithms which provides convergence guarantees for dynamic topologies by using some additional assumptions. However, they too assume symmetric and doubly stochastic mixing matrices.

Our work is based on \citet{Wang2021}, who provide a generalized framework for both centralized and decentralized distributed SGD algorithms. Our work relaxes their assumptions on the nature of the mixing matrix by making it both assymetric (only column/row stochastic) and dynamic, while also incorporating client selection.


\section{Notations}
We express the $m$-element column vector $\left[ 1, 1, 1, ... \right]^{T}$ as \textbf{1} and define the square matrix \boldmath $J = {11^{T}}/{1^{T}1}$ \unboldmath. \boldmath $J$ \unboldmath and \boldmath $I$ \unboldmath are of size $(m+v)\times (m+v)$ unless specified otherwise. We also utilise three norms - the operator norm denoted by \boldmath $||\cdot||_{op}$\unboldmath, the frobenius norm denoted by \boldmath $||\cdot||_{F}$ \unboldmath as well as the regular $l_{2}$ norm denoted by \boldmath $||\cdot||$ \unboldmath. The operator and frobenius norms are matrix norms defined as:
\begin{align*}
    ||A||_{F} = & \sqrt{Tr(A^{T}A)} = \sqrt{Tr(AA^{T})} = \sqrt{\sum_{i}^{m}\sum_{j}^{n}|a_{ij}|^{2}} \\
    ||A||_{op} = & \sup_{x\neq 0}\frac{||Ax||}{||x||} =  \sqrt{\lambda_{max}(A^{T}A)} = \sigma_{max}(A)
\end{align*}
Here $Tr()$ refers to the trace of a matrix; $\lambda_{max}$ and $\sigma_{max}$ refers to the largest eigenvalue and largest singular value of a matrix respectively.


\section{Problem Formulation}
Consider $m$ clients with some topology, of which $c$ fraction of clients are selected at each time. Let the selected set at iteration $k$ be denoted as $C_{k}$ and the mixing matrix is denoted by $W$ with $w_{ij}$ denoting the contribution of $i^{th}$ node to the aggregate for the model at node $j$. The value $w_{ij}$ will be $0$ when there is no connection between workers $i$ and $j$. The model parameters are denoted by $x \in \mathbb{R}^{d}$. We want to minimize the global empirical risk as follows:
\begin{align}
    \min_{x \in \mathbb{R}^{d}} F(x) = \min_{x \in \mathbb{R}^{d}} \frac{1}{m} \sum_{i=1}^{m} \mathbb{E}_{d\sim\mathcal{D}_{i}} \left[ f_{i}(x;d) \right]
\end{align}
where $f_{i}(\cdot)$, the loss function at each client, could be non-convex in nature. Here $\mathcal{D}_{i}$ represents the dataset distribution for client $i$ and let $F_i(x) = \mathbb{E}_{d\sim\mathcal{D}_{i}}\left[ f_{i}(x;d) \right]$ denotes the local loss at client $i$.


In a fully synchronous local SGD, after initializing with a randomized global model, we update the global model using the updates from all clients after each mini batch. The following update rule is used at the central server:
\begin{align}
    x_{k+1} = x_k  - \eta \left[ \frac{1}{m} \sum_{j=1}^{m} g_{i}(x_{k};\xi_{k}^{(i)}) \right]
\end{align}
where $\xi_{k}^{(i)}$ represents the mini-batch of $i^{th}$ client at the $k^{th}$ iteration, and $g_{i}(x,\xi) = \frac{1}{|\xi|}\sum_{s \in \xi}\nabla f_{i}(x;s)$ represent the average gradient of client $i$ with batch $\xi$ and model parameter $x$. For simplicity we will refer to this as $g_{i}(x)$ henceforth. For the above case, the contribution of one client to another is the same for all cases with $w_{ij}=\frac{1}{m}$.


If the models are aggregated after every $\tau$ iterations/mini batches (the communication period), then the update rule can be modified as:
\begin{align}
    x_{k+1}^{(i)} = 
    \begin{cases}
        \frac{1}{m} \sum_{j=1}^{m} \left[ x_{k}^{(j)} -  \eta g_{i}(x_{k}^{(j)}) \right] & {k\text{ mod }\tau=0}\\
         x_{k}^{(i)} -  \eta g_{i}(x_{k}^{(i)}) &\text{otherwise}
    \end{cases}
\end{align}
This is described in existing literature \citep{Wang2021} as Periodic Simple-Averaging SGD (PSASGD). Here too, $w_{ij}=\frac{1}{m}$ when aggregating.

If the models are updated in a decentralized manner after each mini-batch, the update rule is provided by:
\begin{align}
    x_{k+1}^{(i)} = \sum_{j=1}^{m} w_{ji} \left[ x_{k}^{(j)} - \eta g_{i}(x_{k}^{(i)}) \right]
\end{align}
This strategy is referred to as Decentralized SGD (D-PSGD) \citep{Wang2021} .

We also have Elastic Averaging SGD (EASGD) \citep{Zhang2015}. Here we utilize an auxillary variable $z_{k}$ which acts as an anchor during updation of local models. The update rule is given by:
\begin{align}
    &x_{k+1}^{(i)} = 
    \begin{cases}
         x_{k}^{(i)} -  \eta g_{i}(x_{k}^{(i)}) -  \alpha(x_{k}^{(i)} - z_{k}) \quad & {k\text{ mod }\tau=0}\\
         x_{k}^{(i)} -  \eta g_{i}(x_{k}^{(i)}) \quad &\text{otherwise}
    \end{cases}\\
    &z_{k+1} = 
    \begin{cases}
         (1-m\alpha)z_{k} + m\alpha \bar{x_{k}} \qquad \qquad & {k\text{ mod }\tau=0}\\
         z_{k} \qquad \qquad \qquad &\text{otherwise}
    \end{cases}
\end{align}
where $\bar{x_{k}} = \frac{1}{m}\sum_{i=1}^{m}x_{k}^{(i)}$.


\section{Proposed Framework: Cooperative SGD with Dynamic Mixing Matrices}
We take inspiration from the work by \citet{Wang2021} and extend their local SGD framework for distributed systems (named as Cooperative SGD) to the scenario of dynamic mixing matrices which need not be symmetric, incorporating client selection. This allows us to provide convergence guarantees even when we have disparity between the different clients and bias our model aggregation strategy to tackle the lack of uniformity. For example, relaxation of this assumption in a FedAvg setting implies that the dataset sizes can be different across different clients. We are able to provide a unified convergence analysis as is presented in Sections \ref{convergence_proof} and \ref{convergence_proof_niid}. Although we do not discuss about the elastic averaging SGD (EASGD) method explicitly, there exists a provision for use of auxillary variables in our work and would follow a similar analysis as \citet{Wang2021}.

We generalize the update rule for any method. At iteration $k$ each of the clients have their own models $x_{k}^{(1)}, x_{k}^{(2)}, ... x_{k}^{(m)} \in \mathbb{R}^{d}$. They may also use some auxillary variables $z_{k}^{(1)}, z_{k}^{(2)}, ... z_{k}^{(v)}$. The $w_{ij}^{(k)}$ form the mixing matrix $W_{k}$. For model averaging, a more general time-varying matrix $\mathcal{S}_{k}$ of size $(m+v)\times(m+v)$ is  defined as:
\begin{align*}
    \mathcal{S}_{k} =
    \begin{cases}
        W_{k} & k\text{ mod }\tau=0\\
        I & \text{Otherwise}
    \end{cases}
\end{align*}
We also define two matrices $X_{k}, G_{k} \in \mathbb{R}^{d\times(m+v)}$ as below:
\begin{align*}
    X_{k} = \left[ x_{k}^{(1)},...x_{k}^{(m)}, z_{k}^{(1)},..z_{k}^{(v)} \right], 
    G_{k} = \left[ g_{1}(x_{k}^{(1)}),...g_{m}(x_{k}^{(m)}), 0,..0 \right]
\end{align*}

For the above two matrices, we assume the models of clients not selected at iteration $k$ are zero i.e, $x_{k}^{(j)} = 0 \text{ and } g_{j}(x_{k}^{(j)}) = 0 \text{ } \forall j \notin C_{k}$. What this implies is that we assume that there is no computation done for the models not selected in a particular round. This is an efficient approach, as such computation would be an overhead which is better avoided. Even if computation is needed for client selection methods, not transmitting these values renders them as zero-valued.

A thing to note here is that when client selection is considered, it is necessary to ensure that both $X_{k}$ and $G_{k}$ have zero values for the unselected models, as well as to ensure the contribution of these models in the mixing matrix $W$ is zero. The former is necessary to ensure that we eventually get a smaller convergence error bound \footnote{It can be derived along similar lines to our current proof.}. The latter is compulsory as not doing so might break column stochasticity.

Generalizing the update rules of all the methods discussed before (as well as others), we write the update rule as:
\begin{align}\label{eq:update_rule}
    X_{k+1} = \left( X_{k} - \eta G_{k} \right)S_{k}^{T}
\end{align}
Here we consider the transpose of $S_{k}$ in contrast with the version proposed in \citet{Wang2021} so that we can define $W_{k}$ as a column-stochastic matrix, but use it in a row-stochastic manner. FedAvg \citep{Mcmahan2017} uses a column stochastic matrix, for example, as can be seen in Figure \ref{fig:W_ex_mat}. If $W_{k}$ is defined as row-stochastic, then we can directly use $S_{k}$. We also define the following for all mixing matrices $W_{k}$:
\begin{align*}
    W_{s}^{T}W_{s+1}^{T}W_{s+2}^{T}...W_{k}^{T} = \Phi_{s,k}^{T}
\end{align*}

\subsection{Update Rule for the Averaged Model}
After multiplying $1_{m+v}/(m+v)$  on both sides in \eqref{eq:update_rule}, $S_{k}^{T}$ disappears, and we get the update rule as:
\begin{align*}
    X_{k+1} \frac{1_{m+v}}{m+v}= X_{k}\frac{1_{m+v}}{m+v} - \eta G_{k}\frac{1_{m+v}}{m+v} 
\end{align*}
Taking the following quantities to represent average model and average learning rate respectively,
\begin{align*}
    u_{k} = X_{k}\frac{1_{m+v}}{m+v},\text{ } \eta_{eff} = \frac{cm}{m+v}\eta 
\end{align*}
We get the updated learning rule as:
\begin{align}\label{eq:update_rule_average}
    u_{k+1} = u_{k} - \eta_{eff}\left[ \frac{1}{cm}\sum_{i \in C_{k}} g_{i}(x_{k}^{(i)}) \right]
\end{align}
We try to find the averaged squared gradient norm as an indicator of convergence, as described by \citet{Bottou2018}. An $\epsilon$-suboptimal solution is attained after $K$ iterations if
\begin{align}
    \mathbb{E} \left[ \frac{1}{K}\sum_{k=1}^{K}||\nabla F(u_{k}) ||^{2} \right] \leq \epsilon
\end{align}
This guarantees convergence of the algorithm to some good minima. We next provide convergence guarantees of our framework.


\section{Unified Convergence Analysis for IID Scenario}\label{convergence_proof}
We define IID scenario as when the data distribution across all clients are nearly identical in nature, and all data samples are independent of each other. We consider the following assumptions that are standard across the literature \citep{Wang2021, Balakrishnan2022, Li2020, Arjevani2023}. 

\subsection{Assumptions}\label{asmp:iid}
\begin{enumerate}
    \item \textbf{Smoothness}: $||\nabla F_{i}(x) - \nabla F_{i}(y)|| \leq L||x-y||$
    \item \textbf{Lower Bounded}: $F(x) \geq F_{inf}$
    \item \textbf{Unbiased Gradients}: $\mathbb{E}_{\xi|x}[g_{i}(x)] = \nabla F(x)$
    \item \textbf{Bounded Variance}: $\mathbb{E}_{\xi|x}||g_{i}(x)-\nabla F(x)||^{2} \leq \sigma^{2}$\newline
    where $\sigma$ is a non-negative constant and in inverse proportion to mini-batch size
    \item \textbf{Mixing Matrix}:
         $W_{k}^{T}1_{m+v} = 1_{m+v}$ (column stochasticity)
    \item For all rounds, fraction of clients selected $c$ is fixed.
\end{enumerate}
Furthermore, we relax assumption 5 from \citet{Wang2021}, which used doubly stochastic (both row and column stochastic) and symmetric mixing matrices and consider only row-stochastic matrix. 
Having a fixed number of clients every round is a simplifying assumption which reduces compute power by removing the need to recalculate the number of clients needed each round, and is also standard in existing literature \citep{Reisizadeh2020, Luo2023}.

\subsection{Convergence Bounds for IID Case} 

    \begin{theorem}[]\label{thm:iid-conv}
        For algorithm $\mathcal{A}(\tau, W_{k}, v)$ that satisfies the update rule in equation \ref{eq:update_rule}, let the total number of iterations be $K$ and out of $m$ clients a fraction $c$ is selected. Under assumptions 1-6, if learning rate satisfies
            $\eta_{eff}L \leq 1$
        and all local models are initialized at the same point $u_{1}$, then the average squared gradient norm after K iterations is $\epsilon_{IID}$-suboptimal, where
        \begin{align*}
             \epsilon_{IID} = 
             4\left[
             \begin{aligned}
                 &\frac{2[ F(u_{1}) - F_{inf}]}{\eta_{eff}K} + \frac{\eta_{eff}L\sigma^{2}}{cm}\nonumber+ \frac{\delta L^{2}}{cm}\left|\left|X_{1}\right|\right|^{2}_{F} \\&+  \eta^{2}\sigma^{2}L^{2}\delta (K-1)
             \end{aligned}
             \right]
        \end{align*}
        where $X_{1} = \left[ x_{1}^{(1)},...x_{1}^{(m)}, z_{1}^{(1)},..z_{1}^{(v)} \right]$ represents the $m$ initial models $x_{i}^{(p)}$ and $v$ auxillary variables $z_{i}^{(p)}$, $\delta \in [0, c(m+v-1)]$, and we have
        \begin{align*}
            0 \leq P \leq \min\left(\frac{1}{6},\frac{1}{6L^{2}+3}, \frac{c}{6L^{2}}\right)
        \end{align*}
        where $P = \eta^{2}\delta\tau
            \left[2\tau S_{series} + (\tau-1)(1 +K/\tau)\right]$
         and $S_{series} = (K/\tau - 1) \left[ 2  +\frac{K}{2\tau}\right]$.
    \end{theorem}

\begin{proof} Check supplementary material (Section \ref{IID-Theorem1}). \end{proof}

\subsection{Optimized learning rate, IID case}

    \textbf{Corollary 1.} \textit{For algorithm $\mathcal{A}(\tau, W_{k}, v)$ that satisfies the update rule in equation \ref{eq:update_rule}, under assumptions 1-5, if learning rate is $\eta = \frac{m+v}{Lcm}\sqrt{\frac{cm}{K^{2}}}$, all local models are initialized as zero and hyperparameters satisfy $K \geq \mathcal{O}\left(\max(\tau, \delta m \sqrt{\frac{m}{c}}) \right)$, then the average squared gradient norm after K iterations is bounded by
    \begin{align}
        \mathbb{E}\left[ \frac{1}{K}\sum_{k=1}^{K}||\nabla F(u_{k})||^{2} \right] &\leq \mathcal{O}\left( \frac{1}{\sqrt{cm}} \right) + \mathcal{O}\left( \frac{m\delta}{cK} \right)
    \end{align}\label{thm:cor1}}

\begin{proof} Check supplementary material (Section \ref{Corollary1}) \end{proof}

\subsection{Some Comments About the IID Case}

We provide a comparison of our convergence guarantees with \citet{Wang2021} (ignoring leading factors) in Table \ref{tab:wang_comp}.

\noindent \textbf{Initialization Error}: As we can see, our convergence analysis leads to an additional $\frac{\delta L^{2}}{Kcm}\left|\left|X_{1}\right|\right|^{2}_{F}$ term in both the IID and non-IID case. This is an error term that arises due to the values of initial model weights of all the clients, which we refer to as \textit{initialization error}. Unlike \citet{Wang2021}, this term persists in our case due to the lack of symmetry for the mixing matrix. Although it can be said to be a result of non-uniform aggregation strategies, we note that because of this term the error bound is directly proportional to the size of the initial models. This is in keeping with common wisdom which states that it is better to initialize deep learning models with small parameter values to avoid the problem of gradient explosion. However, another contributing factor here is $F(u_{1})$, the starting point for the average model on the loss surface in the first iteration. If we use tiny parameter values, this term might get too large, negating the benefits achieved from reducing the initialization error. So, there is a tradeoff involved when it comes to model initialization. This tradeoff is showcased in Figure \ref{fig:initcomp_psasgd} for PSASGD.\\

\begin{table}[t]
    \centering
    \begin{tabular}{|c|c|c|}
    \hline
         IID & \citet{Wang2021} & $\frac{2[ F(u_{1}) - F_{inf}]}{\eta_{eff}K} + \frac{\eta_{eff}L\sigma^{2}}{m}   $\\ & & $+\eta^{2}\sigma^{2}L^{2}\left[ \frac{1+\varsigma^{2}}{1-\varsigma^{2}}\tau -1 \right] $ \\
    \hline
         IID & Ours & 
             $
                 \frac{2[ F(u_{1}) - F_{inf}]}{\eta_{eff}K} + \frac{\eta_{eff}L\sigma^{2}}{cm} $\\& & $+   \eta^{2}\sigma^{2}L^{2}\delta (K-1)$ \\ & & $+\frac{\delta L^{2}}{Kcm}\left|\left|X_{1}\right|\right|^{2}_{F}
                $\\[0.2em]
    \hline
         Non-IID & \citet{Wang2021} &  $\epsilon_{IID} +C_{2}\kappa^{2}$\\
    \hline
         Non-IID & Ours & $\epsilon_{IID} + 12PL^{2}\kappa^{2}$\\
    \hline
    \end{tabular}
    \caption{Error term $\epsilon$ Comparison (Excluding leading factors)}
    \label{tab:wang_comp}
\end{table}



\noindent \textbf{Aggregation Strategy}: We also get a $\delta(K-1)$ term in place of the eigenvalue-based constant $\frac{1+\varsigma^{2}}{1-\varsigma^{2}}\tau$ where $\varsigma = \max(|\lambda_{2}|, |\lambda_{m}|) < 1$ is the second largest eigenvalue of mixing matrix $W_{k}$. The term $\delta$ is dependent on $W_{k}$ and is equal to $c(m+v-1)(1-(m+v)^{2}t_{1}t_{2})$, where $t_{1}t_{2}$ refers to the smallest multiplicative pair among all such pairs where both elements come from the same column of $W_{k}$.
When we have $\delta \in [0, 1]$, we can replicate the scenarios covered by \citet{Wang2021} and provide analysis for them. We can prove that in such a scenario, if we satisfy $\tau > \frac{1-\varsigma^{2}}{2\varsigma^{2}}$, then we get a tighter convergence bound. This is a very easy condition to satisfy - when $\tau = 1$, we need $\varsigma > \frac{1}{\sqrt{3}}$ to satisfy this criterion. This is quite a small value for what is the second largest eigenvalue by magnitude. For larger values of $\tau$ we get an even lower criterion for $\varsigma$. One easy way to achieve such a $\delta$ value is to simply ensure\footnote{This condition comes from keeping $c(m+v-1)(1-(m+v)^{2}t_{1}t_{2}) \leq 1$} product of the two smallest values in the mixing matrix is greater than $\frac{c}{m+v}$. If the number of clients is large, this can be easily achieved. 
In a more general setting, we get a higher bound on the $\delta$ value, and hence the overall error, due to the fact that various aggregation strategies are possible under our framework, and hence we are unable to say much about them. If we adopt a non-uniform aggregation strategy assigning smaller weights to some clients' models and larger weights to others, we obtain a higher value of $\delta$ compared to a strategy closer to uniform aggregation. In the extreme case where we simply ignore some clients (even if they are selected for the round), we get $\delta = c(m+v-1)$. In general, we see that the convergence error has a direct relation to $\delta$. This implies that aggregation strategies closer to the uniform scenario would give better results. If a fully uniform averaging strategy is used where we average all selected models in the same proportion $\frac{1}{cm}$ to form the global model, then $\delta=0$ leading to a lower error value as per Theorem 1.\\


\noindent \textbf{Client Selection}:
Another thing to note is that there is an inverse relationship of convergence with the fraction of clients selected $c$. This is in keeping with the observation that more number of clients leads to a better convergence. This, however, may lead to a larger value of $\delta$, implying that there is a tradeoff involved for the value of c. This tradeoff is in keeping with previous works which imply that not selecting all clients may lead to better convergence, provided we don't select too few \citep{Lai2021}. We also get a lower bound on the fraction of clients that need to be selected. This is useful as it indicates the minimum number of clients needed by the global model to converge.\\

\noindent \textbf{Dependence on $\tau$}:
Our analysis also shows that the convergence error is independent from the communication period $\tau$ for sufficiently large values of $\delta$. This implies that the convergence error would not change even if we increase the number of local iterations before aggregating. Thus, we can reduce the communication overhead without affecting convergence. We validate this in Section \ref{sec:exp} (Figure \ref{fig:taucomp_psasgd}).



\section{Unified Convergence Analysis for Non-IID Scenario}\label{convergence_proof_niid}

\subsection{Assumptions}
For non-IID case, we have all the basic assumptions that were considered in IID case with a few changes. We consider the fact that the stochastic gradient at each client is not an unbiased estimator of the global gradient, which leads to changes in two assumptions (Assumption 3 and 4) and the addition of one assumption. The changed assumption along with new assumptions are:
\begin{enumerate}
    \item[3.] \textbf{Unbiased Gradients}: $\mathbb{E}_{\xi|x}[g_{i}(x)] = \nabla F_{i}(x)$
    \item[4.] \textbf{Bounded Variance}: $\mathbb{E}_{\xi|x}||g_{i}(x)-\nabla F_{i}(x)||^{2} \leq  \sigma^{2}$\newline
    where $\sigma$ is some non-negative constant and in inverse proportion to batch size
    \item[7.] \textbf{Bounded Dissimilarities}: $\frac{1}{m}\sum_{i=1}^{m}||\nabla F_{i}(x)-\nabla F(x)||^{2} \leq  \kappa^{2}$\newline
    where $\kappa$ is some non-negative constant
\end{enumerate}

\subsection{Convergence Bounds for Non-IID Case} 

    \begin{theorem}[]\label{thm:niid-conv}
        For algorithm $\mathcal{A}(\tau, W_{k}, v)$ that satisfies the update rule in equation \ref{eq:update_rule}, let the total number of iterations be $K$, and out of $m$ clients a fraction $c$ is selected. Under assumptions 1-7 for non-IID scenario, if learning rate satisfies
            $\eta_{eff}L \leq 1$
        and all local models are initialized at the same point $u_{1}$, then the average squared gradient norm after K iterations is $\epsilon_{NIID}$-suboptimal, where
        \begin{align}
            \epsilon_{NIID} =  \epsilon_{IID} +  12PL^{2}\kappa^{2}
        \end{align}
        where $\epsilon_{IID}$ is the $\epsilon$ bound in the IID case, and we have:
        \begin{align}
            0 \leq P \leq \min\left(\frac{1}{6},\frac{1}{6L^{2}+3}, \frac{c}{6L^{2}}\right)
        \end{align}
        where $P = \eta^{2}\delta\tau
            \left[2\tau S_{series} + (\tau-1)(1 +K/\tau)\right]$, $S_{series} = (K/\tau - 1) \left[ 2  +\frac{K}{2\tau}\right]$ and $\delta \in [0, c(m+v-1)]$.
    \end{theorem}
    

\begin{proof} Check supplementary material (Section \ref{NIID-Theorem2}) \end{proof}

\subsection{Some Comments About the Non-IID Case}
We see that there is an additional $12PL^{2}\kappa^{2}$ term compared to the IID case  that arises due to fact that the global average gradient is different from the average gradient of each worker.  Considering the bounds on the value of $P$, this term acts like a constant term independent of the values of different parameters. As with the IID case, having a mixing strategy closer to uniform aggregation leads to a tighter bound. If a fully uniform mixing strategy is used, then we get $P=0$ and, thus, almost the same bounds as in the IID case. We also get a lower bound on the fraction of clients $c$ that need to be selected, as in the IID case.


\subsection{Comparison with \citet{Wang2021}} \label{sec:novelty}

The three core extensions of Asymmetric aggregation, Dynamic topology, and client selection on the work done by \citet{Wang2021} lead to non-trivial changes in the proof, which we explain below.
\begin{enumerate}
    \item The loss of symmetry in the mixing matrix leads to a reworked Lemma 8 (see Supplementary), with non-diagonalizable $W$. This prevents us from using the eigenvalue approach of \citet{Wang2021}. Instead, we define and utilize the $\delta$ term in terms of other properties of the matrix and then optimize it.
    \item Using a dynamic topology prevents the aggregation of the mixing matrices as some power of $W$ (for example, the sequence of mixing matrices till $k$ iterations can be denoted as $W^{k}$), as was done in \citet{Wang2021}. Instead, we treat them as separate and non-interchangeable (for example, $W_{1}W_{2} \neq W_{1}^{2} \neq W_{2}W_{1}$).
    \item Incorporating client selection necessitates modifications in assumptions and definitions while maintaining consistency. We also provide lower bounds on the number of clients selected.
\end{enumerate}


\section{Analysis of Standard Algorithms as Special Cases and Comparison with Existing Works} \label{sec:Analysis}
We now show how our convergence guarantees improve on existing bounds in some special cases.

\subsection{PSASGD Analysis}
We start by analysing PSASGD. This can be done by setting $W_{j}=J \implies \delta=0$ and $v=0$. This leads to the following result for both IID and Non-IID case:
\begin{align} \label{eq:14}
    \mathbb{E}\left[ \frac{1}{K}\sum_{k=1}^{K}||\nabla F(u_{k})||^{2} \right] &\leq \frac{2[ F(u_{1}) - F_{inf}]}{\eta_{eff} K} + \frac{\eta_{eff} L\sigma^{2}}{cm}
\end{align}
If we set the learning rate as $\eta = \frac{1}{Lc}\sqrt{\frac{cm}{K}}$, then we get,
\begin{align} \label{eq:15}
    \mathbb{E}\left[ \frac{1}{K}\sum_{k=1}^{K}||\nabla F(u_{k})||^{2} \right] &\leq \frac{2L[ F(u_{1}) - F_{inf}]}{\sqrt{cmK}} + \frac{\sigma^{2}}{\sqrt{cmK}} \nonumber\\
    &= \mathcal{O}\left( \frac{1}{\sqrt{cmK}} \right) 
\end{align}
We get the simple criteria of  $K>\mathcal{O}(\max(\tau, cm))$ to achieve a convergence rate of $\mathcal{O}\left( \frac{1}{\sqrt{cmK}} \right) $. This improves on the criteria of $K > m^{3}\tau^{2}$ for the same provided by \citet{Wang2021}.

If we use dynamic and unsymmetric mixing matrices then $\delta \neq 0$. Initializing models to nearly zero, in the IID scenario gives us,
\begin{align} \label{eq:16}
    &\mathbb{E}\left[ \frac{1}{K}\sum_{k=1}^{K}||\nabla F(u_{k})||^{2} \right] \leq 4\left[
    \begin{aligned}
        \frac{2L[ F(u_{1}) - F_{inf}]}{\eta_{eff}K} + \frac{\eta_{eff}L\sigma^{2}}{cm}\\+  \eta^{2}\sigma^{2}L^{2}\delta (K-1)
    \end{aligned}\right]
\end{align}
If $\delta \in (0,1]$, then we get a better convergence criteria than \citet{Wang2021} when $\tau > \frac{1-\varsigma^{2}}{2\varsigma^{2}}$ as discussed before. In such a scenario for $\delta$, we can reduce the convergence criteria to
\begin{align} \label{eq:17}
    &\mathbb{E}\left[ \frac{1}{K}\sum_{k=1}^{K}||\nabla F(u_{k})||^{2} \right] \leq 4\left[
    \begin{aligned}
        \frac{2L[ F(u_{1}) - F_{inf}]}{\eta_{eff}K} + \frac{\eta_{eff}L\sigma^{2}}{cm}\\+  \eta^{2}\sigma^{2}L^{2}\left( \frac{1-\varsigma^{2}}{1+\varsigma^{2}}\tau - 1 \right)
    \end{aligned}\right]
\end{align}
Putting $\eta = \frac{1}{Lc}\sqrt{\frac{cm}{K}}$ in this case yields,
\begin{align} \label{eq:18}
    &\mathbb{E}\left[ \frac{1}{K}\sum_{k=1}^{K}||\nabla F(u_{k})||^{2} \right] \leq \mathcal{O}\left( \frac{1}{\sqrt{cmK}}\right) + \mathcal{O}\left( \frac{m\tau}{Kc} \right)
\end{align}

To get a better convergence than $\mathcal{O}\left( \frac{1}{\sqrt{cmK}} \right) $ in this case, we need $K > \mathcal{O}\left( \frac{m^{3}\tau^{2}}{c}\right)$. This matches the criterion of \citet{Wang2021}.

For non-IID scenario, we get,
\begin{align} \label{eq:19}
    &\mathbb{E}\left[ \frac{1}{K}\sum_{k=1}^{K}||\nabla F(u_{k})||^{2} \right] \leq 4\left[
    \begin{aligned}
        \frac{2L[ F(u_{1}) - F_{inf}]}{\eta_{eff}K} + \frac{\eta_{eff}L\sigma^{2}}{cm}\\+  \eta^{2}\sigma^{2}L^{2}\delta (K-1) + 12PL^{2}\kappa^{2}
    \end{aligned}\right]
\end{align}
When $\delta \in (0,1]$ and $\eta = \frac{1}{Lc}\sqrt{\frac{cm}{K}}$, a similar analysis as the IID case leads to equation \ref{eq:18}.
 We achieve a convergence rate of $\mathcal{O}\left( \frac{1}{\sqrt{cmK}} \right) $ with $K > \mathcal{O}\left( \frac{m^{3}\tau^{2}}{c}\right)$ in the non-IID case, which matches the $K > \mathcal{O}\left( m^{3}\tau^{2} \right)$ criteria provided by \citet{Wang2021} \footnote{In \citet{Wang2021} they consider the effect of the additional term even though, effectively, it is a constant. In our above analysis we ignore the constant for both our and their work.}.\\

When $\delta > 1$ in both the IID and non-IID case, putting $\eta = \frac{1}{Lc}\sqrt{\frac{cm}{K}}$ gives us,
\begin{align} \label{eq:21}
    &\mathbb{E}\left[ \frac{1}{K}\sum_{k=1}^{K}||\nabla F(u_{k})||^{2} \right] \leq \mathcal{O}\left( \frac{\delta m}{c} \right)
\end{align}
We see that in such a scenario, we start to get convergence error bound $\epsilon > 1$. So it is highly advisable to not resort to highly non-uniform aggregation strategies where we heavily rely on some clients while ignoring others. If we use smaller learning rates, such as in corollary 1, we can avoid this issue and get (possibly) sublinear error. Even then, getting a convergence rate of $\mathcal{O}\left( \frac{1}{\sqrt{cmK}} \right) $ is not possible.

\subsection{Fully Synchronous SGD Analysis}
If we put $\tau = 1$ for fully synchronous SGD in IID scenario with uniform aggregation ($W_{j} = J \implies \delta=0$), and $\eta = \frac{1}{Lc}\sqrt{\frac{cm}{K}}$, we get equation \ref{eq:15}. This matches the optimisation error bound found by \citet{Ghadimi2013}, \citet{Arjevani2023} and \citet{Wang2021}, with an improved  $K>\mathcal{O}(\max(\tau, cm))$ criterion to achieve the same. Since $\delta=0$ in such a scenario, this implies that the extra term in the error bound for non-IID scenario $12PL^{2}\kappa^{2}=0$, and hence the order of bounds remain the same in the non-IID case.

If we consider a scenario such as in \citet{Assran2019} where dynamic column-stochastic mixing matrices are used, then  with $\tau = 1$ and $\eta = \frac{1}{Lc}\sqrt{\frac{cm}{K}}$ we get the same results as in equations \ref{eq:16} and \ref{eq:19}  for the IID and non-IID cases respectively. The analysis with respect to $\delta$ values follows the PSASGD case with $\tau = 1$.

\subsection{D-PSGD Analysis}
For D-PSGD, we set $\tau = 1$, $W_{j} = W$ and $v=0$ in equations \ref{eq:16} and \ref{eq:19}  for the IID and non-IID cases respectively. When $\delta \in (0, 1]$, as is the case covered by \citet{Wang2021}, we get the same bounds as equation \ref{eq:18} for IID and non-IID scenarios, with the same criterion (putting $\tau = 1$).  This matches \citet{Wang2021} in terms of result and criterion, and is an improvement on \citet{Lian2017}.


\section{Experimental Verification} \label{sec:exp}

\begin{figure}[t]
    \centering
    \begin{subfigure}[t]{.22\textwidth}
        \includegraphics[width=\textwidth]{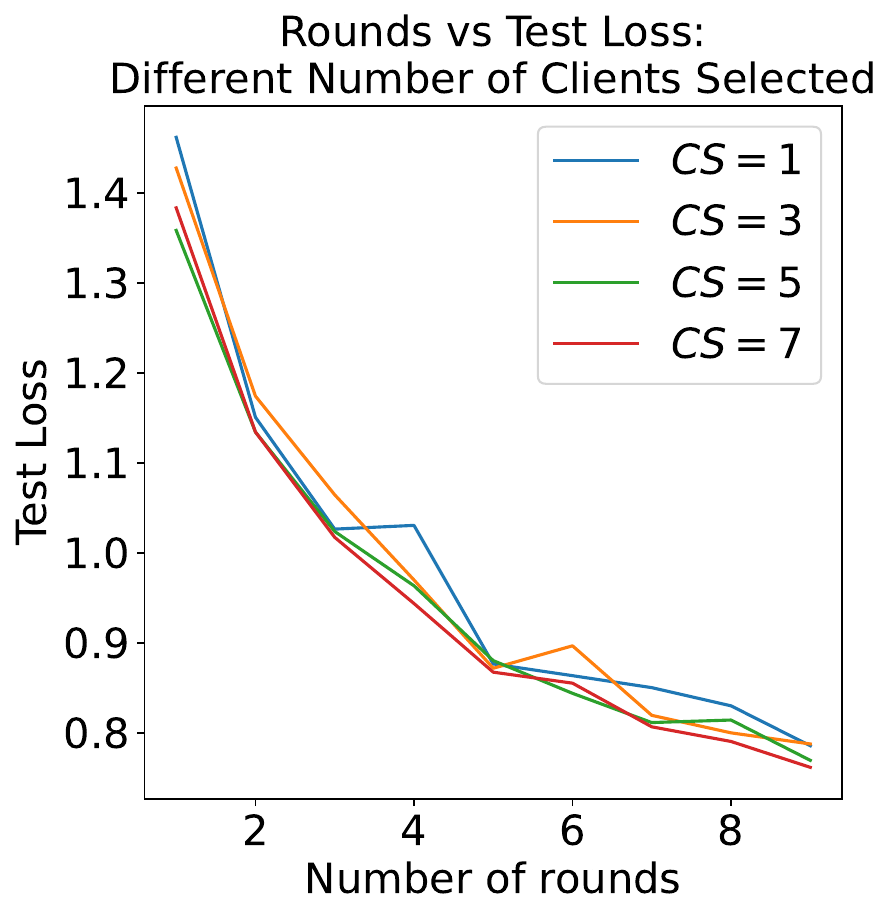}
        \caption{IID Scenario}\label{fig:cs_iid_psasgd}
    \end{subfigure}\hfill
    \begin{subfigure}[t]{.22\textwidth}
        \includegraphics[width=\textwidth]{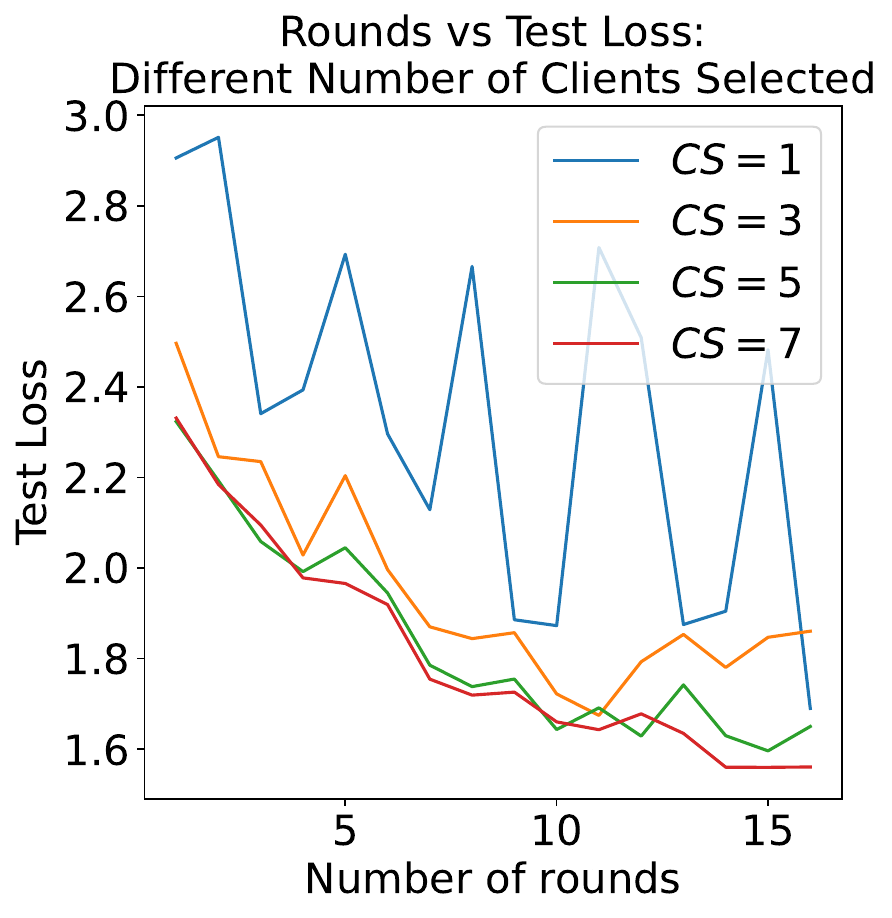}
        \caption{Non-IID Scenario}\label{fig:cs_niid_psasgd}
    \end{subfigure}\hfill
    \vspace{1em}
    \caption{Model convergence with different number of clients selected for both IID and non-IID scenario. The non-IID scenario is using a dirichlet distribution with $\alpha=0.6$. We set the communication period $\tau = 70$ for IID case, and $\tau = 60$ for Non-IID case.} \label{fig:cscomp_psasgd}
    \vspace{2em}
\end{figure}

We run some initial experiments on the CIFAR-10 dataset to validate our analysis. We apply some standard transformations on our images (resize, flip, rotation, jitter and zoom) and utilise VGG16 with default weights for the classification purpose. 
In most experiments, we use a batch size of 128, 8 clients and 20 epochs. We also select clients randomly at every round.
Specific hyperparameters can be found in the supplementary materials and code\footnote{Code can be found at \href{https://github.com/gtmliitrpr/Cooperative-SGD-with-Dynamic-Mixing-Matrices}{GTML Github}}. We check for two different scenarios - IID scenario and Non-IID scenario using a dirichlet distribution.

\subsection{PSASGD} \label{exp:psasgd}

We start by analyzing the importance of communication period $\tau$ on the convergence for PSASGD. We see that there is no observable trend with increasing $\tau$ in both the IID case and Non-IID case, validating our analysis. The graphs are shown in Figure \ref{fig:taucomp_psasgd}.

\begin{figure}[t]
    \centering
    \begin{subfigure}[t]{.22\textwidth}
        \includegraphics[width=\textwidth]{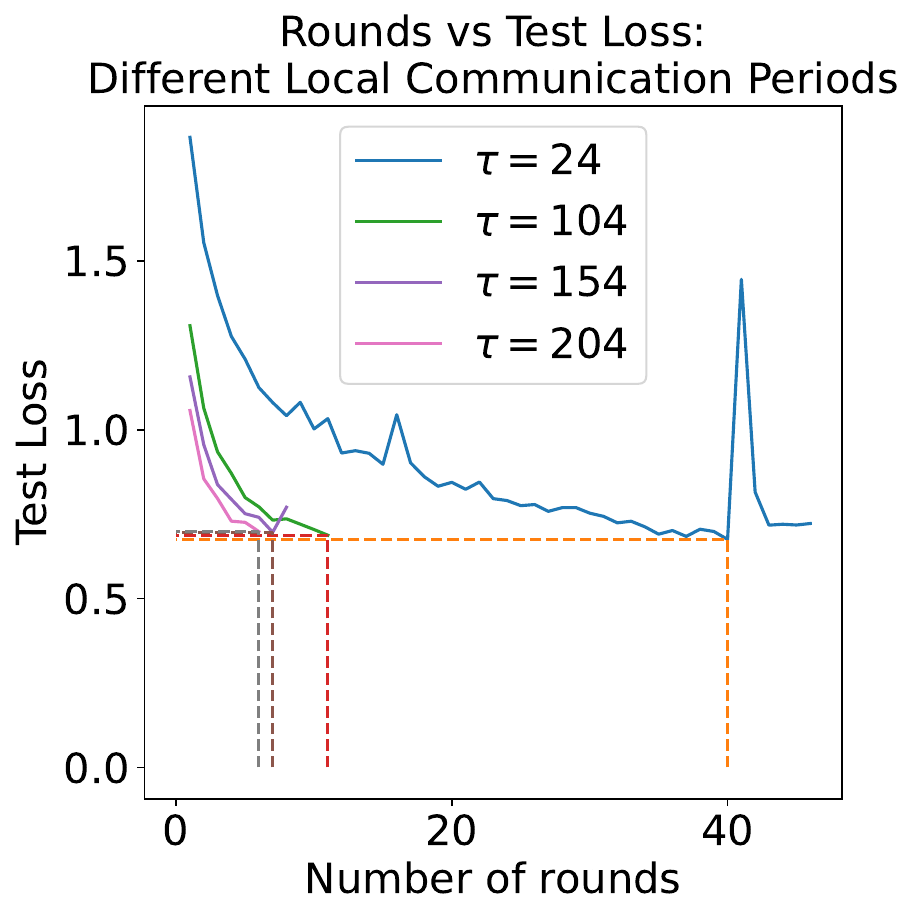} 
        \caption{IID Scenario}\label{fig:tau_iid_psasgd}
    \end{subfigure}\hfill
    \begin{subfigure}[t]{.22\textwidth}
        \includegraphics[width=\textwidth]{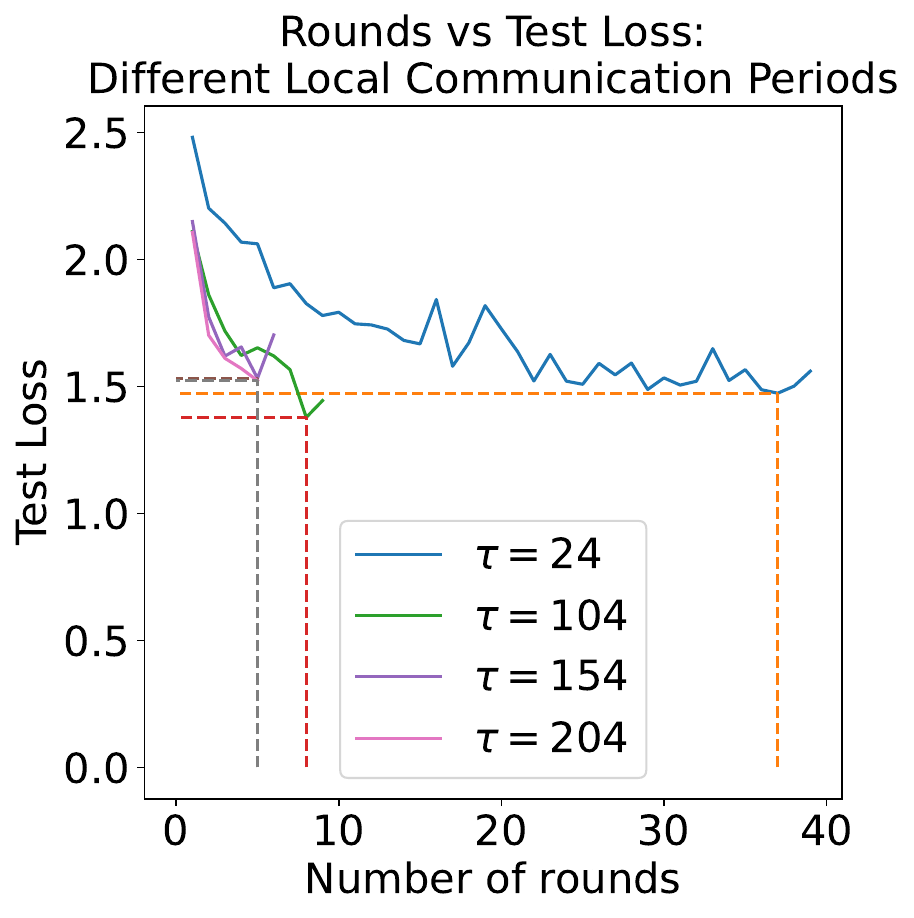}
        \caption{Non-IID Scenario}\label{fig:tau_niid_psasgd}
    \end{subfigure}\hfill
    \vspace{1em}
    \caption{Model convergence with different communication periods for both IID and non-IID scenario. 4 clients were selected out of 8 in IID scenario and 7 were selected out of 8 in Non-IID scenario. The non-IID scenario is using a dirichlet distribution with $\alpha=0.6$.} \label{fig:taucomp_psasgd}
    \vspace{2em}
\end{figure}

We also compare the effect of client selection on model convergence. We showcase the results for both IID and non-IID scenarios in Figure \ref{fig:cscomp_psasgd}. We see that choosing a larger fraction of clients not only leads to improved convergence, but also increased stability.

Finally, we compare different initialization strategies. We want to check the dependency of the convergence on how large the initial parameter values are. We multiply the default weights of all parameters of VGG16 with some factor $i$. The results are provided in Figure \ref{fig:initcomp_psasgd}. The initialization strategy is described in detail in the supplementary material. We see that there is a lack of a proper trend, indicating the tradeoff between the initialization error due to parameter weights and the starting point on the surface of the loss function ($F(u_{1})$).

\begin{figure}[t]
    \centering
    \begin{subfigure}[t]{.22\textwidth}
        \includegraphics[width=\textwidth]{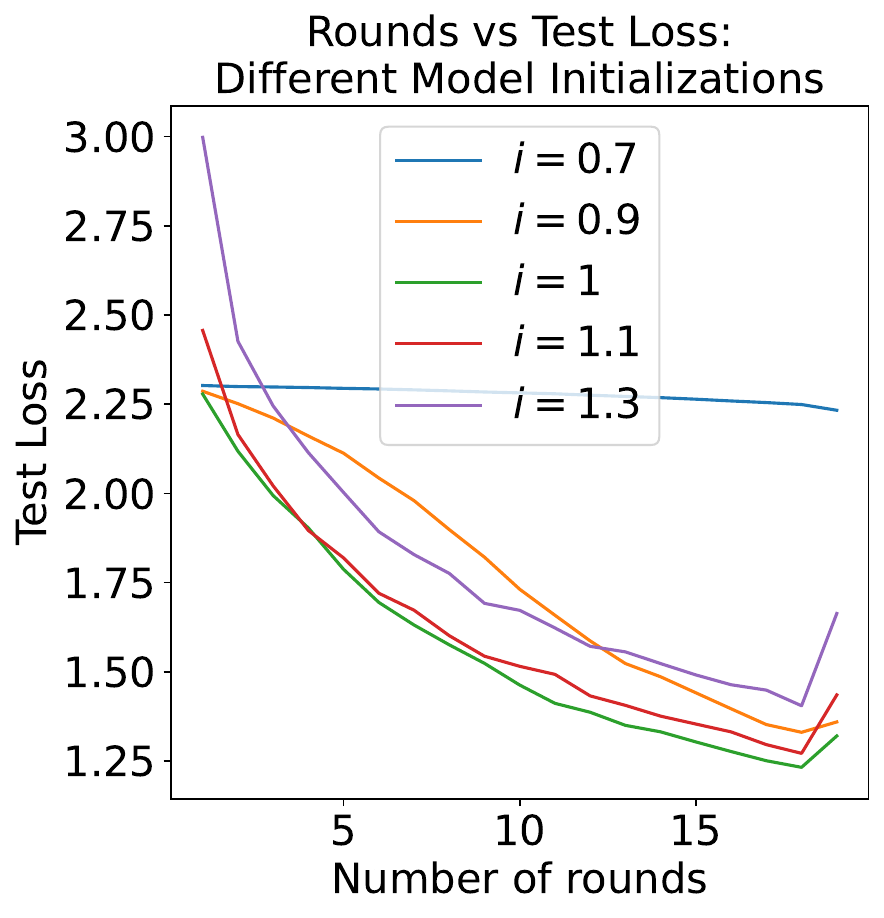}
        \caption{IID Scenario}\label{fig:init_iid_psasgd}
    \end{subfigure}\hfill
    \begin{subfigure}[t]{.22\textwidth}
        \includegraphics[width=\textwidth]{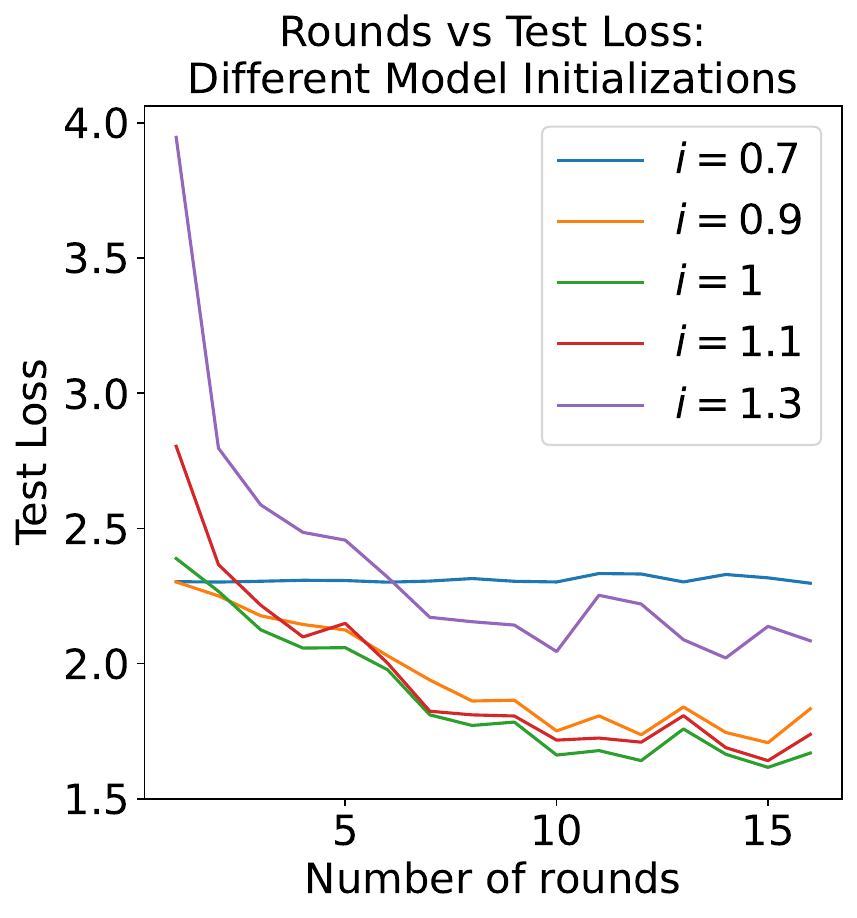}
        \caption{Non-IID Scenario}\label{fig:init_niid_psasgd}
    \end{subfigure}\hfill
    \vspace{1em}
    \caption{Model convergence with different initialization values of parameters for both IID and non-IID scenario. The non-IID scenario is using a dirichlet distribution with $\alpha=0.6$. 5 clients were selected out of 8 in both scenarios, and the communication period $\tau = 60$.} \label{fig:initcomp_psasgd}
    \vspace{2em}
\end{figure}

\subsection{Fully Synchronous SGD and D-PSGD}
We found fully synchronous SGD behaves like PSASGD, which is expected as it is a special case of it. For the D-PSGD scenario, we run the same experiments as was done in Section \ref{exp:psasgd}. We found that when $\tau>1$, we get the same trends as PSASGD. This is in keeping with our analysis so far which indicates D-PSGD with $\tau>1$ would behave just like PSASGD. We also conduct the experiments for $\tau=1$, which is the classical D-PSGD scenario, and see it behaves the same as PSASGD and Fully Synchronized SGD. We provide graphs for both D-PSGD and Fully Synchronous SGD in the supplementary materials (Sections \ref{exp:fsync} and \ref{exp:dpsgd}).

\section{Conclusion}
In this paper we extend the Cooperative SGD framework to the scenario of dynamic and assymetric mixing matrices, while also incorporating client selection. Furthermore, we analyse several existing algorithms and provide convergence guarantees with improved or matching bounds for achieving it in comparison to existing literature. We provide novel analysis and verify them experimentally. We also manage to provide a lower bound on the number of clients that need to be selected for both IID as well as non-IID scenarios. Finally, as a result of relaxing assumptions in previous works, we are able to provide convergence guarantees for a wider family of algorithms than has been previously possible in the distributed local SGD literature, both in the IID and non-IID scenario. We show how closely related the IID scenario is with the non-IID, as well as bridging the gap between different algorithms that are seemingly quite different. Further exploration is possible within the bounds of our work as well as beyond. Most notably, while tighter bounds under fewer assumptions are possible, it could be more worthwhile to study how other variables affect the convergence. For example, an analysis on the effects of variable learning rates might be useful to understand the effect of learning rate schedulers. There is also a possibility to extend this work to communication primitives such as PUSHSUM, and perhaps even generalize for any communication primitive.



\begin{ack}
    The project is supported by ANRF project CRG/2022/004980. 
\end{ack}



\bibliography{mainpaper}

\pagebreak
\onecolumn
\begin{center}
\textbf{\Large Supplementary Materials}
\end{center}
\setcounter{equation}{0}
\setcounter{figure}{0}
\setcounter{table}{0}
\setcounter{page}{1}

\section{Proof Preliminaries} \label{prelims}
In this section, we specify the notations to be used for proving Theorem 1 and Theorem 2 below:
\begin{enumerate}
    \item We express the $m$-element column vector $\left[ 1, 1, 1, ... \right]^{T}$ as \textbf{1} and define the square matrix \boldmath $J = {11^{T}}/{1^{T}1}$ \unboldmath. \boldmath $J$ \unboldmath and \boldmath $I$ \unboldmath are of size $(m+v)\times (m+v)$.
    \item We utilize a dynamic matrix $S_{k}$ in our update rule to be used as a stand-in for the regular mixing matrix (when aggregating), and for the identity matrix (when doing local updates).
        \begin{equation}
            \mathcal{S}_{k} =
            \begin{cases}
                W_{k} & k\mod\tau=0\\
                I & \text{Otherwise}
            \end{cases}
        \end{equation}
    \item We also define two matrices $X_{k}, G_{k} \in \mathbb{R}^{d\times(m+v)}$ as below:
        \begin{align*}
            X_{k} = \left[ x_{k}^{(1)},...x_{k}^{(m)}, z_{k}^{(1)},..z_{k}^{(v)} \right]\\
            G_{k} = \left[ g_{1}(x_{k}^{(1)}),...g_{m}(x_{k}^{(m)}), 0,..0 \right]
        \end{align*}
    For the above two matrices, we assume the models of clients not selected at iteration $k$ are zero i.e, $x_{k}^{(j)} = 0 \text{ and } g_{j}(x_{k}^{(j)}) = 0 \text{ } \forall j \notin C_{k}$.
    \item \textbf{Update Rule}: $X_{k+1} = \left( X_{k} - \eta G_{k} \right)S_{k}^{T}$
    \item We take the following quantities to represent average model and average learning rate respectively,
    \begin{align*}
        u_{k} = X_{k}\frac{1_{m+v}}{m+v},\text{ } \eta_{eff} = \frac{cm}{m+v}\eta 
    \end{align*}
    \item $\equiv_{k} = \{ \xi_{k}^{(1)}, \xi_{k}^{(2)}, ... \} \implies$ Set of mini batches for $m$ workers at iteration $K$.
    \item $E_{k}$ denotes the expression $\mathbb{E}_{\equiv_{k}|X_{k}}$
    \item Averaged stochastic gradient with selected clients as $C_{k}$:
        \begin{equation}
            \mathcal{G}_{k} =
            \begin{cases}
                \frac{1}{cm}\sum_{i \in C_{k}} g(x_{k}^{(i)}) & \text{IID case}\\
                \frac{1}{cm}\sum_{i \in C_{k}} g_{i}(x_{k}^{(i)}) & \text{Non-IID case}
            \end{cases}
        \end{equation}
    \item Averaged full batch gradient with selected clients as $C_{k}$:
        \begin{equation}
            \mathcal{H}_{k} =
            \begin{cases}
                \frac{1}{cm}\sum_{i \in C_{k}} \nabla F(x_{k}^{(i)}) & \text{IID case}\\
                \frac{1}{cm}\sum_{i \in C_{k}} \nabla F_{i}(x_{k}^{(i)}) & \text{Non-IID case}
            \end{cases}
        \end{equation}
    \item We also define the following matrix:
        \begin{equation}
            \nabla F(X_{K}) =
            \begin{cases}
                [\underbrace{\nabla F(x_{k}^{(1)}),..., \nabla F(x_{k}^{(t)}), ..., \nabla F(x_{k}^{(m)})}_{m\text{ elements}}, 0,0,....,0] & \text{IID case}\\
                [\underbrace{\nabla F_{t}(x_{k}^{(1)}),..., \nabla F_{t}(x_{k}^{(t)}), ..., \nabla F_{t}(x_{k}^{(m)}),}_{m\text{ elements}} 0,0,....,0] & \text{Non-IID case}
            \end{cases}
        \end{equation}
    For the above matrix, we assume the average full batch gradient of clients not selected at iteration $k$ are zero i.e, $\nabla F(x_{k}^{(j)}) = 0 \text{ } \forall j \notin C_{k}$.
    
    \item Frobenius norm defined for $A \in M_{n}$ by:
        \begin{equation}
            ||A||_{F}^{2} = ||A^{T}||_{F}^{2} = |Tr(AA^{T})| = \sum_{i,j=1}^{n} |a_{ij}|^{2}
        \end{equation}
    \item Operator norm defined for $A \in M_{n}$ by:
        \begin{equation}
            ||A||_{op} = \max_{||x||=1}||Ax|| = \sqrt{\lambda_{max}(A^{T}A)}
        \end{equation}
    \item $||\nabla F(X_{k})||^{2}_{F} = \sum_{i=1}^{m} ||\nabla F(x_{k}^{(i)})||^{2} = \sum_{i \in C_{k}} ||\nabla F(x_{k}^{(i)})||^{2}$
        
\end{enumerate}


\section{Proof of Theorem 1} \label{IID-Theorem1}
We will proceed to prove Theorem 1 in this section. We start by stating the assumptions for IID scenario  (Section \ref{iid:assumptions}). We thereafter lay the groundwork for proving Lemma 2 by utilizing three lemmas - Lemma 3 (Section \ref{iid:lemma3}), Lemma 4 (Section \ref{iid:lemma4}) and Lemma 5 (Section \ref{iid:lemma5}). We state Lemma 2 initially for convenience, but finally prove it in Section \ref{iid:lemma2}. Thereafter, we prove 2 more lemmas needed for proving Theorem 1- Lemma 7 (Section \ref{iid:lemma7}) and Lemma 8 (Section \ref{iid:lemma8}). Finally, we prove Theorem 1 in Section \ref{theorem1}.
\subsection{Assumptions} \label{iid:assumptions}
\begin{enumerate}
    \item \textbf{Smoothness}: $||\nabla F_{i}(x) - \nabla F_{i}(y)|| \leq L||x-y||$
    \item \textbf{Lower Bounded}: $F(x) \geq F_{ing}$
    \item \textbf{Unbiased Gradients}: $\mathbb{E}_{\xi|x}[g_{i}(x)] = \nabla F(x)$
    \item \textbf{Bounded Variance}: $\mathbb{E}_{\xi|x}||g_{i}(x)-\nabla F(x)||^{2} \leq \sigma^{2}$\newline
    where $\sigma$ is a non-negative constant and in inverse proportion to mini-batch size
    \item \textbf{Mixing Matrix}:
         $W_{k}^{T}1_{m+v} = 1_{m+v}$
    \item For all rounds, fraction of clients selected $c$ is fixed.
\end{enumerate}
\subsection{Lemma 2} \label{iid:lemma2_statement}
\textcolor{red}{\textbf{For algorithm $\mathcal{A}(\tau, W, v)$, under assumption 1-5, if learning rate satisfies $\eta_{eff}L\left(1+\frac{\beta}{m}\right) \leq 1$ and all local parameters are initialized at same point $u_{1}$, then the average-squared gradient after K iterations is bounded as follows:
\begin{equation}
    \mathbb{E} \left[ \frac{1}{K}\sum_{k=1}^{K}||\nabla F(u_{k}) ||^{2} \right] \leq \underbrace{\frac{2[F(u_{1}) - F_{inf}]}{\eta_{eff}K} + \frac{\eta_{eff}L\sigma^{2}}{cm}}_{\text{Fully Sync. SGD}} + \underbrace{\frac{L^{2}}{K}\sum_{k=1}^{K}\frac{\mathbb{E}||X_{k}(I-J)||^{2}_{F}}{cm}}_{\text{Network Error}}
\end{equation}}}
\subsection{Lemmas to prove for Lemma 2}
\subsubsection{Lemma 3} \label{iid:lemma3}
\textcolor{red}{\textbf{Under assumption 3 and 4, we have the following variance bound for averaged stochastic gradient:
\begin{equation}
    \mathbb{E}_{\equiv_{K}|X_{K}}[||\mathcal{G}_{k} - \mathcal{H}_{k}||^{2}] \leq  \frac{\sigma^{2}}{cm}
\end{equation}}}
\textit{Proof}:
\begin{equation*}
    \begin{gathered}
        \mathbb{E}_{\equiv_{K}|X_{K}}[||\mathcal{G}_{k} - \mathcal{H}_{k}||^{2}]\\
        = \mathbb{E}_{\equiv_{K}|X_{K}} \left| \left| \frac{1}{cm}\sum_{i \in C_{k}}[g(x_{k}^{(i)}) - \nabla F(x_{k}^{(i)})]\right| \right| ^{2}\\
        = \frac{1}{(cm)^{2}}\mathbb{E}_{\equiv_{K}|X_{K}} \left[ \sum_{i \in C_{k}}||g(x_{k}^{(i)}) - \nabla F(x_{k}^{(i)})||^{2} + \sum_{j,l \in C_{k}, j\neq l} \left< g(x_{k}^{(j)}) - \nabla F(x_{k}^{(j)}),g(x_{k}^{(l)}) - \nabla F(x_{k}^{(l)}) \right> \right]\\
        = \frac{1}{(cm)^{2}} \sum_{i \in C_{k}} \mathbb{E}_{\xi_{K}^{(i)}|X_{K}} ||g(x_{k}^{(i)}) - \nabla F(x_{k}^{(i)})||^{2} \\+ \frac{1}{(cm)^{2}}\sum_{j,l \in C_{k}, j\neq l} \left< \mathbb{E}_{\xi_{K}^{(j)}|X_{K}} [g(x_{k}^{(j)}) - \nabla F(x_{k}^{(j)})] , \mathbb{E}_{\xi_{k}^{(l)}|X_{K}} [g(x_{k}^{(l)}) - \nabla F(x_{k}^{(l)})] \right>
    \end{gathered}
\end{equation*}
This step was possible as $\equiv_{K}|X_{K}$ can be seen as over batches of all clients, so as we treat $\{\xi_{k}^{(i)}\}$ as independent random variables, for a particular client it should only depend on their own batch. We also use the equation $\mathbb{E}<X, Y> = <\mathbb{E}[X], \mathbb{E}[Y]> $ provided X and Y are independent. As we treat batches as independent, hence $\mathbb{E}_{\xi_{K}^{(j)}|X_{K}} [g(x_{k}^{(j)}) - \nabla F(x_{k}^{(j)})]$ is independent of $\mathbb{E}_{\xi_{k}^{(l)}|X_{K}} [g(x_{k}^{(l)}) - \nabla F(x_{k}^{(l)})]$.\\

Now by assumption 3, we have
\begin{equation*}
    \begin{gathered}
        \mathbb{E}_{\xi_{K}^{(j)}|X_{K}} [g(x_{k}^{(j)}) - \nabla F(x_{k}^{(j)})] = \mathbb{E}_{\xi_{K}^{(j)}|X_{K}} [g(x_{k}^{(j)})] - \nabla F(x_{k}^{(j)})\\
        \textcolor{magenta}{\text{($\because$ Since $\nabla F$ is not dependent on batches)}}\\
        \implies \nabla F(x_{k}^{(j)}) - \nabla F(x_{k}^{(j)}) = 0
    \end{gathered}
\end{equation*}
Now by assumption 4, we have
\begin{equation*}
    \begin{gathered}
        \mathbb{E}_{\xi_{K}^{(i)}|X_{K}} ||g(x_{k}^{(i)}) - \nabla F(x_{k}^{(i)})||^{2} \leq \sigma^{2}
    \end{gathered}
\end{equation*}
Thus, we get
\begin{equation*}
    \begin{gathered}
        \frac{1}{(cm)^{2}} \sum_{i \in C_{k}} \mathbb{E}_{\xi_{K}^{(i)}|X_{K}} ||g(x_{k}^{(i)}) - \nabla F(x_{k}^{(i)})||^{2} \\+ \frac{1}{(cm)^{2}}\sum_{j,l \in S_{k},j\neq l} \left< \mathbb{E}_{\xi_{K}^{(j)}|X_{K}} [g(x_{k}^{(j)}) - \nabla F(x_{k}^{(j)})] , \mathbb{E}_{\xi_{k}^{(l)}|X_{K}} [g(x_{k}^{(l)}) - \nabla F(x_{k}^{(l)})] \right> \\
        \leq \frac{1}{(cm)^{2}} \sum_{i \in C_{k}}  \sigma^{2} \\
        \leq \frac{\sigma^{2}}{cm}
    \end{gathered}
\end{equation*}

\subsubsection{Lemma 4} \label{iid:lemma4}
\textcolor{red}{\textbf{Under assumption 3, the expected inner product between stochastic gradient and full batch gradient for averaged model can be expanded as:
\begin{equation}
    \mathbb{E}_{\equiv_{K}|X_{K}}\left[ \left< \nabla F(u_{k}), \mathcal{G}_{k}  \right> \right] = \frac{1}{2}||\nabla F(u_{k})||^{2} + \frac{1}{2cm}\sum_{i \in C_{k}}||\nabla F(x_{k}^{(i)})||^{2} - \frac{1}{2cm}\sum_{i \in C_{k}}^{m}||\nabla F(u_{k}) - \nabla F(x_{k}^{(i)})||^{2}
\end{equation}}}
\textit{Proof}:
\begin{equation*}
    \begin{gathered}
    \mathbb{E}_{\equiv_{K}|X_{K}}\left[ \left< \nabla F(u_{k}), \mathcal{G}_{k}  \right> \right] = \mathbb{E}_{\equiv_{K}|X_{K}}\left[ \left< \nabla F(u_{k}), \frac{1}{cm} \sum_{i \in C_{k}} g(x_{k}^{(i)})  \right> \right]\\
    = \left< \mathbb{E}_{\equiv_{K}|X_{K}}[\nabla F(u_{k})], \mathbb{E}_{\equiv_{K}|X_{K}}\left[ \frac{1}{cm} \sum_{i \in C_{k}} g(x_{k}^{(i)}) \right]  \right> \\
    = \left< \nabla F(u_{k}), \mathbb{E}_{\equiv_{K}|X_{K}}\left[ \frac{1}{cm} \sum_{i \in C_{k}} g(x_{k}^{(i)}) \right]  \right> \\
    = \left< \frac{1}{cm}\nabla F(u_{k}).cm, \mathbb{E}_{\equiv_{K}|X_{K}}\left[ \frac{1}{cm} \sum_{i \in C_{k}} g(x_{k}^{(i)}) \right]  \right> \\
    = \left< \frac{1}{cm}\nabla F(u_{k}) \sum_{i \in C_{k}} 1, \mathbb{E}_{\equiv_{K}|X_{K}}\left[ \frac{1}{cm} \sum_{i \in C_{k}} g(x_{k}^{(i)}) \right]  \right> \\
    = \frac{1}{cm} \sum_{i \in C_{k}}\left< \nabla F(u_{k}), \mathbb{E}_{\equiv_{K}|X_{K}}\left[g(x_{k}^{(i)}) \right]  \right> \\
    = \frac{1}{cm} \sum_{i \in C_{k}}\left< \nabla F(u_{k}), \nabla F(x_{k}^{(i)})  \right> \\
    \textcolor{magenta}{\text{($\because$ Assumption 3)}}\\
    = \frac{1}{2cm} \sum_{i \in C_{k}}\left[ ||\nabla F(u_{k})||^{2} + ||\nabla F(x_{k}^{(i)})||^{2} - ||\nabla F(u_{k}) - \nabla F(x_{k}^{(i)})||^{2}  \right] \\
    \textcolor{magenta}{\text{($\because 2<a,b> = ||a||^{2} + ||b||^{2} - ||a-b||^{2}$ )}}\\
    = \frac{1}{2} ||\nabla F(u_{k})||^{2} + \frac{1}{2cm} \sum_{i \in C_{k}} ||\nabla F(x_{k}^{(i)})||^{2} - \frac{1}{2cm} \sum_{i \in C_{k}}||\nabla F(u_{k}) - \nabla F(x_{k}^{(i)})||^{2}
    \end{gathered}
\end{equation*}

\subsubsection{Lemma 5} \label{iid:lemma5}
\textcolor{red}{\textbf{Under assumption 3 and 4, the squared norm of averaged stochastic gradient is bounded as:
\begin{equation}
    \mathbb{E}_{\equiv_{K}|X_{K}}\left[ ||\mathcal{G}_{k}||^{2} \right] \leq \frac{||\nabla F(X_{K})||^{2}_{F}}{cm} + \frac{\sigma^{2}}{cm}
\end{equation}}}
\textit{Proof}:
Since $\mathbb{E}_{\equiv_{K}|X_{K}}\left[ \mathcal{G}_{k} \right] = \mathcal{H}_{k}$, we have
\begin{equation*}
    \begin{gathered}
    \mathbb{E}_{\equiv_{K}|X_{K}}\left[ ||\mathcal{G}_{k}||^{2} \right] = \mathbb{E}_{\equiv_{K}|X_{K}}\left[ ||\mathcal{G}_{k} - \mathbb{E}[\mathcal{G}_{k}]||^{2} \right] + ||\mathbb{E}_{\equiv_{K}|X_{K}}[\mathcal{G}_{k}]||^{2}\\
    \textcolor{magenta}{\text{($\because \mathbb{E}[||a||^{2}] = \mathbb{E}[||a - \mathbb{E}[a]||]^{2} + ||\mathbb{E}[a]||^{2}$)}}\\
    = \mathbb{E}_{\equiv_{K}|X_{K}}\left[ ||\mathcal{G}_{k} - \mathbb{E}[\mathcal{G}_{k}]||^{2} \right] + ||\mathcal{H}_{k}||^{2}\\
    \leq \frac{\sigma^{2}}{cm} + ||\mathcal{H}_{k}||^{2}\\
    \textcolor{magenta}{\text{($\because$ Lemma 3)}}\\
    = \frac{\sigma^{2}}{cm} + \frac{1}{cm}||\nabla F({X}_{k})||^{2}_{F}\\
    \textcolor{magenta}{( \because \text{By Jensen's Inequality, } \varphi(\mathbb{E}[X]) \leq \mathbb{E}[\varphi(X)] \implies \text{ Take }\varphi = ||x||^{2} \text{ and } X = \nabla F(x_{k}^{(i)}) \text{, then }}\\\textcolor{magenta}{ ||\mathcal{H}_{k}||^{2} = \left| \left| \frac{1}{cm} \sum_{i \in C_{k}} ||\nabla F(x_{k}^{(i)})||^{2}  \right| \right| \leq  \frac{1}{cm} \sum_{i \in C_{k}} \left| \left| \nabla F(x_{k}^{(i)})  \right| \right|^{2}  = \frac{1}{cm}||\nabla F({X}_{k})||^{2}_{F})}\\
    \implies \mathbb{E}_{\equiv_{K}|X_{K}}\left[ ||\mathcal{G}_{k}||^{2} \right] \leq \frac{||\nabla F(X_{K})||^{2}_{F}}{cm} + \frac{\sigma^{2}}{cm}
    \end{gathered}
\end{equation*}
\subsection{Proof of Lemma 2} \label{iid:lemma2}
By Lipshitz continuous gradient assumption,
\begin{equation*}
    \begin{gathered}
    f(y) \leq f(x) + \nabla f(x)^{T}(y-x) + \frac{L}{2}||y-x||^{2}
    \end{gathered}
\end{equation*}
Here, let $f = F$, and $y = u_{k+1}$ and $x = u_{k}$. Then,
\begin{equation*}
    \begin{gathered}
    F(u_{k+1}) \leq F(u_{k}) + \nabla F(u_{k})^{T}(u_{k+1}-u_{k}) + \frac{L}{2}||u_{k+1}-u_{k}||^{2}
    \end{gathered}
\end{equation*}
By update rule,
\begin{equation*}
    \begin{gathered}
    u_{k+1}-u_{k} = -\eta_{eff} \left[ \frac{1}{cm}\sum_{i \in C_{k}} g(x_{k}^{(i)}) \right]
    \end{gathered}
\end{equation*}
Thus, we get,
\begin{equation*}
    \begin{gathered}
    F(u_{k+1}) \leq F(u_{k}) - \nabla F(u_{k})^{T}\left(\eta_{eff} \left[ \frac{1}{cm}\sum_{i \in C_{k}} g(x_{k}^{(i)}) \right]\right) + \frac{L}{2}\left|\left|-\eta_{eff} \left[ \frac{1}{cm}\sum_{i \in C_{k}} g(x_{k}^{(i)}) \right]\right|\right|^{2}\\
    \implies F(u_{k+1}) \leq F(u_{k}) - \eta_{eff}\nabla F(u_{k})^{T}\mathcal{G}_{k} + \frac{\eta_{eff}^{2}L}{2}\left|\left| \mathcal{G}_{k} \right|\right|^{2}\\
    \implies F(u_{k+1}) \leq F(u_{k}) - \eta_{eff}\left<\nabla F(u_{k}),\mathcal{G}_{k}\right> + \frac{\eta_{eff}^{2}L}{2}\left|\left| \mathcal{G}_{k} \right|\right|^{2}
    \end{gathered}
\end{equation*}
Taking expectations on both sides, we get
\begin{equation*}
    \begin{gathered}
     \mathbb{E}_{\equiv|X_{k}}[F(u_{k+1})] - \mathbb{E}_{\equiv|X_{k}}[F(u_{k})] \leq - \eta_{eff}\mathbb{E}_{\equiv|X_{k}}[\left<\nabla F(u_{k}),\mathcal{G}_{k}\right>] + \frac{\eta_{eff}^{2}L}{2} \mathbb{E}_{\equiv|X_{k}}[\left|\left| \mathcal{G}_{k} \right|\right|^{2}]\\
     \implies F(u_{k+1}) - F(u_{k}) \leq - \eta_{eff}\mathbb{E}_{\equiv|X_{k}}[\left<\nabla F(u_{k}),\mathcal{G}_{k}\right>] + \frac{\eta_{eff}^{2}L}{2} \mathbb{E}_{\equiv|X_{k}}[\left|\left| \mathcal{G}_{k} \right|\right|^{2}]\\
     \textcolor{magenta}{\text{($\because \mathbb{E}_{\equiv|X_{k}}[F(u_{k+1})] = F(u_{k+1})$ and $\mathbb{E}_{\equiv|X_{k}}[F(u_{k})] = F(u_{k})$ due to them being average of models)}}
    \end{gathered}
\end{equation*}
By using Lemmas 4 and 5, we get
\begin{equation*}
    \begin{gathered}
      F(u_{k+1}) - F(u_{k}) \leq - \eta_{eff}\mathbb{E}_{\equiv|X_{k}}[\left<\nabla F(u_{k}),\mathcal{G}_{k}\right>] + \frac{\eta_{eff}^{2}L}{2} \mathbb{E}_{\equiv|X_{k}}[\left|\left| \mathcal{G}_{k} \right|\right|^{2}]\\
      \leq - \eta_{eff} \left[ \frac{1}{2}||\nabla F(u_{k})||^{2} + \frac{1}{2cm}\sum_{i \in C_{k}}||\nabla F(x_{k}^{(i)})||^{2} - \frac{1}{2cm}\sum_{i \in C_{k}}||\nabla F(u_{k}) - \nabla F(x_{k}^{(i)})||^{2} \right] \\+ \frac{\eta_{eff}^{2}L}{2} \left[ \frac{||\nabla F(X_{K})||^{2}_{F}}{cm} + \frac{\sigma^{2}}{cm} \right]\\
      \leq - \frac{\eta_{eff}}{2}||\nabla F(u_{k})||^{2} - \frac{\eta_{eff}}{2cm}\sum_{i \in C_{k}}||\nabla F(x_{k}^{(i)})||^{2} + \frac{\eta_{eff}}{2cm}\sum_{i \in C_{k}}||\nabla F(u_{k}) - \nabla F(x_{k}^{(i)})||^{2} \\+
      \frac{\eta_{eff}^{2}L}{2}  \sum_{i \in C_{k}}\frac{||\nabla F(x_{k}^{(i)})||^{2}}{cm} + \frac{\eta_{eff}^{2}L\sigma^{2}}{2cm}\\
      \leq - \frac{\eta_{eff}}{2}||\nabla F(u_{k})||^{2} - \frac{\eta_{eff}}{2}\left[ 1 - \eta_{eff}L  \right]\frac{1}{cm}\sum_{i \in C_{k}}||\nabla F(x_{k}^{(i)})||^{2} \\ + \frac{\eta_{eff}^{2}L\sigma^{2}}{2cm} + \frac{\eta_{eff}L^{2}}{2cm}\sum_{i \in C_{k}}||u_{k} - x_{k}^{(i)}||^{2}\\
      \textcolor{magenta}{\text{($\because$ Assumption 1)}}
    \end{gathered}
\end{equation*}

Rearranging a bit, we get,
\begin{equation*}
    \begin{gathered}
      \frac{\eta_{eff}}{2}||\nabla F(u_{k})||^{2} \leq F(u_{k}) - F(u_{k+1}) + \frac{\eta_{eff}^{2}L\sigma^{2}}{2cm} + \frac{\eta_{eff}L^{2}}{2cm}\sum_{i \in C_{k}}||u_{k} - x_{k}^{(i)}||^{2} \\ - \frac{\eta_{eff}}{2}\left[ 1 - \eta_{eff}L  \right]\frac{1}{cm}\sum_{i \in C_{k}}||\nabla F(x_{k}^{(i)})||^{2}
    \end{gathered}
\end{equation*}
Dividing by $\eta_{eff}/2$, we get
\begin{equation*}
    \begin{gathered}
      ||\nabla F(u_{k})||^{2} \leq \frac{2[ F(u_{k}) - F(u_{k+1})]}{\eta_{eff}} + \frac{\eta_{eff}L\sigma^{2}}{cm} + \frac{L^{2}}{cm}\sum_{i \in C_{k}}||u_{k} - x_{k}^{(i)}||^{2} \\ - \left[ 1 - \eta_{eff}L  \right]\frac{1}{cm}||\nabla F(X_{k})||^{2}_{F}
    \end{gathered}
\end{equation*}
Averaging over all iterates from $k=1$ to $k=K$ on both sides, we get
\begin{equation*}
    \begin{gathered}
      \frac{1}{K}\sum_{k=1}^{K}||\nabla F(u_{k})||^{2} \leq \frac{2[ \sum_{k=1}^{K}F(u_{k}) - \sum_{k=1}^{K}F(u_{k+1})]}{\eta_{eff}K} + \frac{\eta_{eff}L\sigma^{2}}{cmK}\sum_{k=1}^{K}1 + \frac{L^{2}}{Kcm}\sum_{k=1}^{K}\sum_{i \in C_{k}}||u_{k} - x_{k}^{(i)}||^{2} \\ - \left[ 1 - \eta_{eff}L  \right]\frac{1}{K}\sum_{k=1}^{K}\frac{||\nabla F(X_{k})||^{2}_{F}}{cm}
    \end{gathered}
\end{equation*}
Now,
\begin{equation*}
    \begin{gathered}
      2\sum_{k=1}^{K}[F(u_{k}) - F(u_{k+1})]   = F(u_{1}) - F(u_{k}) \leq F(u_{1}) - F_{inf}
    \end{gathered}
\end{equation*}
Thus, we get
\begin{equation*}
    \begin{gathered}
      \frac{1}{K}\sum_{k=1}^{K}||\nabla F(u_{k})||^{2} \leq \frac{2[ F(u_{1}) - F_{inf}]}{\eta_{eff}K} + \frac{\eta_{eff}L\sigma^{2}}{cm} + \frac{L^{2}}{Kcm}\sum_{k=1}^{K}\sum_{i \in C_{k}}||u_{k} - x_{k}^{(i)}||^{2} \\ - \left[ 1 - \eta_{eff}L  \right]\frac{1}{K}\sum_{k=1}^{K}\frac{||\nabla F(X_{k})||^{2}_{F}}{cm}
    \end{gathered}
\end{equation*}
Taking expectations on both sides, we get,
\begin{equation*}
    \begin{gathered}
      \mathbb{E}\left[ \frac{1}{K}\sum_{k=1}^{K}||\nabla F(u_{k})||^{2} \right] \leq \frac{2[ F(u_{1}) - F_{inf}]}{\eta_{eff}K} + \frac{\eta_{eff}L\sigma^{2}}{cm} + \frac{L^{2}}{Kcm}\sum_{k=1}^{K}\sum_{i \in C_{k}}\mathbb{E}||u_{k} - x_{k}^{(i)}||^{2} \\ - \left[ 1 - \eta_{eff}L  \right]\frac{1}{K}\sum_{k=1}^{K}\frac{\mathbb{E}||\nabla F(X_{k})||^{2}_{F}}{cm}
    \end{gathered}
\end{equation*}
If $\eta_{eff}L \leq 1$, then,
\begin{equation*}
    \begin{gathered}
      \mathbb{E}\left[ \frac{1}{K}\sum_{k=1}^{K}||\nabla F(u_{k})||^{2} \right] \leq \frac{2[ F(u_{1}) - F_{inf}]}{\eta_{eff}K} + \frac{\eta_{eff}L\sigma^{2}}{cm} + \frac{L^{2}}{Kcm}\sum_{k=1}^{K}\sum_{i \in C_{k}}\mathbb{E}||u_{k} - x_{k}^{(i)}||^{2}
    \end{gathered}
\end{equation*}
Now $u_{k} = X_{k}\frac{1_{m+v}}{m+v}$. Adding some positive term to RHS we get,
\begin{equation*}
    \begin{gathered}
      \sum_{i \in C_{k}}||u_{k} - x_{k}^{(i)}||^{2} \leq \sum_{i=1}^{m}||u_{k} - x_{k}^{(i)}||^{2} \leq \sum_{i=1}^{m}||u_{k} - x_{k}^{(i)}||^{2} + \sum_{j=1}^{v}||u_{k} - z_{k}^{(j)}||^{2}\\
      = \left|\left| U1_{m+v}^{T} - X_{k}\right|\right|^{2}_{F}\\
      = \left|\left| X_{k}\frac{1_{m+v} 1_{m+v}^{T}}{m+v} - X_{k}\right|\right|^{2}_{F}\\
      = \left|\left| X_{k}J - X_{k}\right|\right|^{2}_{F}\\
      = \left|\left| X_{k}(I - J)\right|\right|^{2}_{F}
    \end{gathered}
\end{equation*}
Substituting, we get Lemma 2.

\textbf{NOTE}: If we do not consider the specific value for learning rate, we get from derivation:
\begin{equation}
    \label{eqn9}
    \begin{gathered}
      \mathbb{E}\left[ \frac{1}{K}\sum_{k=1}^{K}||\nabla F(u_{k})||^{2} \right] \leq \frac{2[ F(u_{1}) - F_{inf}]}{\eta_{eff}K} + \frac{\eta_{eff}L\sigma^{2}}{cm} + \frac{L^{2}}{Kcm}\sum_{k=1}^{K}\mathbb{E}\left|\left| X_{k}(I - J)\right|\right|^{2}_{F} \\ - \left[ 1 - \eta_{eff}L  \right]\frac{1}{K}\sum_{k=1}^{K}\frac{||\nabla F(X_{k})||^{2}_{F}}{cm}
    \end{gathered}
\end{equation}
\subsection{Other Lemmas to prove for Theorem 1}
\subsubsection{Lemma 6} \label{iid:lemma6}
\textcolor{red}{\textbf{For two matrices $A \in \mathbb{R}^{d\times m}$ and $B \in \mathbb{R}^{m\times m}$ we have:
\begin{equation}
    ||AB||_{F} \leq ||A||_{op}||B||_{F}
\end{equation}}}
\textit{Proof}:
Let the rows of A be denoted by $a_{1}^{T}, a_{2}^{T}, ..., a_{d}^{T}$ and $\mathcal{I} = \{ i \in [1,d]: ||a_{i}|| \neq 0\}$. Then,
\begin{equation*}
    \begin{gathered}
    ||AB||^{2}_{F} = \sum_{i=1}^{d}||a_{i}^{T}B||^{2} = \sum_{i \in \mathcal{I}}||Ab_{i}||^{2}\\
    = \sum_{i \in \mathcal{I}}\frac{||Ab_{i}||^{2}}{||b_{i}||^{2}} ||b_{i}||^{2}\\
    \leq \sum_{i \in \mathcal{I}}||A||^{2}_{op} ||b_{i}||^{2}\\
    \leq ||A||^{2}_{op} \sum_{i \in \mathcal{I}}||b_{i}||^{2}\\
    \leq ||A||_{op}^{2}||B||_{F}^{2}
    \end{gathered}
\end{equation*}

\subsubsection{Lemma 7} \label{iid:lemma7}
\textcolor{red}{\textbf{For two matrices $A \in \mathbb{R}^{m\times n}$ and $B \in \mathbb{R}^{n\times m}$, we have:
\begin{equation}
    |Tr(AB)| \leq ||A||_{F} ||B||_{F}
\end{equation}}}
\textit{Proof}:
Let the $i^{th}$ row of A be denoted by $a_{i}^{T} \in \mathbb{R}^{n}$ and the $i^{th}$ column of B be denoted by $b_{i} \in \mathbb{R}^{n}$. Then, by definition of matrix trace,
\begin{equation*}
    \begin{gathered}
    Tr(AB) = \sum_{i=1}^{m}\sum_{j=1}^{n}A_{ij}B_{ji} = \sum_{i=1}^{m}a_{i}^{T}b_{i} = \sum_{i=1}^{m}\left< a_{i}, b_{i} \right>
    \end{gathered}
\end{equation*}
By Cauchu Shwartz Inequality,
\begin{equation*}
    \begin{gathered}
    |\left< u, v \right>|^{2} \leq \left< u, u \right>\left< v, v \right> = ||u||||v||
    \end{gathered}
\end{equation*}
For sums of inner products, we have,
\begin{equation*}
    \begin{gathered}
    \left| \sum_{i=1}^{m} \left< a_{i}, b_{i} \right>\right|^{2} \leq \sqrt{\sum_{i=1}^{m} a_{i}^{2}}\sqrt{\sum_{i=1}^{m} b_{i}^{2}}
    \end{gathered}
\end{equation*}
Thus, we have,
\begin{equation*}
    \begin{gathered}
    \left| \sum_{i=1}^{m} \left< a_{i}, b_{i} \right>\right|^{2} \leq \left(  \sum_{i=1}^{m} \left| \left| a_{i} \right| \right|^{2} \right) \left(  \sum_{i=1}^{m} \left| \left| b_{i} \right| \right|^{2} \right)
    = \left| \left|A\right| \right|_{F}^{2} \left| \left|B\right| \right|_{F}^{2}
    \end{gathered}
\end{equation*}

\subsubsection{Lemma 8} \label{iid:lemma8}
\textcolor{red}{\textbf{Let there be some $m\times m$ matrix $\Phi_{s,k}$ that satisfies assumption 5. Then,
\begin{equation}
    ||\Phi_{s,k}^{T}(I-J)||^{2}_{F} \leq \delta
\end{equation}
where $\delta \in [0, c(m-1)]$}}\\
\textit{Proof}:
We can say,
\begin{equation*}
    \begin{gathered}
    ||\Phi_{s,k}^{T}(I-J)||^{2}_{F} = ||\Phi_{s,k}^{T}-J||^{2}_{F}
    \end{gathered}
\end{equation*}
We also know that product of column stochastic matrices is column stochastic. Hence, $\Phi_{s,k}^{T}$ is column stochastic.
Now, we consider the definition of the frobenius norm
\begin{equation*}
    \begin{gathered}
    ||A||_{F}^{2} = Tr(A^{T}A) = \sum_{i=1, j=1}^{m} a_{ij}^{2} \\
    = (a_{11}^{2} + a_{12}^{2} \cdots + a_{1m}^{2}) + (a_{21}^{2} + a_{22}^{2} \cdots + a_{2m}^{2}) + \cdots + (a_{m1}^{2} + a_{m2}^{2} \cdots + a_{mm}^{2})\\
    = (a_{11}^{2} + a_{21}^{2} \cdots + a_{m1}^{2}) + (a_{12}^{2} + a_{22}^{2} \cdots + a_{m2}^{2}) + \cdots + (a_{1m}^{2} + a_{2m}^{2} \cdots + a_{mm}^{2})
    \end{gathered}
\end{equation*}
where the column stochastic matrix $A_{m\times m}$ is as follows:
    $$A =
    \begin{bmatrix}
        a_{11} & a_{12} & \cdots & a_{1m}\\
        a_{21} & a_{22} & \cdots & a_{2m}\\
        \vdots & \vdots & \ddots & \vdots\\
        a_{m1} & a_{m2} & \cdots & a_{mm}
    \end{bmatrix}$$
Let the non-selected clients have entire columns as zero. Also, let us write this in a form similar to $||\Phi_{s,k}^{T}-J||^{2}_{F}$, which implies some constant $d = \frac{1}{m}$ is subtracted from the all elements of A (For $\Phi_{s,k}^{T}-J$ it will be $d = \frac{1}{m+v}$). So we write as:
\begin{equation*}
    \begin{gathered}
    ||A||_{F}^{2} = ((a_{11}-d)^{2} + (a_{21}-d)^{2} \cdots + (a_{m1}^{2}-d)) + ((a_{12}-d)^{2} + (a_{22}-d)^{2} \cdots + (a_{m2}-d)^{2}) \\+ \cdots + ((a_{1m}-d)^{2} + (a_{2m}-d)^{2} \cdots + (a_{mm}-d)^{2})
    \end{gathered}
\end{equation*}
Now, 
\begin{equation*}
    \begin{gathered}
    ((a_{11}-d)^{2} + (a_{21}-d)^{2} \cdots + (a_{m1}-d)^{2}) + ((a_{12}-d)^{2} + (a_{22}-d)^{2} \cdots + (a_{m2}-d)^{2}) \\+ \cdots + ((a_{1m}-d)^{2} + (a_{2m}-d)^{2} \cdots + (a_{mm}-d)^{2})\\
    = [((a_{11}-d) + (a_{21}-d) \cdots + (a_{m1}-d))^{2} - 2\sum_{j\neq i} (a_{i1}-d).(a_{j1}-d)]+ \\\cdots + [((a_{1m}-d) + (a_{2m}-d) \cdots + (a_{mm}-d))^{2} - 2\sum_{j\neq i} (a_{im}-d).(a_{jm}-d)]\\
    = [(a_{11} + a_{21} \cdots + a_{m1} - md)^{2} - 2\sum_{j\neq i} (a_{i1}-d).(a_{j1}-d) ] + \\\cdots + [(a_{1m} + a_{2m} \cdots + a_{mm}-md) - 2\sum_{j\neq i} (a_{im}-d).(a_{jm}-d) ]\\
    = [(1-md)^{2} - 2\sum_{j\neq i} (a_{i1}-d).(a_{j1}-d)] + \cdots + [(1-md)^{2} - 2\sum_{j\neq i} (a_{im}-d).(a_{jm}-d) ]\\
    = [m(1-md)^{2} -2[\sum_{j\neq i} (a_{i1}-d).(a_{j1}-d) + \cdots +   \sum_{j\neq i} (a_{im}-d).(a_{jm}-d)]\\
    = -2[\sum_{j\neq i} (a_{i1}-d).(a_{j1}-d) + \cdots +   \sum_{j\neq i} (a_{im}-d).(a_{jm}-d)]\\
    \textcolor{magenta}{\left(\because d = \frac{1}{m}\right)}
    \end{gathered}
\end{equation*}
Now, let us rewrite each term as:
\begin{equation*}
    \begin{gathered}
    -2\sum_{j\neq i} (a_{ip}-d).(a_{jp}-d) = -2\sum_{j\neq i}[a_{ip}a_{jp} -d(a_{ip} + a_{jp}) + d^{2}]\\
    = 2d\sum_{j\neq i}(a_{ip} + a_{jp}) - 2[\sum_{j\neq i}a_{ip}a_{jp} + \sum_{j\neq i}d^{2}]
    \end{gathered}
\end{equation*}
We can say that as each element appears at most $(m-1)$ in sum (when all clients are selected), hence,
\begin{equation*}
    \begin{gathered}
    2d\sum_{j\neq i}(a_{ip} + a_{jp}) \leq [(a_{1p}+a_{2p}) + \cdots + (a_{1p}+a_{mp})] + [(a_{2p}+a_{3p}) + \cdots + (a_{2p}+a_{mp})] + \cdots + [a_{(m-1)p}+a_{mp}]\\
    \leq 2d(m-1)\sum_{i=1}^{m}a_{ip} = 2d(m-1)\\
    \end{gathered}
\end{equation*}
Thus, we can say,
\begin{equation*}
    \begin{gathered}
    -2\sum_{j\neq i} (a_{ip}-d).(a_{jp}-d) \leq 
     2d(m-1) - 2\sum_{j\neq i}a_{ip}a_{jp} -d^{2}m(m-1)\\
    \textcolor{magenta}{(\because \text{ there are } \Mycomb{m}{2} \text{ values}  )}\\
    \leq \frac{m-1}{m} - 2\sum_{j\neq i}a_{ip}a_{jp}
    \end{gathered}
\end{equation*}
Now, let us denote the smallest element in the column $p$ as $t_{p}^{(1)}$ 
and the second smallest element as $t_{p}^{(2)}$
. Then,
\begin{equation*}
    \begin{gathered}
    2\sum_{j\neq i} a_{ip}.a_{jp} \geq 2\sum_{j\neq i} t_{p}^{(1)}.t_{p}^{(2)} = 2t_{p}^{(1)}t_{p}^{(2)} \sum_{j\neq i}1  = t_{p}^{(1)}t_{p}^{(2)}m(m-1)\\
    \textcolor{magenta}{(\because \text{ there are } \Mycomb{m}{2} \text{ values}  )}
    \end{gathered}
\end{equation*}
Thus, we can say,
\begin{equation*}
    \begin{gathered}
    -2\sum_{j\neq i} a_{ip}.a_{jp} \leq  -t_{p}^{(1)}t_{p}^{(2)}m(m-1)
    \end{gathered}
\end{equation*}
We can, therefore, say,
\begin{equation*}
    \begin{gathered}
    -2\sum_{j\neq i} (a_{ip}-d).(a_{jp}-d)  
    \leq \frac{m-1}{m}  -t_{p}^{(1)}t_{p}^{(2)}m(m-1)\\
    \end{gathered}
\end{equation*}
For clients who aren't selected, we have,
\begin{equation*}
    \begin{gathered}
    -2\sum_{j\neq i} (a_{ip}-d).(a_{jp}-d)  
    =-2\sum_{j\neq i} d^{2} = -d(m-1) = -\frac{m-1}{m}\\
    \end{gathered}
\end{equation*}
Thus, we get,
\begin{equation*}
    \begin{gathered}
    - 2[\sum_{j\neq i} (a_{i1}-d).(a_{j1}-d) + \cdots +   \sum_{j\neq i} (a_{im}-d).(a_{jm}-d)] \leq \sum_{p=1, p\in C_{k}}^{m} \frac{m-1}{m}  \\- m(m-1)\sum_{p=1,p\in C_{k}}^{m}t_{p}^{(1)}t_{p}^{(2)} - \sum_{p=1, p\notin C_{k}}^{m}\frac{m-1}{m}\\
    \implies - 2[\sum_{j\neq i} (a_{i1}-d).(a_{j1}-d) + \cdots +   \sum_{j\neq i} (a_{im}-d).(a_{jm}-d)] \leq \sum_{p=1, p\in C_{k}}^{m} \frac{m-1}{m}  - m(m-1)\sum_{p=1,p\in C_{k}}^{m}t_{p}^{(1)}t_{p}^{(2)}\\
    \implies - 2[\sum_{j\neq i} (a_{i1}-d).(a_{j1}-d) + \cdots +   \sum_{j\neq i} (a_{im}-d).(a_{jm}-d)] \leq c(m-1) - m(m-1)\sum_{p=1,p\in C_{k}}^{m}t_{p}^{(1)}t_{p}^{(2)}\\
    \implies - 2[\sum_{j\neq i} (a_{i1}-d).(a_{j1}-d) + \cdots +   \sum_{j\neq i} (a_{im}-d).(a_{jm}-d)] \leq c(m-1) - m(m-1)[t_{1}^{(1)}t_{1}^{(2)} + \cdots +  t_{m}^{(1)}t_{m}^{(2)}]
    \end{gathered}
\end{equation*}
Let the smallest pair of these be $t^{(1)}t^{(2)}$. Then,
\begin{equation*}
    \begin{gathered}
    - 2[\sum_{j\neq i} (a_{i1}-d).(a_{j1}-d) + \cdots +   \sum_{j\neq i} (a_{im}-d).(a_{jm}-d)] \leq c(m-1) -m(m-1)[t^{(1)}t^{(2)}  + \cdots +  t^{(1)}t^{(2)}]\\
    \implies - 2[\sum_{j\neq i} (a_{i1}-d).(a_{j1}-d) + \cdots +   \sum_{j\neq i} (a_{im}-d).(a_{jm}-d)] \leq \underbrace{c(m-1)[1 - m^{2}t^{(1)}t^{(2)}]}_{\delta}\\
    \end{gathered}
\end{equation*}
If  $t^{(1)} \rightarrow \frac{1}{m}$ and $t^{(2)} \rightarrow \frac{1}{m}$ i.e, we are close to the largest possible value of these smallest values, or we can say we are close to the uniform aggregation state, then,
\begin{equation*}
    \begin{gathered}
    - 2[\sum_{j\neq i} (a_{i1}-d).(a_{j1}-d) + \cdots +   \sum_{j\neq i} (a_{im}-d).(a_{jm}-d)] \leq 0\\
    \implies ||A||_{F}^{2} = 0 \text{ when } c=1\\
    \end{gathered}
\end{equation*}
If $t^{(1)} \rightarrow 0$ and $t^{(2)} \rightarrow 0$ i.e, we are close to the smallest possible value of these small values, or we can say we are close to the aggregation state where we heavily bias against some clients. Then,
\begin{equation*}
    \begin{gathered}
    - 2[\sum_{j\neq i} a_{i1}.a_{j1} + \cdots +   \sum_{j\neq i} a_{im}.a_{jm}] \leq c(m-1)\\
    \implies ||A||_{F}^{2} \leq c(m-1)\\
    \end{gathered}
\end{equation*}
This is a larger value than the one we get with uniform aggregation. Thus, getting closer to uniform aggregation allows us to get a smaller value of $\delta$.

\subsection{Theorem 1} \label{theorem1}
As per Lemma 2 derivation, we achieved Equation \ref{eqn9}. We want to provide an upper bound to the network error term.\\

Now,
\begin{equation*}
    \begin{gathered}
    X_{k}(I-J) = (X_{k-1} - \eta G_{k-1})S_{k-1}^{T}(I-J) = X_{k-1}S_{k-1}^{T}(I-J) - \eta G_{k-1}S_{k-1}^{T}(I-J)\\
    \implies X_{k-1}S_{k-1}^{T}(I-J) - \eta (G_{k-1}S_{k-1}^{T}-G_{k-1}S_{k-1}^{T}J) = X_{k-1}S_{k-1}^{T}(I-J) - \eta G_{k-1}(S_{k-1}^{T}-J)
    \end{gathered}
\end{equation*}
Expanding $X_{k-1}$ further, we get after rearranging,
\begin{equation*}
    \begin{gathered}
    X_{k}(I-J) = X_{k-2}S_{k-2}^{T}S_{k-1}^{T}(I-J) - \eta G_{k-2}(S_{k-2}^{T}S_{k-1}^{T}-J) - \eta G_{k-1}(S_{k-1}^{T}-J)
    \end{gathered}
\end{equation*}
Going all the way, we get,
\begin{equation*}
    \begin{gathered}
    X_{k}(I-J) = X_{1}\Phi_{1,k-1}^{T}(I-J) - \eta \sum_{s=1}^{k-1} G_{s}\Phi_{s,k-1}^{T}(I-J)
    \end{gathered}
\end{equation*}
where $\Phi_{s,k-1}^{T} = \prod_{l=s}^{k-1}S_{l}^{T}$.\\
Taking the squared frobenius norm and expectations on both sides, we get,
\begin{equation*}
    \begin{gathered}
    \mathbb{E} ||X_{k}(I-J)||^{2}_{F} = \mathbb{E} \left|\left|X_{1}\Phi_{1,k-1}^{T}(I-J) - \eta\sum_{s=1}^{k-1} G_{s}\Phi_{s,k-1}^{T}(I-J)\right|\right|^{2}_{F}\\
    \implies \mathbb{E} ||X_{k}(I-J)||^{2}_{F} = \mathbb{E} \left|\left|X_{1}\Phi_{1,k-1}^{T}(I-J)\right|\right|^{2}_{F} +  \eta^{2} \mathbb{E} \left|\left|\sum_{s=1}^{k-1} G_{s}\Phi_{s,k-1}^{T}(I-J)\right|\right|^{2}_{F} \\- 2\left< X_{1}\Phi_{1,k-1}^{T}(I-J), \eta\sum_{s=1}^{k-1} G_{s}\Phi_{s,k-1}^{T}(I-J)  \right>\\
    \implies \mathbb{E} ||X_{k}(I-J)||^{2}_{F} \leq \mathbb{E} \left|\left|X_{1}\Phi_{1,k-1}^{T}(I-J)\right|\right|^{2}_{F} +  \eta^{2} \mathbb{E} \left|\left|\sum_{s=1}^{k-1} G_{s}\Phi_{s,k-1}^{T}(I-J)\right|\right|^{2}_{F}
    \end{gathered}
\end{equation*}

Now let $K = j\tau + i$, where $j$ denotes the index of communication rounds and $i$ denotes index of local updates. Then, considering that for non-communicating rounds $S_{k} = I$, then we get,
\begin{equation*}
    \Phi_{s,k-1} = 
    \begin{cases}
        I  = \Phi_{j,j}^{T} & j\tau < s \leq j\tau + i\\
        W_{j}^{T}  = \Phi_{j-1,j}^{T} & (j-1)\tau < s \leq j\tau\\
        W_{j-1}^{T}W_{j}^{T}  = \Phi_{j-2,j}^{T} & (j-2)\tau < s \leq (j-1)\tau\\
        \vdots\\
        W_{1}^{T}...W_{j-1}^{T}W_{j}^{T}  = \Phi_{0,j}^{T} & 0 < s \leq \tau\\
    \end{cases}
\end{equation*}
For the ease of writing, we write $\Phi_{j\tau,j\tau+i} = \Phi_{j,j}$ and so on. We also define accumulated stochastic gradient within one local period as:
\begin{equation*}
    Y_{r} = \sum_{s=r\tau+1}^{(r+1)\tau} G_{s} \tag{$0 \leq r < j$}
\end{equation*}
and,
\begin{equation*}
    \begin{gathered}
    Y_{j} = \sum_{s=j\tau+1}^{j\tau+i-1} G_{s}
    \end{gathered}
\end{equation*}
Similarly, define accumulated full batch gradient
\begin{equation*}
    Q_{r} = \sum_{s=r\tau+1}^{(r+1)\tau} \nabla F(X_{s}) \tag{$0 \leq r < j$}
\end{equation*}
and,
\begin{equation*}
    \begin{gathered}
    Q_{j} = \sum_{s=j\tau+1}^{j\tau+i-1} \nabla F(X_{s})
    \end{gathered}
\end{equation*}
Then, we have,
\begin{equation*}
    \begin{gathered}
    \sum_{s=1}^{\tau} G_{s}\Phi_{s,k-1}^{T}(I-J) = Y_{0}\Phi_{0,j}^{T}(I-J)\\
    \sum_{s=\tau+1}^{2\tau} G_{s}\Phi_{s,k-1}^{T}(I-J) = Y_{1}\Phi_{1,j}^{T}(I-J)\\
    \vdots\\
    \sum_{s=j\tau+1}^{j\tau+i-1} G_{s}\Phi_{s,k-1}^{T}(I-J) = Y_{j}(I-J)\\
    \end{gathered}
\end{equation*}
Adding up, we get,
\begin{equation*}
    \begin{gathered}
    \sum_{s=1}^{j\tau+i-1} G_{s}\Phi_{s,k-1}^{T}(I-J) = \sum_{r=0}^{j}Y_{r}(\Phi_{r,j}^{T}-J)
    \end{gathered}
\end{equation*}
Now the network error can be decomposed into
\begin{equation*}
    \begin{gathered}
    \mathbb{E} ||X_{k}(I-J)||^{2}_{F} \leq \mathbb{E} \left|\left|X_{1}\Phi_{1,k-1}^{T}(I-J)\right|\right|^{2}_{F} + \eta^{2} \mathbb{E} \left|\left|\sum_{r=0}^{j} Y_{r}\Phi_{r,j}^{T}(I-J)\right|\right|^{2}_{F}\\
    = \mathbb{E} \left|\left|X_{1}\Phi_{1,k-1}^{T}(I-J)\right|\right|^{2}_{F} + \eta^{2} \mathbb{E} \left|\left|\sum_{r=0}^{j} (Y_{r}-Q_{r})\Phi_{r,j}^{T}(I-J) + \sum_{r=0}^{j} Q_{r}\Phi_{r,j}^{T}(I-J)\right|\right|^{2}_{F}\\
    \leq \underbrace{\mathbb{E} \left|\left|X_{1}\Phi_{1,k-1}^{T}(I-J)\right|\right|^{2}_{F}}_{T0} + \underbrace{2\eta^{2} \mathbb{E} \left|\left|\sum_{r=0}^{j} (Y_{r}-Q_{r})\Phi_{r,j}^{T}(I-J)\right|\right|^{2}_{F}}_{\text{T1}} + \underbrace{2\eta^{2} \mathbb{E} \left|\left|\sum_{r=0}^{j} Q_{r}\Phi_{r,j}^{T}(I-J)\right|\right|^{2}_{F}}_{\text{T2}}\\
    \textcolor{magenta}{(\because ||a+b||^{2} \leq 2||a||^{2} + 2||b||^{2})}
    \end{gathered}
\end{equation*}
We now look to bound T0, T1 and T2. We will derive bounds for average of all iterates.
\subsubsection{Bounding T0} \label{t0}
\begin{equation*}
    \begin{gathered}
    T0 = \mathbb{E} \left|\left|X_{1}\Phi_{1,k-1}^{T}(I-J)\right|\right|^{2}_{F}\\
    = \mathbb{E} \left|\left|X_{1}(\Phi_{1,k-1}^{T}-J)\right|\right|^{2}_{F}\\
    \leq \mathbb{E} \left|\left|X_{1}\right|\right|^{2}_{op} \left|\left| \Phi_{1,k-1}^{T}-J\right|\right|^{2}_{F}\\
    \leq \mathbb{E} \left|\left|X_{1}\right|\right|^{2}_{F} \left|\left| \Phi_{1,k-1}^{T}-J\right|\right|^{2}_{F}\\
    \textcolor{magenta}{(\because ||A||_{op} \leq ||A||_{F})}\\
    \leq \mathbb{E} \left|\left|X_{1}\right|\right|^{2}_{F} \delta\\
    \textcolor{magenta}{(\because \text{Lemma 8})}\\
    \leq \delta\left|\left|X_{1}\right|\right|^{2}_{F}
    \end{gathered}
\end{equation*}

\subsubsection{Bounding T1} \label{t1}
\begin{equation*}
    \begin{gathered}
    T1 = 2\eta^{2} \mathbb{E} \left|\left|\sum_{r=0}^{j} (Y_{r}-Q_{r})\Phi_{r,j}^{T}(I-J)\right|\right|^{2}_{F} = 2\eta^{2} \sum_{r=0}^{j} \mathbb{E} \left|\left| (Y_{r}-Q_{r})\Phi_{r,j}^{T}(I-J)\right|\right|^{2}_{F}
    \end{gathered}
\end{equation*}
This is possible as cross terms are zero. We can prove this as follows:
\begin{equation*}
    \begin{gathered}
    2\eta^{2} \mathbb{E} \left|\left|\sum_{r=0}^{j} (Y_{r}-Q_{r})\Phi_{r,j}^{T}(I-J)\right|\right|^{2}_{F} = 2\eta^{2} \sum_{r=0}^{j} \mathbb{E} \left|\left| (Y_{r}-Q_{r})\Phi_{r,j}^{T}(I-J)\right|\right|^{2}_{F} \\+ 2\eta^{2}\sum_{l,s=0, l\neq s}^{j} \mathbb{E}\left< (Y_{l}-Q_{l})\Phi_{l,j}^{T}(I-J), (Y_{s}-Q_{s})\Phi_{s,j}^{T}(I-J) \right>\\
    = 2\eta^{2} \sum_{r=0}^{j} \mathbb{E} \left|\left| (Y_{r}-Q_{r})\Phi_{r,j}^{T}(I-J)\right|\right|^{2}_{F} + 2\eta^{2}\sum_{l,s=0, l\neq s}^{j} \left< \mathbb{E}(Y_{l}-Q_{l})\Phi_{l,j}^{T}(I-J), \mathbb{E}(Y_{s}-Q_{s})\Phi_{s,j}^{T}(I-J) \right>
    \end{gathered}
\end{equation*}
Now,
\begin{equation*}
    \begin{gathered}
    \mathbb{E}[Y_{r}] = \sum_{s=r\tau+1}^{(r+1)\tau} \mathbb{E}[G_{s}] = \sum_{s=r\tau+1}^{(r+1)\tau} [\mathbb{E}[g_{1}(x_{s}^{1})], \mathbb{E}[g_{1}(x_{s}^{1})], ...]\\
    = \sum_{s=r\tau+1}^{(r+1)\tau} [\nabla F_{1}(x_{s}^{1}), \nabla F_{1}(x_{s}^{1}), ...]  = Q_{r}
    \end{gathered}
\end{equation*}
Thus, the cross terms become zero. By Lemma 6, we have,
\begin{equation*}
    \begin{gathered}
    \left|\left| (Y_{r}-Q_{r})\Phi_{r,j}^{T}(I-J)\right|\right|^{2}_{F} \leq  \left|\left| Y_{r}-Q_{r}\right|\right|^{2}_{op} \left|\left| \Phi_{r,j}^{T}(I-J)\right|\right|^{2}_{F}\\
    \implies \left|\left| (Y_{r}-Q_{r})\Phi_{r,j}^{T}(I-J)\right|\right|^{2}_{F} \leq  \left|\left| Y_{r}-Q_{r}\right|\right|^{2}_{F} \left|\left| \Phi_{r,j}^{T}(I-J)\right|\right|^{2}_{F}\\
    \end{gathered}
\end{equation*}
Thus, we have,
\begin{equation*}
    \begin{gathered}
    T1 \leq 2\eta^{2} \sum_{r=0}^{j} \mathbb{E} \left|\left| Y_{r}-Q_{r}\right|\right|^{2}_{F} \left|\left| \Phi_{r,j}^{T}(I-J)\right|\right|^{2}_{F}\\
    \implies T1 \leq 2\eta^{2} \sum_{r=0}^{j} \mathbb{E} \left|\left| Y_{r}-Q_{r}\right|\right|^{2}_{F} \delta\\
    \textcolor{magenta}{(\because \text{Lemma 8})}\\
    \implies T1 \leq 2\eta^{2} \sum_{r=0}^{j-1} \mathbb{E} \left|\left| Y_{r}-Q_{r}\right|\right|^{2}_{F} \delta + 2\eta^{2} \mathbb{E} \left|\left| Y_{j}-Q_{j}\right|\right|^{2}_{F}\delta\\
    \end{gathered}
\end{equation*}
Now, for any $0\leq r < j$,
\begin{equation*}
    \begin{gathered}
     \mathbb{E} \left|\left| Y_{r}-Q_{r}\right|\right|^{2}_{F} = \mathbb{E} \left[ \left|\left| \sum_{s=r\tau+1}^{(r+1)\tau}[G_{s}-\nabla F(X_{s})]\right|\right|^{2}_{F} \right] = \sum_{i = 1}^{m} \mathbb{E} \left[ \left|\left| \sum_{s=r\tau+1}^{(r+1)\tau}[g(x_{s}^{(i)})-\nabla F(x_{s}^{(i)})]\right|\right|^{2} \right]\\
     = \sum_{i = 1}^{m}  \left( \mathbb{E} \left[ \sum_{s=r\tau+1}^{(r+1)\tau} \left|\left| g(x_{s}^{(i)})-\nabla F(x_{s}^{(i)})\right|\right|^{2} \right] + \mathbb{E} \left[  \sum_{s,l \in S_{t}, s\neq l}\left< g(x_{s}^{(i)})-\nabla F(x_{s}^{(i)}), g(x_{l}^{(i)})-\nabla F(x_{l}^{(i)}) \right> \right] \right)
    \end{gathered}
\end{equation*}
Here, the cross terms are zero. We prove as follows:
\begin{equation*}
    \begin{gathered}
     \mathbb{E} \left< g(x_{s}^{(i)})-\nabla F(x_{s}^{(i)}), g(x_{l}^{(i)})-\nabla F(x_{l}^{(i)}) \right> =   \mathbb{E}_{x_{s}^{i}, \xi_{s}^{i}, x_{l}^{i}} \mathbb{E}_{\xi_{l}^{i} | x_{s}^{i}, \xi_{s}^{i}, x_{l}^{i}} \left[ \left< g(x_{s}^{(i)})-\nabla F(x_{s}^{(i)}), g(x_{l}^{(i)})-\nabla F(x_{l}^{(i)}) \right> \right]\\
     \textcolor{magenta}{(\because \mathbb{E}_{x}[\mathbb{E}_{y|x} (Z(x,y))] = \mathbb{E}(Z))}\\
     = \mathbb{E}_{x_{s}^{i}, \xi_{s}^{i}, x_{l}^{i}} \left[ \left< \mathbb{E}_{\xi_{l}^{i} | x_{s}^{i}, \xi_{s}^{i}, x_{l}^{i}} \left[ g(x_{s}^{(i)})-\nabla F(x_{s}^{(i)}) \right], \mathbb{E}_{\xi_{l}^{i} | x_{s}^{i}, \xi_{s}^{i}, x_{l}^{i}} \left[ g(x_{l}^{(i)})-\nabla F(x_{l}^{(i)}) \right] \right> \right]\\
     = \mathbb{E}_{x_{s}^{i}, \xi_{s}^{i}, x_{l}^{i}} \left[ \left<  g(x_{s}^{(i)})-\nabla F(x_{s}^{(i)}), \mathbb{E}_{\xi_{l}^{i} | x_{s}^{i}, \xi_{s}^{i}, x_{l}^{i}} \left[ g(x_{l}^{(i)})-\nabla F(x_{l}^{(i)}) \right] \right> \right]\\
     = \mathbb{E} \left[ \left<  g(x_{s}^{(i)})-\nabla F(x_{s}^{(i)}), \mathbb{E}_{\xi_{l}^{i} | x_{s}^{i}, \xi_{s}^{i}, x_{l}^{i}} \left[ g(x_{l}^{(i)})-\nabla F(x_{l}^{(i)}) \right] \right> \right]\\
     = \mathbb{E} \left[ \left<  g(x_{s}^{(i)})-\nabla F(x_{s}^{(i)}), 0 \right> \right]\\
     \textcolor{magenta}{(\because \mathbb{E}_{\xi_{l}^{i} | x_{s}^{i}, \xi_{s}^{i}, x_{l}^{i}} \left[ g(x_{l}^{(i)})-\nabla F(x_{l}^{(i)}) \right] = \mathbb{E}_{\xi_{l}^{i} | x_{s}^{i}, \xi_{s}^{i}, x_{l}^{i}} [g(x_{l}^{(i)})]-\nabla F(x_{l}^{(i)}) = \nabla F(x_{l}^{(i)}) - \nabla F(x_{l}^{(i)}) = 0)}
    \end{gathered}
\end{equation*}
Thus, we have,
\begin{equation*}
    \begin{gathered}
     \mathbb{E} \left|\left| Y_{r}-Q_{r}\right|\right|^{2}_{F} 
     = \sum_{i = 1}^{m}   \mathbb{E} \left[ \sum_{s=r\tau+1}^{(r+1)\tau} \left|\left| g(x_{s}^{(i)})-\nabla F(x_{s}^{(i)})\right|\right|^{2} \right] = \mathbb{E}    \left[ \sum_{i = 1}^{m} \sum_{s=r\tau+1}^{(r+1)\tau} \left|\left| g(x_{s}^{(i)})-\nabla F(x_{s}^{(i)})\right|\right|^{2} \right]\\
     = \mathbb{E}    \left[ \sum_{s=r\tau+1}^{(r+1)\tau} \sum_{i = 1}^{m} \left|\left| g(x_{s}^{(i)})-\nabla F(x_{s}^{(i)})\right|\right|^{2} \right]\\
     \textcolor{magenta}{\left( \because \sum_{x}\sum_{y}a = \sum_{y}\sum_{x}a \right)}\\
     =    \left[ \sum_{s=r\tau+1}^{(r+1)\tau} \sum_{i \in C_{k}} \mathbb{E} 
 \left|\left| g(x_{s}^{(i)})-\nabla F(x_{s}^{(i)})\right|\right|^{2} \right]\\
     \leq    \sum_{s=r\tau+1}^{(r+1)\tau} \sum_{i \in C_{k}} \left[ \sigma^{2} \right]\\
     \textcolor{magenta}{\left( \because \text{Assumption 4} \right)}\\
     \leq  \tau cm \sigma^{2} \\
    \end{gathered}
\end{equation*}
Similarly,
\begin{equation*}
    \begin{gathered}
     \mathbb{E} \left|\left| Y_{j}-Q_{j}\right|\right|^{2}_{F} 
     \leq  (i-1) cm \sigma^{2} \\
    \end{gathered}
\end{equation*}
Thus, we get,
\begin{equation*}
    \begin{gathered}
    T1 \leq 2\eta^{2} \sum_{r=0}^{j-1} \left[ \delta \left( \tau cm \sigma^{2} \right) \right] + 2\eta^{2}(i-1) cm \sigma^{2} \delta\\
    \leq 2\eta^{2}cm \sigma^{2} \delta\left[ \tau j + i-1 \right]\\
    \end{gathered}
\end{equation*}
Now, we sum over all iterates in $j^{th}$ local update period (from $i=1$ to $i=\tau$).
\begin{equation*}
    \begin{gathered}
    \sum_{i=1}^{\tau} T1 \leq  2\eta^{2}cm \sigma^{2}\delta\left[ \tau j \sum_{i=1}^{\tau}1 + \sum_{i=1}^{\tau}i-\sum_{i=1}^{\tau}1 \right]\\
    \implies \sum_{i=1}^{\tau} T1 \leq  2\eta^{2}cm \sigma^{2}\delta\left[ \tau^{2}j + \frac{\tau(\tau+1)}{2}-\tau \right]\\
    \implies \sum_{i=1}^{\tau} T1 \leq  \eta^{2}cm \sigma^{2}\delta\left[ 2\tau^{2}j + \tau(\tau-1) \right]\\ 
    \implies \sum_{i=1}^{\tau} T1 \leq  \eta^{2}cm \sigma^{2} \delta \tau\left[ 2\tau j + \tau-1 \right]
    \end{gathered}
\end{equation*}
Then, summing over all periods $j=0$ to $j=K/\tau-1$ we get,
\begin{equation*}
    \begin{gathered}
     \sum_{j=0}^{K/\tau-1}\sum_{i=1}^{\tau} T1 \leq  \eta^{2}cm \sigma^{2}\delta \tau\left[ 2\tau\sum_{j=0}^{K/\tau-1}j  + (\tau-1)\sum_{j=0}^{K/\tau-1}1 \right]\\ 
     \implies \sum_{j=0}^{K/\tau-1}\sum_{i=1}^{\tau} T1 \leq  \eta^{2}cm \sigma^{2} \delta \tau\left[ 2\tau\frac{K/\tau (K/\tau - 1)}{2} + (\tau-1)K/\tau \right]\\
     \implies \sum_{j=0}^{K/\tau-1}\sum_{i=1}^{\tau} T1 \leq  \eta^{2}cm \sigma^{2}\delta \tau\left[ K\left(K/\tau-1\right) + (\tau-1)K/\tau \right]\\
     \implies \sum_{j=0}^{K/\tau-1}\sum_{i=1}^{\tau} T1 \leq  \eta^{2}cm \sigma^{2}\delta K(K-1)
    \end{gathered}
\end{equation*}

\subsubsection{Bounding T2} \label{t2}
\begin{equation*}
    \begin{gathered}
    T2 = 2\eta^{2} \mathbb{E} \left|\left|\sum_{r=0}^{j} Q_{r}\Phi_{r,j}^{T}(I-J)\right|\right|^{2}_{F}
    \end{gathered}
\end{equation*}
Since $||A||^{2}_{F} = Tr(A^{T}A)$, we have,
\begin{equation*}
    \begin{gathered}
    T2 = 2\eta^{2} \mathbb{E} Tr \left[ \left( \sum_{n=0}^{j} Q_{n}\Phi_{n,j}^{T}(I-J) \right)^{T} \left( \sum_{r=0}^{j} Q_{r}\Phi_{r,j}^{T}(I-J) \right) \right]
    \end{gathered}
\end{equation*}
Expanding the trace equation,
\begin{equation*}
    \begin{gathered}
    T2 = 2\eta^{2} \sum_{r=0}^{j} \sum_{n=0}^{j} \mathbb{E} Tr \left[ \left( \left[\Phi_{n,j}^{T}(I-J)\right]^{T}Q_{n}^{T} \right) \left( Q_{r}\Phi_{r,j}^{T}(I-J) \right) \right]
    \\
    \textcolor{magenta}{
    \left(
    \begin{aligned}
        \because \text{By linearity of $Tr$ and $\mathbb{E}$ as well as by distributive property which states that }\\ \left( \sum_{i=1}^{n}A_{i} \right) \left( \sum_{j=1}^{m}B_{j} \right)  = \sum_{i=1}^{n}\sum_{j=1}^{m}A_{i}B_{j}
    \end{aligned}
    \right)}
    \\
    \end{gathered}
\end{equation*}
We now split into two parts:
\begin{itemize}
    \item Diagonal Terms where $r=n$
    \item Off Diagonal Terms where $r\neq n$
\end{itemize}
Thus,
\begin{equation*}
    \begin{gathered}
    T2 = 2\eta^{2} \sum_{r=0}^{j} \sum_{n=0}^{j} \mathbb{E} Tr \left[ \left( \left[\Phi_{n,j}^{T}(I-J)\right]^{T}Q_{n}^{T} \right) \left( Q_{r}\Phi_{r,j}^{T}(I-J) \right) \right]
    \\
    = 2\eta^{2} \sum_{r=0}^{j} \mathbb{E} Tr \left[ \left[\Phi_{r,j}^{T}( I-J)\right]^{T}Q_{r}^{T}   Q_{r}\Phi_{r,j}^{T}(I-J) \right] + 2\eta^{2}  \sum_{n=0}^{j} \sum_{r=0, r\neq n}^{j} \mathbb{E} Tr \left[  \left[\Phi_{n,j}^{T}(I-J)\right]^{T}Q_{n}^{T}  Q_{r}\Phi_{r,j}^{T}(I-J)  \right]\\
    = 2\eta^{2} \sum_{r=0}^{j} \mathbb{E} || Q_{r}\Phi_{r,j}^{T}(I-J) ||^{2}_{F} + 2\eta^{2}  \sum_{n=0}^{j} \sum_{l=0, l\neq n}^{j} \mathbb{E} Tr \left[  \left[\Phi_{n,j}^{T}(I-J)\right]^{T}Q_{n}^{T}  Q_{l}\Phi_{l,j}^{T}(I-J)  \right]\\
    \textcolor{magenta}{(\because \text{By frobenius norm definition})}
    \end{gathered}
\end{equation*}
Now, by Lemma 7,
\begin{equation*}
    \begin{gathered}
     \left|Tr \left[  \left[\Phi_{n,j}^{T}(I-J)\right]^{T}Q_{n}^{T}  Q_{l}\Phi_{l,j}^{T}(I-J)  \right]\right|
     \leq  ||\left[\Phi_{n,j}^{T}(I-J)\right]^{T}Q_{n}^{T}||_{F}  ||Q_{l}\Phi_{l,j}^{T}(I-J)||_{F}\\
     \leq  ||\Phi_{n,j}^{T}(I-J)||_{op}||Q_{n}||_{F}  ||Q_{l}||_{op}||\Phi_{l,j}^{T}(I-J)||_{F}\\
     \textcolor{magenta}{(\because \text{Lemma 6 and }||A||_{F} = ||A^{T}||_{F})}\\
     \leq  ||\Phi_{n,j}^{T}(I-J)||_{F}||Q_{n}||_{F}  ||Q_{l}||_{op}||\Phi_{l,j}^{T}(I-J)||_{F}\\
     \textcolor{magenta}{(\because ||A||_{op} \leq ||A||_{F})}\\
     \leq  \sqrt{\delta}||Q_{n}||_{F}  ||Q_{l}||_{F}\sqrt{\delta}\\
     \textcolor{magenta}{(\because \text{Lemma 8 and 9})}\\
     \leq \frac{\delta}{2} [||Q_{n}||_{F}^{2} + ||Q_{l}||_{F}^{2}]\\
     \textcolor{magenta}{(\because \text{AM $\geq$ GM})}\\
    \end{gathered}
\end{equation*}
Thus, we have,
\begin{equation*}
    \begin{gathered}
    T2 \leq 2\eta^{2} \sum_{r=0}^{j} \mathbb{E} || Q_{r}\Phi_{r,j}^{T}(I-J) ||^{2}_{F} + \eta^{2}\delta  \sum_{n=0}^{j} \sum_{l=0, l\neq n}^{j}  \mathbb{E} \left[  ||Q_{n}||_{F}^{2} + ||Q_{l}||_{F}^{2}  \right]\\
    \leq 2\eta^{2} \sum_{r=0}^{j} \mathbb{E} || Q_{r}||^{2}_{op}||\Phi_{r,j}^{T}(I-J) ||^{2}_{F} + \eta^{2}\delta  \sum_{n=0}^{j} \sum_{l=0, l\neq n}^{j}  \mathbb{E} \left[  ||Q_{n}||_{F}^{2} + ||Q_{l}||_{F}^{2}  \right]\\
    \leq 2\eta^{2} \sum_{r=0}^{j} \mathbb{E} || Q_{r}||^{2}_{F}||\Phi_{r,j}^{T}(I-J) ||^{2}_{F} + \eta^{2}\delta  \sum_{n=0}^{j} \sum_{l=0, l\neq n}^{j}  \mathbb{E} \left[  ||Q_{n}||_{F}^{2} + ||Q_{l}||_{F}^{2}  \right]\\
    \leq 2\eta^{2}\delta \sum_{r=0}^{j} \mathbb{E} || Q_{r}||^{2}_{F} + \eta^{2}\delta  \sum_{n=0}^{j} \sum_{l=0, l\neq n}^{j}  \mathbb{E} \left[  ||Q_{n}||_{F}^{2} + ||Q_{l}||_{F}^{2}  \right]\\
    \end{gathered}
\end{equation*}
Now $||Q_{n}||$ and $||Q_{l}||$ are the same thing, so $\mathbb{E}||Q_{n}|| = \mathbb{E}||Q_{l}||$. Thus,
\begin{equation*}
    \begin{gathered}
    T2 \leq 2\eta^{2}\delta \sum_{r=0}^{j} \mathbb{E} || Q_{r}||^{2}_{F} + 2\eta^{2}\delta  \sum_{n=0}^{j} \sum_{l=0, l\neq n}^{j}  \mathbb{E} \left[  ||Q_{n}||_{F}^{2}  \right]\\
    \leq 2\eta^{2}\delta \sum_{r=0}^{j} \mathbb{E} || Q_{r}||^{2}_{F} + 2\eta^{2}\delta  \sum_{n=0}^{j}  \mathbb{E} \left[  ||Q_{n}||_{F}^{2}  \right]\sum_{l=0, l\neq n}^{j}1\\
    \leq 2\eta^{2}\delta \left[\sum_{r=0}^{j} \mathbb{E} || Q_{r}||^{2}_{F} +  \sum_{n=0}^{j}  \mathbb{E}  ||Q_{n}||_{F}^{2}  \sum_{l=0, l\neq n}^{j}1 \right]\\
    \leq 2\eta^{2}\delta \left[ \sum_{r=0}^{j-1} \mathbb{E} || Q_{r}||^{2}_{F} +   \sum_{n=0}^{j-1}  \mathbb{E}  ||Q_{n}||_{F}^{2}  \sum_{l=0, l\neq n}^{j}1 +  \mathbb{E} || Q_{j}||^{2}_{F} +  \mathbb{E}  ||Q_{j}||_{F}^{2}  \sum_{l=0, l\neq n}^{j}1 \right]\\
    \leq 2\eta^{2}\delta \left[ \sum_{r=0}^{j-1} \mathbb{E} || Q_{r}||^{2}_{F} +  \sum_{n=0}^{j-1}  \mathbb{E}  ||Q_{n}||_{F}^{2}  \sum_{l=0}^{j}1 +  \mathbb{E} || Q_{j}||^{2}_{F} +  \mathbb{E}  ||Q_{j}||_{F}^{2}  \sum_{l=0}^{j}1 \right]\\
    \textcolor{magenta}{(\because \text{More terms on RHS})}\\
    \leq 2\eta^{2}\delta \left[ \sum_{r=0}^{j-1} \mathbb{E} || Q_{r}||^{2}_{F} +  \sum_{n=0}^{j-1}  \mathbb{E}  ||Q_{n}||_{F}^{2}  (j+1) +  \mathbb{E} || Q_{j}||^{2}_{F} +  \mathbb{E}  ||Q_{j}||_{F}^{2}  (j+1) \right]\\
    \leq 2\eta^{2}\delta  \sum_{r=0}^{j-1} (1+(j+1)) \mathbb{E} || Q_{r}||^{2}_{F} +  2\eta^{2}\delta(1+ (j+1)) \mathbb{E} || Q_{j}||^{2}_{F} \\
    \leq 2\eta^{2}\delta  \sum_{r=0}^{j-1} (j+2) \mathbb{E} || \sum_{s=1}^{\tau} \nabla F(X_{r\tau+s})||^{2}_{F} +  2\eta^{2}\delta(j+2) \mathbb{E} || \sum_{s=1}^{i-1} \nabla F(X_{j\tau+s}) ||^{2}_{F} \\
    \end{gathered}
\end{equation*}
Now, By Jensen's Inequality, 
\begin{equation*}
    \begin{gathered}
     \varphi(\mathbb{E}[X]) \leq \mathbb{E}[\varphi(X)] \implies \text{ Take }\varphi = ||\cdot ||^{2}_{F} \text{ and } X = \nabla F(X_{r\tau+s}) \text{, then }
     \\||E[X]||^{2}_{F} = \left| \left| \frac{1}{\tau} \sum_{s=1}^{\tau} \nabla F(X_{r\tau+s})  \right| \right|^{2}_{F} \leq  \frac{1}{\tau} \sum_{s=1}^{\tau} \left| \left| \nabla F(X_{r\tau+s})  \right| \right|^{2}_{F} \\
    \implies \left| \left| \sum_{s=1}^{\tau} \nabla F(X_{r\tau+s})  \right| \right|^{2}_{F} \leq  \tau \sum_{s=1}^{\tau} \left| \left| \nabla F(X_{r\tau+s})  \right| \right|^{2}_{F}\\
    \end{gathered}
\end{equation*}
For second term in a similar way, 
\begin{equation*}
    \begin{gathered}
     \left| \left| \sum_{s=1}^{\tau} \nabla F(X_{r\tau+s})  \right| \right|^{2}_{F} \leq  (i-1) \sum_{s=1}^{\tau} \left| \left| \nabla F(X_{r\tau+s})  \right| \right|^{2}_{F}\\
    \end{gathered}
\end{equation*}
Thus, we get,
\begin{equation*}
    \begin{gathered}
    T2 \leq 2\eta^{2}\delta\tau  \sum_{r=0}^{j-1} \left[ (j+2)  \sum_{s=1}^{\tau}  \mathbb{E} \left| \left| \nabla F(X_{r\tau+s})  \right| \right|^{2}_{F} \right] +  2\eta^{2}\delta(j+2)(i-1) \sum_{s=1}^{i-1} \mathbb{E}  \left| \left| \nabla F(X_{r\tau+s})  \right| \right|^{2}_{F}
    \end{gathered}
\end{equation*}

Next, we sum over all iterates in $j^{th}$ period.
\begin{equation*}
    \begin{gathered}
    \sum_{i=1}^{\tau}T2 \leq 2\eta^{2}\delta\tau  \sum_{r=0}^{j-1} \left[ (j+2)  \sum_{s=1}^{\tau}  \mathbb{E} \left| \left| \nabla F(X_{r\tau+s})  \right| \right|^{2}_{F} \right] \sum_{i=1}^{\tau}1 \\+  2\eta^{2}\delta(j+2)\sum_{i=1}^{\tau}(i-1) \sum_{s=1}^{i-1} \mathbb{E}  \left| \left| \nabla F(X_{r\tau+s})  \right| \right|^{2}_{F}\\
    \end{gathered}
\end{equation*}
Now,
\begin{equation*}
    \begin{gathered}
    \sum_{i=1}^{\tau}(1-i) \sum_{s=1}^{i-1} \mathbb{E} \left|\left|  \nabla F(X_{j\tau + s}) \right|\right|_{F}^{2} \leq \frac{\tau(\tau-1)}{2} \sum_{s=1}^{\tau-1} \mathbb{E} \left|\left|  \nabla F(X_{j\tau + s}) \right|\right|_{F}^{2}\\
    \end{gathered}
\end{equation*}
Thus, we get,
\begin{equation*}
    \begin{gathered}
    \sum_{i=1}^{\tau}T2 \leq 2\eta^{2}\delta\tau^{2}  \sum_{r=0}^{j-1} \left[ (j+2)  \sum_{s=1}^{\tau}  \mathbb{E} \left| \left| \nabla F(X_{r\tau+s})  \right| \right|^{2}_{F} \right] \\+  \eta^{2}\delta\tau(\tau-1)(j+2) \sum_{s=1}^{\tau-1} \mathbb{E} \left|\left|  \nabla F(X_{j\tau + s}) \right|\right|_{F}^{2}\\
    \end{gathered}
\end{equation*}
Now summing over all periods ($j=0$ to $j=K/\tau - 1$), we get
\begin{equation*}
    \begin{gathered}
    \sum_{j=0}^{K/\tau-1}\sum_{i=1}^{\tau}T2 \leq 2\eta^{2}\delta\tau^{2}  \sum_{j=0}^{K/\tau-1}\sum_{r=0}^{j-1} \left[ (j+2)  \sum_{s=1}^{\tau}  \mathbb{E} \left| \left| \nabla F(X_{r\tau+s})  \right| \right|^{2}_{F} \right] \\
    +  \eta^{2}\delta\tau(\tau-1)\sum_{j=0}^{K/\tau-1}(j+2) \sum_{s=1}^{\tau-1} \mathbb{E} \left|\left|  \nabla F(X_{j\tau + s}) \right|\right|_{F}^{2}\\
    \end{gathered}
\end{equation*}
Now, we have,
\begin{equation*}
    \begin{gathered}
    \sum_{j=0}^{K/\tau-1}\sum_{r=0}^{j-1} \left[( j+2)  \sum_{s=1}^{\tau}  \mathbb{E} \left| \left| \nabla F(X_{r\tau+s})  \right| \right|^{2}_{F} \right] \leq S_{series}\sum_{j=0}^{K/\tau-1}\sum_{s=1}^{\tau} \mathbb{E} \left| \left| \nabla F(X_{j\tau+s})  \right| \right|^{2}_{F}\\
    \end{gathered}
\end{equation*}
Here, $S_{series}$ is given by the following method:
\begin{equation*}
    \begin{gathered}
    \sum_{j=0}^{K/\tau-1}\sum_{r=0}^{j-1} \left[ (j+2)  \sum_{s=1}^{\tau}  \mathbb{E} \left| \left| \nabla F(X_{r\tau+s})  \right| \right|^{2}_{F} \right]  =  \sum_{j=0}^{K/\tau-1}\sum_{r=0}^{j-1} \left[ j  \underbrace{\sum_{s=1}^{\tau}  \mathbb{E} \left| \left| \nabla F(X_{r\tau+s})  \right| \right|^{2}_{F}}_{V_{r}} \right]  \\
    + 2\sum_{j=0}^{K/\tau-1}\sum_{r=0}^{j-1} \left[  \underbrace{\sum_{s=1}^{\tau}  \mathbb{E} \left| \left| \nabla F(X_{r\tau+s})  \right| \right|^{2}_{F}}_{V_{r}} \right]\\
    = \sum_{j=0}^{K/\tau-1}\sum_{r=0}^{j-1}  jV_{r}  + 2\sum_{j=0}^{K/\tau-1}\sum_{r=0}^{j-1}  V_{r}
    \end{gathered}
\end{equation*}
The first term is broken down as:
\begin{equation*}
    \begin{gathered}
    \sum_{j=0}^{K/\tau-1}\sum_{r=0}^{j-1}  jV_{r}  = \cancel{\sum_{r=0}^{-1}0.V_{r}} + \sum_{r=0}^{0}1.V_{r} + \sum_{r=0}^{1}2.V_{r}... + \sum_{r=0}^{K/\tau-2}(K/\tau - 1).V_{r}\\
    = \left[ 1.V_{0} \right] + \left[ 2.V_{0} + 2.V_{1} \right] + \left[ 3.V_{0} + 3.V_{1} + 3.V_{2} \right] ... \\
    = \left[ 1+2+3...+(K/\tau - 1) \right]V_{0} + \left[ 2+3...+(K/\tau - 1) \right]V_{1} + \left[ 3...+(K/\tau - 1) \right]V_{2} ... + \left[ K/\tau - 1 \right]V_{K/\tau - 2}\\
    \leq \left[ 1+2...+(K/\tau - 1) \right]V_{0} + \left[ 1+2...+(K/\tau - 1) \right]V_{1} + ... + \left[ 1+2...+(K/\tau - 1) \right]V_{K/\tau - 2}\\
    \leq \left[ 1+2...+(K/\tau - 1) \right] (V_{0} + V_{1} + V_{2} ... +V_{K/\tau - 2})\\
    \leq \frac{K/\tau(K/\tau - 1)}{2} (V_{0} + V_{1} + V_{2} ... +V_{K/\tau - 2})\\
    \leq \frac{K/\tau(K/\tau - 1)}{2} \sum_{r=0}^{K/\tau-2}V_{r}\\
    \leq \frac{K/\tau(K/\tau - 1)}{2} \sum_{r=0}^{K/\tau-1}V_{r}\\
    \leq \frac{K/\tau(K/\tau - 1)}{2} \sum_{r=0}^{K/\tau-1}\sum_{s=1}^{\tau}  \mathbb{E} \left| \left| \nabla F(X_{r\tau+s})  \right| \right|^{2}_{F}\\
    \end{gathered}
\end{equation*}
The second term can be written as:
\begin{equation*}
    \begin{gathered}
    \sum_{j=0}^{K/\tau-1}\sum_{r=0}^{j-1}  V_{r} = \cancel{\sum_{r=0}^{-1}  V_{r}} + \sum_{r=0}^{0}  V_{r} + \sum_{r=0}^{1}  V_{r} + ... \sum_{r=0}^{K/\tau-2}  V_{r}\\
    = \left[ V_{0} \right] + \left[ V_{0} + V_{1} \right] + \left[ V_{0} + V_{1} + V{2} \right] + ...\\
    = \left[ K/\tau-1 \right]V_{0} + \left[ K/\tau-2 \right]V_{1}  + ...\\
    \leq \left( K/\tau - 1 \right) (V_{0} + V_{1} + V_{2} ... +V_{K/\tau - 2})\\
    \leq \left( K/\tau - 1 \right) \sum_{r=0}^{K/\tau-2}V_{r}\\
    \leq \left( K/\tau - 1 \right) \sum_{r=0}^{K/\tau-1}V_{r}\\
    \leq \left( K/\tau - 1 \right) \sum_{r=0}^{K/\tau-1}\sum_{s=1}^{\tau}  \mathbb{E} \left| \left| \nabla F(X_{r\tau+s})  \right| \right|^{2}_{F}
    \end{gathered}
\end{equation*}
Then, we can write $S_{series} = (K/\tau - 1) \left[ 2 +\frac{ K}{2\tau}\right]$.
Thus, we get,
\begin{equation*}
    \begin{gathered}
    \sum_{j=0}^{K/\tau-1}\sum_{i=1}^{\tau}T2 \leq 2\eta^{2}\delta\tau^{2}  S_{series}\sum_{j=0}^{K/\tau-1}\sum_{s=1}^{\tau} \mathbb{E} \left| \left| \nabla F(X_{j\tau+s})  \right| \right|^{2}_{F} \\+  \eta^{2}\delta \tau(\tau-1)\sum_{j=0}^{K/\tau-1}(j+2) \sum_{s=1}^{\tau-1} \mathbb{E} \left|\left|  \nabla F(X_{j\tau + s}) \right|\right|_{F}^{2}\\
    \end{gathered}
\end{equation*}
Now, 
\begin{equation*}
    \begin{gathered}
    \sum_{j=0}^{K/\tau-1}(j+2) \sum_{s=1}^{\tau-1} \mathbb{E} \left|\left|  \nabla F(X_{j\tau + s}) \right|\right|_{F}^{2} \leq (2 + (K/\tau -1))\sum_{j=0}^{K/\tau-1} \sum_{s=1}^{\tau} \mathbb{E}\left|\left|  \nabla F(X_{j\tau + s}) \right|\right|_{F}^{2}\\
    \end{gathered}
\end{equation*}

Thus, we get,
\begin{equation*}
    \begin{gathered}
    \sum_{j=0}^{K/\tau-1}\sum_{i=1}^{\tau}T2 \leq 2\eta^{2}\delta\tau^{2}  S_{series}\sum_{j=0}^{K/\tau-1}\sum_{s=1}^{\tau} \mathbb{E} \left| \left| \nabla F(X_{j\tau+s})  \right| \right|^{2}_{F} \\+  \eta^{2}\delta \tau(\tau-1)(1 +  K/\tau)\sum_{j=0}^{K/\tau-1} \sum_{s=1}^{\tau}\mathbb{E}\left|\left|  \nabla F(X_{j\tau + s}) \right|\right|_{F}^{2}\\
    \end{gathered}
\end{equation*}

Replacing $j\tau+s$ with $K$, we get,
\begin{equation*}
    \begin{gathered}
    \sum_{j=0}^{K/\tau-1}\sum_{i=1}^{\tau}T2 \leq 2\eta^{2}\delta\tau^{2}  S_{series}\sum_{k=1}^{K} \mathbb{E} \left| \left| \nabla F(X_{k})  \right| \right|^{2}_{F} +  \eta^{2}\delta \tau(\tau-1)(1 + K/\tau)\sum_{k=1}^{K}\mathbb{E}\left|\left|  \nabla F(X_{k}) \right|\right|_{F}^{2}\\
    \implies \sum_{j=0}^{K/\tau-1}\sum_{i=1}^{\tau}T2 \leq \eta^{2}\delta\tau\left[ 2\tau S_{series} + (\tau-1)(1 + K/\tau)\right] \sum_{k=1}^{K}\mathbb{E}\left|\left|  \nabla F(X_{k}) \right|\right|_{F}^{2}\\
    \end{gathered}
\end{equation*}

\subsubsection{Final Result} \label{t2}
\begin{equation*}
    \begin{gathered}
    \mathbb{E} ||X_{k}(I-J)||^{2}_{F} = \leq \underbrace{\mathbb{E} \left|\left|X_{1}\Phi_{1,k-1}^{T}(I-J)\right|\right|^{2}_{F}}_{T0} + \underbrace{2\eta^{2} \mathbb{E} \left|\left|\sum_{r=0}^{j} (Y_{r}-Q_{r})\Phi_{r,j}^{T}(I-J)\right|\right|^{2}_{F}}_{\text{T1}} + \underbrace{2\eta^{2} \mathbb{E} \left|\left|\sum_{r=0}^{j} Q_{r}\Phi_{r,j}^{T}(I-J)\right|\right|^{2}_{F}}_{\text{T2}}
    \end{gathered}
\end{equation*}
Now taking summation over all periods, we get,
\begin{equation*}
    \begin{gathered}
    \frac{1}{Kcm}\sum_{k=1}^{K}\mathbb{E} ||X_{k}(I-J)||^{2}_{F}  \leq \frac{1}{Kcm}\sum_{k=1}^{K} (T0 + T1 + T2)\\
    \implies \frac{1}{Kcm}\sum_{k=1}^{K}\mathbb{E} ||X_{k}(I-J)||^{2}_{F} \leq \frac{1}{Kcm}\sum_{j=0}^{K/\tau-1}\sum_{i=1}^{\tau} (T0+ T1 + T2)\\
    \implies \frac{L^{2}}{Kcm}\sum_{k=1}^{K}\mathbb{E} ||X_{k}(I-J)||^{2}_{F} \leq  \frac{L^{2}K}{Kcm} \left( \delta\left|\left|X_{1}\right|\right|^{2}_{F} \right) + \frac{L^{2}}{Kcm} \left( \eta^{2}cm \sigma^{2}\delta K(K-1) \right) \\
    + \frac{L^{2}}{Kcm} \left(\eta^{2}\delta\tau\left[ 2\tau S_{series} + (\tau-1)(1 +  K/\tau)\right] \sum_{k=1}^{K}\mathbb{E}\left|\left|  \nabla F(X_{k}) \right|\right|_{F}^{2}  \right)\\
    \leq \underbrace{\frac{\delta L^{2}}{cm}\left|\left|X_{1}\right|\right|^{2}_{F}}_{\text{Contribution of T0}} 
    +  \underbrace{\eta^{2}\sigma^{2}L^{2}\delta (K-1)}_{\text{Contribution of T1}}  \\+\underbrace{\frac{\eta^{2}\delta\tau L^{2}}{Kcm} \left[ 2\tau S_{series} + (\tau-1)(1 + K/\tau)\right] \sum_{k=1}^{K}\mathbb{E}\left|\left|  \nabla F(X_{k}) \right|\right|_{F}^{2}}_{\text{Contribution of T2}}  \\
    \end{gathered}
\end{equation*}

Now,
\begin{equation*}
    \begin{gathered}
    \mathbb{E}||\nabla F(X_{k}) ||^{2}_{F} = \sum_{i=1}^{m}\mathbb{E}||\nabla F_{i}(x_{k}^{(i)}) ||^{2} 
    = \sum_{i \in C_{K}}\mathbb{E}||\nabla F_{i}(x_{k}^{(i)}) ||^{2}\\
    \leq 3 \sum_{i \in C_{K}} \mathbb{E}\left[||\nabla F_{i}(x_{k}^{(i)}) - \nabla F_{i}(u_{k})||^{2} + ||\nabla F_{i}(u_{k}) - \nabla F(u_{k})||^{2} + ||\nabla F(u_{k})||^{2}\right]\\
    \textcolor{magenta}{(\because ||a+b+c||^{2} \leq 3(||a||^{2} + ||b||^{2} + ||c||^{2}))}\\
    \leq 3\mathbb{E}||X_{k}(I-J)||^{2} + 3cm\mathbb{E}||\nabla F(u_{k})||^{2}\\
    \textcolor{magenta}{(\because \text{client gradients approximate the global gradients})}\\
    \end{gathered}
\end{equation*}

Thus,
\begin{equation*}
    \begin{gathered}
    \frac{L^{2}}{Kcm}\sum_{k=1}^{K}\mathbb{E} ||X_{k}(I-J)||^{2}_{F}  \leq \frac{\delta L^{2}}{cm}\left|\left|X_{1}\right|\right|^{2}_{F} + \eta^{2}\sigma^{2}L^{2}\delta (K-1)
    \\+\left[ \frac{\eta^{2}\delta\tau L^{2}}{Kcm} \left[ 2\tau S_{series} + (\tau-1)(1 + K/\tau)\right] \right]\sum_{k=1}^{K}\mathbb{E}\left|\left|  \nabla F(X_{k}) \right|\right|_{F}^{2}  \\
    \implies \frac{L^{2}}{Kcm}\sum_{k=1}^{K}\mathbb{E} ||X_{k}(I-J)||^{2}_{F}  \leq \frac{\delta L^{2}}{cm}\left|\left|X_{1}\right|\right|^{2}_{F} + \eta^{2}\sigma^{2}L^{2}\delta (K-1)
    \\+\frac{L^{2}}{Kcm}\left[ \underbrace{{\eta^{2}\delta\tau} \left[ 2\tau S_{series} + (\tau-1)(1+ K/\tau )\right]}_{P} \right]\sum_{k=1}^{K}\mathbb{E}\left|\left|  \nabla F(X_{k}) \right|\right|_{F}^{2} \\ 
    \implies \frac{L^{2}}{Kcm}\sum_{k=1}^{K}\mathbb{E} ||X_{k}(I-J)||^{2}_{F}  \leq \frac{\delta L^{2}}{cm}\left|\left|X_{1}\right|\right|^{2}_{F} + \eta^{2}\sigma^{2}L^{2}\delta (K-1)
    +\frac{3PL^{2}}{Kcm}\sum_{k=1}^{K}\mathbb{E}||X_{k}(I-J)||^{2} \\+\frac{3PL^{2}}{Kcm}\sum_{k=1}^{K}cm\mathbb{E}||\nabla F(u_{k})||^{2} \\
    \implies \left(1-  3P \right)\frac{L^{2}}{Kcm}\sum_{k=1}^{K}\mathbb{E} ||X_{k}(I-J)||^{2}_{F}  \leq \frac{\delta L^{2}}{cm}\left|\left|X_{1}\right|\right|^{2}_{F} + \eta^{2}\sigma^{2}L^{2}\delta (K-1)
    +\frac{3PL^{2}}{Kcm}\sum_{k=1}^{K}cm\mathbb{E}||\nabla F(u_{k})||^{2}\\
    \implies 
    \frac{L^{2}}{Kcm}\sum_{k=1}^{K}\mathbb{E} ||X_{k}(I-J)||^{2}_{F}  \leq \frac{1}{1-3P}\left[\frac{\delta L^{2}}{cm}\left|\left|X_{1}\right|\right|^{2}_{F} + \eta^{2}\sigma^{2}L^{2}\delta (K-1)
    +\frac{3PL^{2}}{Kcm}\sum_{k=1}^{K}cm\mathbb{E}||\nabla F(u_{k})||^{2}\right]\\
    \textcolor{magenta}{\left(\text{Assuming $P<\frac{1}{3}$}\right)}\\
    \end{gathered}
\end{equation*}

Going all the way back, we now substitute in Equation \ref{eqn9}.
\begin{equation*}
    \begin{gathered}
      \mathbb{E}\left[ \frac{1}{K}\sum_{k=1}^{K}||\nabla F(u_{k})||^{2} \right] \leq \frac{2[ F(u_{1}) - F_{inf}]}{\eta_{eff}K} + \frac{\eta_{eff}L\sigma^{2}}{cm} + \frac{L^{2}}{Kcm}\sum_{k=1}^{K}\mathbb{E}_{k}\left|\left| X_{k}(I - J)\right|\right|^{2}_{F} \\ - \left[ 1 - \eta_{eff}L  \right]\frac{1}{K}\sum_{k=1}^{K}\frac{\mathbb{E}||\nabla F(X_{k})||^{2}_{F}}{cm}\\
      \implies \mathbb{E}\left[ \frac{1}{K}\sum_{k=1}^{K}||\nabla F(u_{k})||^{2} \right] \leq \frac{2[ F(u_{1}) - F_{inf}]}{\eta_{eff}K} + \frac{\eta_{eff}L\sigma^{2}}{cm} +\\
    \frac{1}{1-3P}\left[ \frac{\delta L^{2}}{cm}\left|\left|X_{1}\right|\right|^{2}_{F} + \eta^{2}\sigma^{2}L^{2}\delta (K-1)
    +\frac{3PL^{2}}{K}\sum_{k=1}^{K}\mathbb{E}||\nabla F(u_{k})||^{2}\right] \\
    - \left[ 1 - \eta_{eff}L  \right]\frac{1}{K}\sum_{k=1}^{K}\frac{\mathbb{E}||\nabla F(X_{k})||^{2}_{F}}{cm}\\
    \end{gathered}
\end{equation*}
When learning rate satisfies
\begin{equation*}
    \begin{gathered}
         \eta_{eff}L - 1 \leq 0
    \end{gathered}
\end{equation*}
Then,
\begin{equation*}
    \begin{gathered}
      \mathbb{E}\left[ \frac{1}{K}\sum_{k=1}^{K}||\nabla F(u_{k})||^{2} \right] \leq \frac{2[ F(u_{1}) - F_{inf}]}{\eta_{eff}K} + \frac{\eta_{eff}L\sigma^{2}}{cm} \\
    + \frac{1}{1-3P}\left[ \frac{\delta L^{2}}{cm}\left|\left|X_{1}\right|\right|^{2}_{F} + \eta^{2}\sigma^{2}L^{2}\delta (K-1)
    +\frac{3PL^{2}}{K}\sum_{k=1}^{K}\mathbb{E}||\nabla F(u_{k})||^{2}\right]\\
    \implies \left( 1- \frac{3PL^{2}}{1-3P}\right)\mathbb{E}\left[ \frac{1}{K}\sum_{k=1}^{K}||\nabla F(u_{k})||^{2} \right] \leq \frac{2[ F(u_{1}) - F_{inf}]}{\eta_{eff}K} + \frac{\eta_{eff}L\sigma^{2}}{cm} +  \frac{\delta L^{2}}{cm(1-3P)}\left|\left|X_{1}\right|\right|^{2}_{F} + \frac{\eta^{2}\sigma^{2}L^{2}\delta (K-1)}{1-3P}
    \end{gathered}
\end{equation*}

\subsubsection{Case 1: $\delta=0$}\label{theorem1:res-sec1}
Consider the scenario where we have $W = J$ for all iterations. We have to rewrite Lemma 8 for this. We know the following inequality holds:
\begin{equation*}
    \begin{gathered}
    ||A||_{F} \leq \sqrt{m}||A||_{op}
    \end{gathered}
\end{equation*}
Thus, we have, 
\begin{equation}
    ||J^{k}-J||_{F} \leq \sqrt{m+v}||J^{k}-J||_{op}
\end{equation}
Now, following an analysis similar to \citet{Wang2021}, we find that 
\begin{equation}
    ||J^{k}-J||_{op} = \varsigma^{k}
\end{equation}
But we know that for matrix $J$, $\varsigma = 0$. Also, we know that frobenius norm is greater than zero. Thus, we can conclude that $||J^{k}-J||_{F} = \delta$ where $\delta = 0$. In reality, this would imply simply substituting the value of $\delta = 0$ at all points. Thus, we would have $P=0$. This would lead to the scenario that:
\begin{equation*}
    \begin{gathered}
      \mathbb{E}\left[ \frac{1}{K}\sum_{k=1}^{K}||\nabla F(u_{k})||^{2} \right] \leq \frac{2[ F(u_{1}) - F_{inf}]}{\eta_{eff}K} + \frac{\eta_{eff}L\sigma^{2}}{cm}
    \end{gathered}
\end{equation*}

\subsubsection{Case 2: $1\geq\delta>0$}\label{theorem1:res-sec2}
In this scenario, as $\delta > 0$, we have $P>0$. Let us look at the constant term we had got previously and simplify it.
\begin{equation*}
    \begin{gathered}
      \left[ 1 - \frac{3PL^{2}}{\left( 1-3P\right)} \right] = \left[ \frac{(1-3P) -3PL^{2}}{\left( 1-3P\right)} \right]  = \left[ \frac{1-3P\left( 1 + L^{2} \right)}{ 1-3P} \right]
    \end{gathered}
\end{equation*}

Let the following condition be satisfied, so that our calculations are simplified.
\begin{equation*}
    \begin{gathered}
     \frac{ 1-3P}{1-3P\left( 1 + L^{2} \right)} \leq 2 \\
     \implies 1-3P\left( 1 + L^{2} \right) \geq \frac{1}{2} (1-3P)\\
     \implies (1-3P) - 3PL^{2} \geq \frac{1}{2} (1-3P)\\
     \implies \frac{1}{2}(1-3P) \geq 3PL^{2}\\
     \implies (1-3P) \geq 6PL^{2}\\
     \implies P \leq \frac{1}{6L^{2}+3}\\
    \end{gathered}
\end{equation*}
 Furthermore, let $0 \leq P \leq \frac{1}{6} \implies \frac{1}{1-3P} \leq 2$. Then, we can say,
\begin{equation*}
    \begin{gathered}
    \mathbb{E}\left[ \frac{1}{K}\sum_{k=1}^{K}||\nabla F(u_{k})||^{2} \right] \leq 2\left[\frac{2[ F(u_{1}) - F_{inf}]}{\eta_{eff}K} + \frac{\eta_{eff}L\sigma^{2}}{cm} +  \frac{\delta L^{2}}{cm(1-3P)}\left|\left|X_{1}\right|\right|^{2}_{F} + \frac{\eta^{2}\sigma^{2}L^{2}\delta (K-1)}{1-3P}\right]\\
      \implies \mathbb{E}\left[ \frac{1}{K}\sum_{k=1}^{K}||\nabla F(u_{k})||^{2} \right] \leq 2\left[
    \begin{aligned}
        \frac{2[ F(u_{1}) - F_{inf}]}{\eta_{eff}K} + \frac{\eta_{eff}L\sigma^{2}}{cm} +  \frac{2\delta L^{2}}{cm}\left|\left|X_{1}\right|\right|^{2}_{F} + 2\eta^{2}\sigma^{2}L^{2}\delta (K-1)
    \end{aligned}
    \right]\\
    \implies \mathbb{E}\left[ \frac{1}{K}\sum_{k=1}^{K}||\nabla F(u_{k})||^{2} \right] \leq 4\left[\frac{2[ F(u_{1}) - F_{inf}]}{\eta_{eff}K} + \frac{\eta_{eff}L\sigma^{2}}{cm} +  \frac{\delta L^{2}}{cm}\left|\left|X_{1}\right|\right|^{2}_{F} + \eta^{2}\sigma^{2}L^{2}\delta (K-1)\right]\\
    \end{gathered}
\end{equation*}
Taking this criteria for P ensures that our previous assumption of $P<\frac{1}{3}$ also holds. This holds for Case 3 (Section \ref{theorem1:res-sec3}) as well.

If we further have the condition that $\delta \leq \frac{\tau}{K-1}$, then we can say
\begin{equation*}
    \begin{gathered}
    \mathbb{E}\left[ \frac{1}{K}\sum_{k=1}^{K}||\nabla F(u_{k})||^{2} \right] \leq 4\left[\frac{2[ F(u_{1}) - F_{inf}]}{\eta_{eff}K} + \frac{\eta_{eff}L\sigma^{2}}{cm} +  \frac{\delta L^{2}}{cm}\left|\left|X_{1}\right|\right|^{2}_{F} + \eta^{2}\sigma^{2}L^{2}\tau\right]\\
    \end{gathered}
\end{equation*}
This somewhat matches the result from \citep{Wang2021} (ignoring the model initialization error and constant factor). In fact, if we compare with their criterion:
\begin{equation*}
    \begin{gathered}
    \textcolor{blue}{\mathbb{E}\left[ \frac{1}{K}\sum_{k=1}^{K}||\nabla F(u_{k})||^{2} \right] \leq \frac{2[ F(u_{1}) - F_{inf}]}{\eta_{eff}K} + \frac{\eta_{eff}L\sigma^{2}}{cm} + \eta^{2}\sigma^{2}L^{2}\left(\frac{1+\varsigma^{2}}{1-\varsigma^{2}}\tau - 1\right)}
    \end{gathered}
\end{equation*}
We find that there is only one major difference: $\left(\frac{1+\varsigma^{2}}{1-\varsigma^{2}}\tau - 1\right)$ v/s $\delta(K-1)$. For our error bound to be smaller than theirs, we need $\left(\frac{1+\varsigma^{2}}{1-\varsigma^{2}}\tau - 1\right) > \delta(K-1)$. Furthermore, we have already assumed that $\delta (K-1)  \leq \tau$, thus if we have $\tau < \left(\frac{1+\varsigma^{2}}{1-\varsigma^{2}}\tau - 1\right)$ then we can say our error bound is smaller. For this situation to arise, we need to ensure $\tau > \frac{1-\varsigma^{2}}{2\varsigma^{2}}$. A thing to note here is that  $\varsigma > 0$ as it is zero in Case 1 (Section \ref{theorem1:res-sec1}).\\

When does this hold? Let us assume $\frac{1-\varsigma^{2}}{2\varsigma^{2}} = \theta \implies \varsigma = \frac{1}{2\theta + 1}$. Thus, when eigenvalues become very small, then we need to ensure that our communication period $\tau$ is sufficiently large. For example, as per above condition to have $\tau > 2$, we need $\varsigma = \frac{1}{5}$, and to have $\tau > 3$, we need $\varsigma = \frac{1}{7}$. If this condition does not hold, then it is not certain whose analysis gives a tighter bound. To be on the safe side, we proceed by assuming $\tau > \frac{1-\varsigma^{2}}{2\varsigma^{2}}$, so that our method always gives a better result, provided $\delta < \frac{\tau}{K-1}$.\\

If the latter bound does not hold, then we have two conditions needed for our bound to be smaller - $\left(\frac{1+\varsigma^{2}}{1-\varsigma^{2}}\tau - 1\right) > \delta(K-1)$ and $\delta(K-1) > \tau$. Then, we get
\begin{equation*}
    \begin{gathered}
    \delta(K-1) < \frac{1+\varsigma^{2}}{1-\varsigma^{2}}\tau - 1 < \frac{1+\varsigma^{2}}{1-\varsigma^{2}}\delta(K-1) - 1\\
    \implies 1 < \left(\frac{1+\varsigma^{2}}{1-\varsigma^{2}} - 1\right)\delta(K-1)\\
    \implies \frac{1-\varsigma^{2}}{2\varsigma^{2}} < \delta(K-1)
    \end{gathered}
\end{equation*}
Again, by ensuring $\delta(K-1) > \tau > \frac{1-\varsigma^{2}}{2\varsigma^{2}}$, we can ensure even in this scenario we get a better error bound. So, irrespective of the value of $\delta(K-1)$, if we ensure 
$\tau > \frac{1-\varsigma^{2}}{2\varsigma^{2}}$, we are assured in getting a lower error bound than \citep{Wang2021}.

\subsubsection{Case 3: $\delta>1$}\label{theorem1:res-sec3}
In this scenario, we can only get,
\begin{equation*}
    \begin{gathered}
    \mathbb{E}\left[ \frac{1}{K}\sum_{k=1}^{K}||\nabla F(u_{k})||^{2} \right] \leq 4\left[\frac{2[ F(u_{1}) - F_{inf}]}{\eta_{eff}K} + \frac{\eta_{eff}L\sigma^{2}}{cm} +  \frac{\delta L^{2}}{cm}\left|\left|X_{1}\right|\right|^{2}_{F} + \eta^{2}\sigma^{2}L^{2}\delta (K-1)\right]\\
    \end{gathered}
\end{equation*}
There is no clear optimization possible for this scenario.

\subsubsection{Lower Bound on the Fraction of Client's Selected $c$}
We know,
\begin{equation*}
    \begin{gathered}
    \mathbb{E}||\nabla F(X_{k}) ||^{2}_{F} \leq 3\mathbb{E}||X_{k}(I-J)||^{2} + 3cm\mathbb{E}||\nabla F(u_{k})||^{2}
    \end{gathered}
\end{equation*}
Taking a looser upper bound, we get,
\begin{equation*}
    \begin{gathered}
    \mathbb{E}||\nabla F(X_{k}) ||^{2}_{F} \leq 3\mathbb{E}||X_{k}(I-J)||^{2} + 3m\mathbb{E}||\nabla F(u_{k})||^{2}\\
    \end{gathered}
\end{equation*}
Now, using the looser upper bound we get,
\begin{equation*}
    \begin{gathered}
      \mathbb{E}\left[ \frac{1}{K}\sum_{k=1}^{K}||\nabla F(u_{k})||^{2} \right] \leq \frac{2[ F(u_{1}) - F_{inf}]}{\eta_{eff}K} + \frac{\eta_{eff}L\sigma^{2}}{cm} \\+ \frac{1}{1-3P}\left[ \frac{\delta L^{2}}{cm}\left|\left|X_{1}\right|\right|^{2}_{F} + \eta^{2}\sigma^{2}L^{2}\delta (K-1)
    +\frac{3PL^{2}}{Kc}\sum_{k=1}^{K}\mathbb{E}||\nabla F(u_{k})||^{2}\right]\\
    \implies \left( 1- \frac{3PL^{2}}{c(1-3P)}\right)\mathbb{E}\left[ \frac{1}{K}\sum_{k=1}^{K}||\nabla F(u_{k})||^{2} \right] \leq
    \left[
    \begin{aligned}
        \frac{2[ F(u_{1}) - F_{inf}]}{\eta_{eff}K} + \frac{\eta_{eff}L\sigma^{2}}{cm} \\+  \frac{\delta L^{2}}{cm(1-3P)}\left|\left|X_{1}\right|\right|^{2}_{F} + \frac{\eta^{2}\sigma^{2}L^{2}\delta (K-1)}{1-3P}
    \end{aligned} \right]
    \end{gathered}
\end{equation*}
If we have $\delta=0 \implies P = 0$, as discussed previously, we get (without any extra restriction on the value of c),
\begin{equation*}
    \begin{gathered}
      \mathbb{E}\left[ \frac{1}{K}\sum_{k=1}^{K}||\nabla F(u_{k})||^{2} \right] \leq \frac{2[ F(u_{1}) - F_{inf}]}{\eta_{eff}K} + \frac{\eta_{eff}L\sigma^{2}}{cm} \\
    \end{gathered}
\end{equation*}

When $\delta>0 \implies P>0$. In that case, we can simplify the constant term as
\begin{equation*}
    \begin{gathered}
      \left[ 1 - \frac{3PL^{2}}{c\left( 1-3P\right)} \right] = \left[ \frac{c(1-3P) -3PL^{2}}{c\left( 1-3P\right)} \right]  = \left[ \frac{1-3P\left( 1 + \frac{L^{2}}{c} \right)}{ 1-3P} \right] \\
    \end{gathered}
\end{equation*}

If the following condition is satisfied, we get a lower bound on the fraction of client's selected.
\begin{equation*}
    \begin{gathered}
     \frac{ 1-3P}{1-3P\left( 1 + \frac{L^{2}}{c} \right)} \leq 2 \\
     \implies 1-3P\left( 1 + \frac{L^{2}}{c} \right) \geq \frac{1}{2} (1-3P)\\
     \implies (1-3P) - \frac{3PL^{2}}{c} \geq \frac{1}{2} (1-3P)\\
     \implies \frac{1}{2}(1-3P) \geq \frac{3PL^{2}}{c}\\
     \implies (1-3P) \geq \frac{6PL^{2}}{c}\\
     \implies c \geq \frac{6PL^{2}}{1-3P}\\
     \implies c \geq 6PL^{2}\\
     \textcolor{magenta}{\left(\because \frac{1}{1-3P} \geq 1\right)}
    \end{gathered}
\end{equation*}

From this, we get the condition that $c \geq 6PL^{2}$, thus giving a criteria for minimum number of clients selected. As we already know that $P \leq \frac{1}{6L^{2}+3}$, thus satisfying the above criterion for $c$ doesn't violate $c\leq 1$.

\section{Proof of Theorem 2} \label{NIID-Theorem2}
In this section we will prove Theorem 2. We state all the assumptions for non-IID scenario  (Section \ref{niid:assumptions}). We thereafter lay the groundwork for proving Lemma 2 for the non-IID scenario by modifying Lemma 3 (Section \ref{niid:lemma3}), Lemma 4 (Section \ref{niid:lemma4}) and Lemma 5 (Section \ref{niid:lemma5}) to suit non-IID assumptions. In Section \ref{niid:lemma2} we prove Lemma 2, modified for the non-IID case. Thereafter, we reuse the other 2 lemmas (Lemma 7 (Section \ref{niid:lemma7}), Lemma 8 (Section \ref{niid:lemma8})) from the proof of Theorem 1 as they do not change for this scenario. We prove Theorem 2 in Section \ref{theorem2}.
\subsection{Assumptions} \label{niid:assumptions}
\begin{enumerate}
    \item \textbf{Smoothness}: $||\nabla F_{i}(x) - \nabla F_{i}(y)|| \leq L||x-y||$
    \item \textbf{Lower Bounded}: $F(x) \geq F_{ing}$
    \item \textbf{Unbiased Gradients}: $\mathbb{E}_{\xi|x}[g_{i}(x)] = \nabla F_{i}(x)$
    \item \textbf{Bounded Variance}: $\mathbb{E}_{\xi|x}||g_{i}(x)-\nabla F_{i}(x)||^{2} \leq  \sigma^{2}$\newline
    where $\sigma$ is some non-negative constant and in inverse proportion to batch size
    \item \textbf{Mixing Matrix}:
         $W_{k}^{T}1_{m+v} = 1_{m+v}$
    \item For all rounds, fraction of clients selected $c$ is fixed.
    \item \textbf{Bounded Dissimilarities}: $\frac{1}{m}\sum_{i=1}^{m}||\nabla F_{i}(x)-\nabla F(x)||^{2} \leq  \kappa^{2}$\newline
    where $\kappa$ is some non-negative constant
\end{enumerate}

\subsection{Lemma 2} \label{niid:lemma2_statement}
\textcolor{red}{\textbf{For algorithm $\mathcal{A}(\tau, W, v)$, under assumption 1-5, if learning rate satisfies $\eta_{eff}L\left(1+\frac{\beta}{m}\right) \leq 1$ and all local parameters are initialized at same point $u_{1}$, then the average-squared gradient after K iterations is bounded as follows:
\begin{equation}
    \mathbb{E} \left[ \frac{1}{K}\sum_{k=1}^{K}||\nabla F(u_{k}) ||^{2} \right] \leq \underbrace{\frac{2[F(u_{1}) - F_{inf}]}{\eta_{eff}K} + \frac{\eta_{eff}L\sigma^{2}}{cm}}_{\text{Fully Sync. SGD}} + \underbrace{\frac{L^{2}}{K}\sum_{k=1}^{K}\frac{\mathbb{E}||X_{k}(I-J)||^{2}_{F}}{cm}}_{\text{Network Error}}
\end{equation}}}
\subsection{Lemmas to prove for Lemma 2}
\subsubsection{Lemma 3} \label{niid:lemma3}
\textcolor{red}{\textbf{Under assumption 3 and 4, we have the following variance bound for averaged stochastic gradient:
\begin{equation}
    \mathbb{E}_{\equiv_{K}|X_{K}}[||\mathcal{G}_{k} - \mathcal{H}_{k}||^{2}] \leq  \frac{\sigma^{2}}{cm}
\end{equation}}}
\textit{Proof}:
\begin{equation*}
    \begin{gathered}
        \mathbb{E}_{\equiv_{K}|X_{K}}[||\mathcal{G}_{k} - \mathcal{H}_{k}||^{2}]\\
        = \mathbb{E}_{\equiv_{K}|X_{K}} \left| \left| \frac{1}{cm}\sum_{i \in C_{k}}[g_{i}(x_{k}^{(i)}) - \nabla F_{i}(x_{k}^{(i)})]\right| \right| ^{2}\\
        = \frac{1}{(cm)^{2}}\mathbb{E}_{\equiv_{K}|X_{K}} \left[ \sum_{i \in C_{k}}||g_{i}(x_{k}^{(i)}) - \nabla F_{i}(x_{k}^{(i)})||^{2} + \sum_{j,l \in C_{k}, j\neq l} \left< g_{j}(x_{k}^{(j)}) - \nabla F_{j}(x_{k}^{(j)}),g_{l}(x_{k}^{(l)}) - \nabla F_{l}(x_{k}^{(l)}) \right> \right]\\
        = \frac{1}{(cm)^{2}} \sum_{i \in C_{k}} \mathbb{E}_{\xi_{K}^{(i)}|X_{K}} ||g_{i}(x_{k}^{(i)}) - \nabla F_{i}(x_{k}^{(i)})||^{2} \\+ \frac{1}{(cm)^{2}}\sum_{j,l \in C_{k}, j\neq l} \left< \mathbb{E}_{\xi_{K}^{(j)}|X_{K}} [g_{j}(x_{k}^{(j)}) - \nabla F_{j}(x_{k}^{(j)})] , \mathbb{E}_{\xi_{k}^{(l)}|X_{K}} [g_{l}(x_{k}^{(l)}) - \nabla F_{l}(x_{k}^{(l)})] \right>\\
    \end{gathered}
\end{equation*}
This step was possible as $\equiv_{K}|X_{K}$ can be seen as over batches of all clients, so as we treat $\{\xi_{k}^{(i)}\}$ as independent random variables, for a particular client it should only depend on their own batch. We also use the equation $\mathbb{E}<X, Y> = <\mathbb{E}[X], \mathbb{E}[Y]> $ provided X and Y are independent. As we treat batches as independent, hence $\mathbb{E}_{\xi_{K}^{(j)}|X_{K}} [g_{j}(x_{k}^{(j)}) - \nabla F_{j}(x_{k}^{(j)})]$ is independent of $\mathbb{E}_{\xi_{k}^{(l)}|X_{K}} [g_{l}(x_{k}^{(l)}) - \nabla F_{l}(x_{k}^{(l)})]$.\\

Now by assumption 3, we have
\begin{equation*}
    \begin{gathered}
        \mathbb{E}_{\xi_{K}^{(j)}|X_{K}} [g_{j}(x_{k}^{(j)}) - \nabla F_{j}(x_{k}^{(j)})] = \mathbb{E}_{\xi_{K}^{(j)}|X_{K}} [g_{j}(x_{k}^{(j)})] - \nabla F_{j}(x_{k}^{(j)})\\
        \textcolor{magenta}{\text{($\because$ Since $\nabla F$ is not dependent on batches)}}\\
        \implies \nabla F_{j}(x_{k}^{(j)}) - \nabla F_{j}(x_{k}^{(j)}) = 0
    \end{gathered}
\end{equation*}

Now by assumption 4, we have
\begin{equation*}
    \begin{gathered}
        \mathbb{E}_{\xi_{K}^{(i)}|X_{K}} ||g_{i}(x_{k}^{(i)}) - \nabla F_{i}(x_{k}^{(i)})||^{2} \leq \sigma^{2}
    \end{gathered}
\end{equation*}

Thus, we get
\begin{equation*}
    \begin{gathered}
        \frac{1}{(cm)^{2}} \sum_{i \in C_{k}} \mathbb{E}_{\xi_{K}^{(i)}|X_{K}} ||g_{i}(x_{k}^{(i)}) - \nabla F_{i}(x_{k}^{(i)})||^{2} \\+ \frac{1}{(cm)^{2}}\sum_{j,l \in S_{k},j\neq l} \left< \mathbb{E}_{\xi_{K}^{(j)}|X_{K}} [g_{j}(x_{k}^{(j)}) - \nabla F_{j}(x_{k}^{(j)})] , \mathbb{E}_{\xi_{k}^{(l)}|X_{K}} [g_{l}(x_{k}^{(l)}) - \nabla F_{l}(x_{k}^{(l)})] \right> \\
        \leq \frac{1}{(cm)^{2}} \sum_{i \in C_{k}} \left[ \sigma^{2} \right]\\
        \implies \mathbb{E}_{\equiv_{K}|X_{K}}[||\mathcal{G}_{k} - \mathcal{H}_{k}||^{2}] \leq \frac{\sigma^{2}}{cm}\\
    \end{gathered}
\end{equation*}

\subsubsection{Lemma 4} \label{niid:lemma4}
\textcolor{red}{\textbf{Under assumption 3, the expected inner product between stochastic gradient and full batch gradient for averaged model can be expanded as:
\begin{equation}
    \mathbb{E}_{\equiv_{K}|X_{K}}\left[ \left< \nabla F(u_{k}), \mathcal{G}_{k}  \right> \right] \geq \frac{1}{2}||\nabla F(u_{k})||^{2} + \frac{1}{2}||\mathcal{H}_{k}||^{2} - \frac{L^{2}}{2cm}||X_{k}(I-J)||^{2}_{F}
\end{equation}}}
\textit{Proof}:
\begin{equation*}
    \begin{gathered}
    \mathbb{E}_{\equiv_{K}|X_{K}}\left[ \left< \nabla F(u_{k}), \mathcal{G}_{k}  \right> \right] = \mathbb{E}_{\equiv_{K}|X_{K}}\left[ \left< \nabla F(u_{k}), \frac{1}{cm} \sum_{i \in C_{k}} g_{i}(x_{k}^{(i)})  \right> \right]\\
    \implies \mathbb{E}_{\equiv_{K}|X_{K}}\left[ \left< \nabla F(u_{k}), \mathcal{G}_{k}  \right> \right] = \left< \mathbb{E}_{\equiv_{K}|X_{K}}[\nabla F(u_{k})],  \frac{1}{cm} \sum_{i \in C_{k}} \mathbb{E}_{\equiv_{K}|X_{K}} \left[g_{i}(x_{k}^{(i)})\right]  \right> \\
    \implies \mathbb{E}_{\equiv_{K}|X_{K}}\left[ \left< \nabla F(u_{k}), \mathcal{G}_{k}  \right> \right] = \left< \nabla F(u_{k}),  \frac{1}{cm} \sum_{i \in C_{k}} \nabla F_{i}(x_{k}^{(i)})  \right> \\
    \implies \mathbb{E}_{\equiv_{K}|X_{K}}\left[ \left< \nabla F(u_{k}), \mathcal{G}_{k}  \right> \right] = \left< \nabla F(u_{k}), \mathcal{H}_{k}  \right> \\
    \implies \mathbb{E}_{\equiv_{K}|X_{K}}\left[ \left< \nabla F(u_{k}), \mathcal{G}_{k}  \right> \right] = \frac{1}{2} \left[ ||\nabla F(u_{k})||^{2} + ||\mathcal{H}_{k}||^{2} - \left|\left|\frac{1}{cm} \sum_{i \in C_{k}}\nabla F_{i}(u_{k}) - \frac{1}{cm} \sum_{i \in C_{k}}\nabla F_{i}(x_{k}^{(i)})\right|\right|^{2}  \right] \\
    \textcolor{magenta}{\text{($\because 2<a,b> = ||a||^{2} + ||b||^{2} - ||a-b||^{2}$ )}}\\
    \end{gathered}
\end{equation*}
By Jensen's Inequality, we get $\left| \left| \frac{1}{cm} \sum_{i \in C_{k}} [\nabla F_{i}(u_{k}) - \nabla F_{i}(x_{k}^{(i)})]  \right| \right| \leq  \frac{1}{cm} \sum_{i \in C_{k}} \left| \left| \nabla F_{i}(u_{k}) - \nabla F_{i}(x_{k}^{(i)})  \right| \right|^{2}$\\

Thus, we get,
\begin{equation*}
    \begin{gathered}
    \mathbb{E}_{\equiv_{K}|X_{K}}\left[ \left< \nabla F(u_{k}), \mathcal{G}_{k}  \right> \right] \geq \frac{1}{2} \left[ ||\nabla F(u_{k})||^{2} + ||\mathcal{H}_{k}||^{2} - \frac{1}{cm} \sum_{i \in C_{k}}\left|\left|\nabla F_{i}(u_{k}) - \nabla F_{i}(x_{k}^{(i)})\right|\right|^{2}  \right]\\
    \geq \frac{1}{2} \left[ ||\nabla F(u_{k})||^{2} + ||\mathcal{H}_{k}||^{2} - \frac{L^{2}}{cm} \sum_{i \in C_{k}}\left|\left|\nabla u_{k} - \nabla x_{k}^{(i)}\right|\right|^{2}  \right]\\
    \textcolor{magenta}{\text{($\because$ Assumption 1)}}\\
    \geq \frac{1}{2} \left[ ||\nabla F(u_{k})||^{2} + ||\mathcal{H}_{k}||^{2} - \frac{L^{2}}{cm} \sum_{i \in C_{k}}\left|\left|\nabla X_{k}(I - J)\right|\right|^{2}_{F}  \right]\\
    \end{gathered}
\end{equation*}

\subsubsection{Lemma 5} \label{niid:lemma5}
\textcolor{red}{\textbf{Under assumption 3 and 4, the squared norm of averaged stochastic gradient is bounded as:
\begin{equation}
    \mathbb{E}_{\equiv_{K}|X_{K}}\left[ ||\mathcal{G}_{k}||^{2} \right] \leq \frac{\sigma^{2}}{cm} + ||\mathcal{H}_{k}||^{2}
\end{equation}}}
\textit{Proof}:
Since $\mathbb{E}_{\equiv_{K}|X_{K}}\left[ \mathcal{G}_{k} \right] = \mathcal{H}_{k}$, we have
\begin{equation*}
    \begin{gathered}
    \mathbb{E}_{\equiv_{K}|X_{K}}\left[ ||\mathcal{G}_{k}||^{2} \right] = \mathbb{E}_{\equiv_{K}|X_{K}}\left[ ||\mathcal{G}_{k} - \mathbb{E}[\mathcal{G}_{k}]||^{2} \right] + ||\mathbb{E}_{\equiv_{K}|X_{K}}[\mathcal{G}_{k}]||^{2}\\
    \textcolor{magenta}{\text{($\because \mathbb{E}[||a||^{2}] = \mathbb{E}[||a - \mathbb{E}[a]||]^{2} + ||\mathbb{E}[a]||^{2}$)}}\\
    \implies \mathbb{E}_{\equiv_{K}|X_{K}}\left[ ||\mathcal{G}_{k} - \mathbb{E}[\mathcal{G}_{k}]||^{2} \right] + ||\mathcal{H}_{k}||^{2}\\
    \leq \frac{\sigma^{2}}{cm} + ||\mathcal{H}_{k}||^{2}\\
    \end{gathered}
\end{equation*}
\subsection{Proof of Lemma 2} \label{niid:lemma2}
By Lipshitz continuous gradient assumption,
\begin{equation*}
    \begin{gathered}
    f(y) \leq f(x) + \nabla f(x)^{T}(y-x) + \frac{L}{2}||y-x||^{2}
    \end{gathered}
\end{equation*}
Here, let $f = F$, and $y = u_{k+1}$ and $x = u_{k}$. Then,
\begin{equation*}
    \begin{gathered}
    F(u_{k+1}) \leq F(u_{k}) + \nabla F(u_{k})^{T}(u_{k+1}-u_{k}) + \frac{L}{2}||u_{k+1}-u_{k}||^{2}
    \end{gathered}
\end{equation*}
By update rule,
\begin{equation*}
    \begin{gathered}
    u_{k+1}-u_{k} = -\eta_{eff} \left[ \frac{1}{cm}\sum_{i \in C_{k}} g_{i}(x_{k}^{(i)}) \right]
    \end{gathered}
\end{equation*}
Thus, we get,
\begin{equation*}
    \begin{gathered}
    F(u_{k+1}) \leq F(u_{k}) - \nabla F(u_{k})^{T}\left(\eta_{eff} \left[ \frac{1}{cm}\sum_{i \in C_{k}} g_{i}(x_{k}^{(i)}) \right]\right) + \frac{L}{2}\left|\left|-\eta_{eff} \left[ \frac{1}{cm}\sum_{i \in C_{k}} g_{i}(x_{k}^{(i)}) \right]\right|\right|^{2}\\
    \implies F(u_{k+1}) \leq F(u_{k}) - \eta_{eff}\nabla F(u_{k})^{T}\mathcal{G}_{k} + \frac{\eta_{eff}^{2}L}{2}\left|\left| \mathcal{G}_{k} \right|\right|^{2}\\
    \implies F(u_{k+1}) \leq F(u_{k}) - \eta_{eff}\left<\nabla F(u_{k}),\mathcal{G}_{k}\right> + \frac{\eta_{eff}^{2}L}{2}\left|\left| \mathcal{G}_{k} \right|\right|^{2}
    \end{gathered}
\end{equation*}
Taking expectations on both sides, we get
\begin{equation*}
    \begin{gathered}
     \mathbb{E}_{\equiv|X_{k}}[F(u_{k+1})] - \mathbb{E}_{\equiv|X_{k}}[F(u_{k})] \leq - \eta_{eff}\mathbb{E}_{\equiv|X_{k}}[\left<\nabla F(u_{k}),\mathcal{G}_{k}\right>] + \frac{\eta_{eff}^{2}L}{2} \mathbb{E}_{\equiv|X_{k}}[\left|\left| \mathcal{G}_{k} \right|\right|^{2}]\\
     \implies F(u_{k+1}) - F(u_{k}) \leq - \eta_{eff}\mathbb{E}_{\equiv|X_{k}}[\left<\nabla F(u_{k}),\mathcal{G}_{k}\right>] + \frac{\eta_{eff}^{2}L}{2} \mathbb{E}_{\equiv|X_{k}}[\left|\left| \mathcal{G}_{k} \right|\right|^{2}]\\
     \textcolor{magenta}{\text{($\because \mathbb{E}_{\equiv|X_{k}}[F(u_{k+1})] = F(u_{k+1})$ and $\mathbb{E}_{\equiv|X_{k}}[F(u_{k})] = F(u_{k})$ due to them being average of models)}}
    \end{gathered}
\end{equation*}
By using Lemmas 4 and 5, we get
\begin{equation*}
    \begin{gathered}
      F(u_{k+1}) - F(u_{k}) \leq - \eta_{eff}\mathbb{E}_{\equiv|X_{k}}[\left<\nabla F(u_{k}),\mathcal{G}_{k}\right>] + \frac{\eta_{eff}^{2}L}{2} \mathbb{E}_{\equiv|X_{k}}[\left|\left| \mathcal{G}_{k} \right|\right|^{2}]\\
      \leq - \frac{\eta_{eff}}{2} \left[ ||\nabla F(u_{k})||^{2} + ||\mathcal{H}_{k}||^{2} - \frac{L^{2}}{cm}||X_{k}(I-J)||^{2}_{F} \right] + \frac{\eta_{eff}^{2}L}{2} \left[ \frac{\sigma^{2}}{cm} + ||\mathcal{H}_{k}||^{2} \right]\\
      \leq - \frac{\eta_{eff}}{2}||\nabla F(u_{k})||^{2} - \frac{\eta_{eff}}{2}||\mathcal{H}_{k}||^{2} + \frac{\eta_{eff}L^{2}}{2cm}||X_{k}(I-J)||^{2}_{F} + \frac{\eta_{eff}^{2}L}{2}  ||\mathcal{H}_{k}||^{2} + \frac{\eta_{eff}^{2}L\sigma^{2}}{2cm}\\
      \leq - \frac{\eta_{eff}}{2}||\nabla F(u_{k})||^{2} - \frac{\eta_{eff}}{2}\left(1-\eta_{eff}L\right)||\mathcal{H}_{k}||^{2} + \frac{\eta_{eff}L^{2}}{2cm}||X_{k}(I-J)||^{2}_{F} + \frac{\eta_{eff}^{2}L\sigma^{2}}{2cm}\\
    \end{gathered}
\end{equation*}

If $\eta_{eff}L \leq 1$, then we get,
\begin{equation*}
    \begin{gathered}
      F(u_{k+1}) - F(u_{k}) \leq - \frac{\eta_{eff}}{2}||\nabla F(u_{k})||^{2} + \frac{\eta_{eff}L^{2}}{2cm}||X_{k}(I-J)||^{2}_{F} + \frac{\eta_{eff}^{2}L\sigma^{2}}{2cm}
    \end{gathered}
\end{equation*}
Rearranging,
\begin{equation*}
    \begin{gathered}
       \frac{\eta_{eff}}{2}||\nabla F(u_{k})||^{2} \leq F(u_{k}) - F(u_{k+1}) + \frac{\eta_{eff}L^{2}}{2cm}||X_{k}(I-J)||^{2}_{F} + \frac{\eta_{eff}^{2}L\sigma^{2}}{2cm}
    \end{gathered}
\end{equation*}
Dividing by $\eta_{eff}/2$, we get
\begin{equation*}
    \begin{gathered}
      ||\nabla F(u_{k})||^{2} \leq \frac{2[ F(u_{k}) - F(u_{k+1})]}{\eta_{eff}} + \frac{\eta_{eff}L\sigma^{2}}{cm} + \frac{L^{2}}{cm}||X_{k}(I-J)||^{2}_{F}
    \end{gathered}
\end{equation*}
Averaging over all iterates from $k=1$ to $k=K$ on both sides, we get
\begin{equation*}
    \begin{gathered}
      \frac{1}{K}\sum_{k=1}^{K}||\nabla F(u_{k})||^{2} \leq \frac{2[ \sum_{k=1}^{K}F(u_{k}) - \sum_{k=1}^{K}F(u_{k+1})]}{\eta_{eff}K} + \frac{\eta_{eff}L\sigma^{2}}{cmK}\sum_{k=1}^{K}1 + \frac{L^{2}}{Kcm}\sum_{k=1}^{K}||X_{k}(I-J)||^{2}_{F}
    \end{gathered}
\end{equation*}
Now,
\begin{equation*}
    \begin{gathered}
      2\sum_{k=1}^{K}[F(u_{k}) - F(u_{k+1})]   = F(u_{1}) - F(u_{k}) \leq F(u_{1}) - F_{inf}
    \end{gathered}
\end{equation*}
Thus, we get
\begin{equation*}
    \begin{gathered}
      \frac{1}{K}\sum_{k=1}^{K}||\nabla F(u_{k})||^{2} \leq \frac{2[ F(u_{1}) - F_{inf}]}{\eta_{eff}K} + \frac{\eta_{eff}L\sigma^{2}}{cm} + \frac{L^{2}}{Kcm}\sum_{k=1}^{K}||X_{k}(I-J)||^{2}_{F}
    \end{gathered}
\end{equation*}
Taking expectations on both sides, we get,
\begin{equation*}
    \begin{gathered}
      \mathbb{E}\left[ \frac{1}{K}\sum_{k=1}^{K}||\nabla F(u_{k})||^{2} \right] \leq \frac{2[ F(u_{1}) - F_{inf}]}{\eta_{eff}K} + \frac{\eta_{eff}L\sigma^{2}}{cm} + \frac{L^{2}}{Kcm}\sum_{k=1}^{K}\mathbb{E}||X_{k}(I-J)||^{2}_{F}
    \end{gathered}
\end{equation*}

\subsection{Other Lemmas to prove for Theorem 1}
\subsubsection{Lemma 6} \label{niid:lemma6}
\textcolor{red}{\textbf{For two matrices $A \in \mathbb{R}^{d\times m}$ and $B \in \mathbb{R}^{m\times m}$ we have:
\begin{equation}
    ||AB||_{F} \leq ||A||_{op}||B||_{F}
\end{equation}}}

\subsubsection{Lemma 7} \label{niid:lemma7}
\textcolor{red}{\textbf{For two matrices $A \in \mathbb{R}^{m\times n}$ and $B \in \mathbb{R}^{n\times m}$, we have:
\begin{equation}
    |Tr(AB)| \leq ||A||_{F} ||B||_{F}
\end{equation}}}

\subsubsection{Lemma 8} \label{niid:lemma8}
\textcolor{red}{\textbf{Let there be some $m\times m$ matrix $\Phi_{s,k}$ that satisfies assumption 5. Then,
\begin{equation}
    ||\Phi_{s,k}^{T}(I-J)||^{2}_{F} \leq \delta
\end{equation}
where $\delta \in [0, c(m - 1)]$}}\\

\subsection{Theorem 2} \label{theorem2}
As per Lemma 2 derivation, we achieved Equation \ref{eqn9}. We want to provide an upper bound to the network error term.\\

Similar to the IID case, we have
\begin{equation*}
    \begin{gathered}
    X_{k}(I-J) = X_{1}\Phi_{1,k-1}^{T}(I-J) - \eta \sum_{s=1}^{k-1} G_{s}\Phi_{s,k-1}^{T}(I-J)
    \end{gathered}
\end{equation*}
where $\Phi_{s,k-1}^{T} = \prod_{l=s}^{k-1}S_{l}^{T}$.\\
Taking the squared frobenius norm and expectations on both sides, we get,
\begin{equation*}
    \begin{gathered}
    \mathbb{E} ||X_{k}(I-J)||^{2}_{F} = \mathbb{E} \left|\left|X_{1}\Phi_{1,k-1}^{T}(I-J) - \eta\sum_{s=1}^{k-1} G_{s}\Phi_{s,k-1}^{T}(I-J)\right|\right|^{2}_{F}\\
    \implies \mathbb{E} ||X_{k}(I-J)||^{2}_{F} = \mathbb{E} \left|\left|X_{1}\Phi_{1,k-1}^{T}(I-J)\right|\right|^{2}_{F} +  \eta^{2} \mathbb{E} \left|\left|\sum_{s=1}^{k-1} G_{s}\Phi_{s,k-1}^{T}(I-J)\right|\right|^{2}_{F} \\- 2\left< X_{1}\Phi_{1,k-1}^{T}(I-J), \eta\sum_{s=1}^{k-1} G_{s}\Phi_{s,k-1}^{T}(I-J)  \right>\\
    \implies \mathbb{E} ||X_{k}(I-J)||^{2}_{F} \leq \mathbb{E} \left|\left|X_{1}\Phi_{1,k-1}^{T}(I-J)\right|\right|^{2}_{F} +  \eta^{2} \mathbb{E} \left|\left|\sum_{s=1}^{k-1} G_{s}\Phi_{s,k-1}^{T}(I-J)\right|\right|^{2}_{F}
    \end{gathered}
\end{equation*}

Now let $K = j\tau + i$, where $j$ denotes the index of communication rounds and $i$ denotes index of local updates. Then, similar to IID case, we get,
\begin{equation*}
    \begin{gathered}
    \mathbb{E} ||X_{k}(I-J)||^{2}_{F} \leq \mathbb{E} \left|\left|X_{1}\Phi_{1,k-1}^{T}(I-J)\right|\right|^{2}_{F} + \eta^{2} \mathbb{E} \left|\left|\sum_{r=0}^{j} Y_{r}\Phi_{r,j}^{T}(I-J)\right|\right|^{2}_{F}\\
    = \mathbb{E} \left|\left|X_{1}\Phi_{1,k-1}^{T}(I-J)\right|\right|^{2}_{F} + \eta^{2} \mathbb{E} \left|\left|\sum_{r=0}^{j} (Y_{r}-Q_{r})\Phi_{r,j}^{T}(I-J) + \sum_{r=0}^{j} Q_{r}\Phi_{r,j}^{T}(I-J)\right|\right|^{2}_{F}\\
    \leq \underbrace{\mathbb{E} \left|\left|X_{1}\Phi_{1,k-1}^{T}(I-J)\right|\right|^{2}_{F}}_{T0} + \underbrace{2\eta^{2} \mathbb{E} \left|\left|\sum_{r=0}^{j} (Y_{r}-Q_{r})\Phi_{r,j}^{T}(I-J)\right|\right|^{2}_{F}}_{\text{T1}} + \underbrace{2\eta^{2} \mathbb{E} \left|\left|\sum_{r=0}^{j} Q_{r}\Phi_{r,j}^{T}(I-J)\right|\right|^{2}_{F}}_{\text{T2}}\\
    \textcolor{magenta}{(\because ||a+b||^{2} \leq 2||a||^{2} + 2||b||^{2})}
    \end{gathered}
\end{equation*}
We now look to bound T0, T1 and T2. We will derive bounds for average of all iterates.
\subsubsection{Bounding T0} \label{t0}
\begin{equation*}
    \begin{gathered}
    T0 = \mathbb{E} \left|\left|X_{1}\Phi_{1,k-1}^{T}(I-J)\right|\right|^{2}_{F}\\
    = \mathbb{E} \left|\left|X_{1}(\Phi_{1,k-1}^{T}-J)\right|\right|^{2}_{F}\\
    \leq \mathbb{E} \left|\left|X_{1}\right|\right|^{2}_{op} \left|\left| \Phi_{1,k-1}^{T}-J\right|\right|^{2}_{F}\\
    \leq \mathbb{E} \left|\left|X_{1}\right|\right|^{2}_{F} \left|\left| \Phi_{1,k-1}^{T}-J\right|\right|^{2}_{F}\\
    \textcolor{magenta}{(\because ||A||_{op} \leq ||A||_{F})}\\
    \leq \mathbb{E} \left|\left|X_{1}\right|\right|^{2}_{F} \delta\\
    \textcolor{magenta}{(\because \text{Lemma 8})}\\
    \leq \delta\left|\left|X_{1}\right|\right|^{2}_{F}
    \end{gathered}
\end{equation*}

\subsubsection{Bounding T1} \label{t1}
\begin{equation*}
    \begin{gathered}
    T1 = 2\eta^{2} \mathbb{E} \left|\left|\sum_{r=0}^{j} (Y_{r}-Q_{r})\Phi_{r,j}^{T}(I-J)\right|\right|^{2}_{F} = 2\eta^{2} \sum_{r=0}^{j} \mathbb{E} \left|\left| (Y_{r}-Q_{r})\Phi_{r,j}^{T}(I-J)\right|\right|^{2}_{F} 
    \end{gathered} 
\end{equation*}
This is possible since cross terms are zero (similar logic as IID case). By Lemma 6, we have,
\begin{equation*}
    \begin{gathered}
    \left|\left| (Y_{r}-Q_{r})\Phi_{r,j}^{T}(I-J)\right|\right|^{2}_{F} \leq  \left|\left| Y_{r}-Q_{r}\right|\right|^{2}_{op} \left|\left| \Phi_{r,j}^{T}(I-J)\right|\right|^{2}_{F}\\
    \implies \left|\left| (Y_{r}-Q_{r})\Phi_{r,j}^{T}(I-J)\right|\right|^{2}_{F} \leq  \left|\left| Y_{r}-Q_{r}\right|\right|^{2}_{F} \left|\left| \Phi_{r,j}^{T}(I-J)\right|\right|^{2}_{F}\\
    \end{gathered}
\end{equation*}
---here
Thus, we have,
\begin{equation*}
    \begin{gathered}
    T1 \leq 2\eta^{2} \sum_{r=0}^{j} \mathbb{E} \left|\left| Y_{r}-Q_{r}\right|\right|^{2}_{F} \left|\left| \Phi_{r,j}^{T}(I-J)\right|\right|^{2}_{F}\\
    \implies T1 \leq 2\eta^{2} \sum_{r=0}^{j} \mathbb{E} \left|\left| Y_{r}-Q_{r}\right|\right|^{2}_{F} \delta\\
    \textcolor{magenta}{(\because \text{Lemma 8})}\\
    \implies T1 \leq 2\eta^{2} \sum_{r=0}^{j-1} \mathbb{E} \left|\left| Y_{r}-Q_{r}\right|\right|^{2}_{F} \delta + 2\eta^{2} \mathbb{E} \left|\left| Y_{j}-Q_{j}\right|\right|^{2}_{F}\delta\\
    \end{gathered}
\end{equation*}
Now, for any $0\leq r < j$,
\begin{equation*}
    \begin{gathered}
     \mathbb{E} \left|\left| Y_{r}-Q_{r}\right|\right|^{2}_{F} = \mathbb{E} \left[ \left|\left| \sum_{s=r\tau+1}^{(r+1)\tau}[G_{s}-\nabla F(X_{s})]\right|\right|^{2}_{F} \right] = \sum_{i = 1}^{m} \mathbb{E} \left[ \left|\left| \sum_{s=r\tau+1}^{(r+1)\tau}[g(x_{s}^{(i)})-\nabla F(x_{s}^{(i)})]\right|\right|^{2} \right]\\
     = \sum_{i = 1}^{m}  \left( \mathbb{E} \left[ \sum_{s=r\tau+1}^{(r+1)\tau} \left|\left| g(x_{s}^{(i)})-\nabla F(x_{s}^{(i)})\right|\right|^{2} \right] + \mathbb{E} \left[  \sum_{s,l \in S_{t}, s\neq l}\left< g(x_{s}^{(i)})-\nabla F(x_{s}^{(i)}), g(x_{l}^{(i)})-\nabla F(x_{l}^{(i)}) \right> \right] \right)
    \end{gathered}
\end{equation*}
Here, the cross terms are zero. We prove as follows:
\begin{equation*}
    \begin{gathered}
     \mathbb{E} \left< g(x_{s}^{(i)})-\nabla F(x_{s}^{(i)}), g(x_{l}^{(i)})-\nabla F(x_{l}^{(i)}) \right> =   \mathbb{E}_{x_{s}^{i}, \xi_{s}^{i}, x_{l}^{i}} \mathbb{E}_{\xi_{l}^{i} | x_{s}^{i}, \xi_{s}^{i}, x_{l}^{i}} \left[ \left< g(x_{s}^{(i)})-\nabla F(x_{s}^{(i)}), g(x_{l}^{(i)})-\nabla F(x_{l}^{(i)}) \right> \right]\\
     \textcolor{magenta}{(\because \mathbb{E}_{x}[\mathbb{E}_{y|x} (Z(x,y))] = \mathbb{E}(Z))}\\
     = \mathbb{E}_{x_{s}^{i}, \xi_{s}^{i}, x_{l}^{i}} \left[ \left< \mathbb{E}_{\xi_{l}^{i} | x_{s}^{i}, \xi_{s}^{i}, x_{l}^{i}} \left[ g(x_{s}^{(i)})-\nabla F(x_{s}^{(i)}) \right], \mathbb{E}_{\xi_{l}^{i} | x_{s}^{i}, \xi_{s}^{i}, x_{l}^{i}} \left[ g(x_{l}^{(i)})-\nabla F(x_{l}^{(i)}) \right] \right> \right]\\
     = \mathbb{E}_{x_{s}^{i}, \xi_{s}^{i}, x_{l}^{i}} \left[ \left<  g(x_{s}^{(i)})-\nabla F(x_{s}^{(i)}), \mathbb{E}_{\xi_{l}^{i} | x_{s}^{i}, \xi_{s}^{i}, x_{l}^{i}} \left[ g(x_{l}^{(i)})-\nabla F(x_{l}^{(i)}) \right] \right> \right]\\
     = \mathbb{E} \left[ \left<  g(x_{s}^{(i)})-\nabla F(x_{s}^{(i)}), \mathbb{E}_{\xi_{l}^{i} | x_{s}^{i}, \xi_{s}^{i}, x_{l}^{i}} \left[ g(x_{l}^{(i)})-\nabla F(x_{l}^{(i)}) \right] \right> \right]\\
     = \mathbb{E} \left[ \left<  g(x_{s}^{(i)})-\nabla F(x_{s}^{(i)}), 0 \right> \right]\\
     \textcolor{magenta}{(\because \mathbb{E}_{\xi_{l}^{i} | x_{s}^{i}, \xi_{s}^{i}, x_{l}^{i}} \left[ g(x_{l}^{(i)})-\nabla F(x_{l}^{(i)}) \right] = \mathbb{E}_{\xi_{l}^{i} | x_{s}^{i}, \xi_{s}^{i}, x_{l}^{i}} [g(x_{l}^{(i)})]-\nabla F(x_{l}^{(i)}) = \nabla F(x_{l}^{(i)}) - \nabla F(x_{l}^{(i)}) = 0)}
    \end{gathered}
\end{equation*}
Thus, we have,
\begin{equation*}
    \begin{gathered}
     \mathbb{E} \left|\left| Y_{r}-Q_{r}\right|\right|^{2}_{F} 
     = \sum_{i = 1}^{m}   \mathbb{E} \left[ \sum_{s=r\tau+1}^{(r+1)\tau} \left|\left| g(x_{s}^{(i)})-\nabla F(x_{s}^{(i)})\right|\right|^{2} \right] = \mathbb{E}    \left[ \sum_{i = 1}^{m} \sum_{s=r\tau+1}^{(r+1)\tau} \left|\left| g(x_{s}^{(i)})-\nabla F(x_{s}^{(i)})\right|\right|^{2} \right]\\
     = \mathbb{E}    \left[ \sum_{s=r\tau+1}^{(r+1)\tau} \sum_{i = 1}^{m} \left|\left| g(x_{s}^{(i)})-\nabla F(x_{s}^{(i)})\right|\right|^{2} \right]\\
     \textcolor{magenta}{\left( \because \sum_{x}\sum_{y}a = \sum_{y}\sum_{x}a \right)}\\
     =    \left[ \sum_{s=r\tau+1}^{(r+1)\tau} \sum_{i \in C_{k}} \mathbb{E} 
 \left|\left| g(x_{s}^{(i)})-\nabla F(x_{s}^{(i)})\right|\right|^{2} \right]\\
     \leq    \sum_{s=r\tau+1}^{(r+1)\tau} \sum_{i \in C_{k}} \left[\sigma^{2} \right]\\
     \textcolor{magenta}{\left( \because \text{Assumption 4} \right)}\\
     \leq    \sum_{s=r\tau+1}^{(r+1)\tau} \sum_{i \in C_{k}} \sigma^{2} \\
     \leq    \tau cm \sigma^{2} \\
    \end{gathered}
\end{equation*}
Similarly,
\begin{equation*}
    \begin{gathered}
     \mathbb{E} \left|\left| Y_{j}-Q_{j}\right|\right|^{2}_{F} 
     \leq  (i-1) cm \sigma^{2} \\
    \end{gathered}
\end{equation*}
Thus, we get,
\begin{equation*}
    \begin{gathered}
    T1 \leq 2\eta^{2} \sum_{r=0}^{j-1} \left[ \delta \left(  \tau cm \sigma^{2} \right) \right] + 2\eta^{2}(i-1) cm \sigma^{2} \delta\\
    \leq 2\eta^{2}cm \sigma^{2} \delta\left[ \tau j + i-1 \right]
    \end{gathered}
\end{equation*}
Now, we sum over all iterates in $j^{th}$ local update period (from $i=1$ to $i=\tau$).
\begin{equation*}
    \begin{gathered}
    \sum_{i=1}^{\tau} T1 \leq  2\eta^{2}cm \sigma^{2}\delta\left[ \tau j \sum_{i=1}^{\tau}1 + \sum_{i=1}^{\tau}i-\sum_{i=1}^{\tau}1 \right]\\
    \implies \sum_{i=1}^{\tau} T1 \leq  2\eta^{2}cm \sigma^{2}\delta\left[ \tau^{2}j + \frac{\tau(\tau+1)}{2}-\tau \right]\\
    \implies \sum_{i=1}^{\tau} T1 \leq  \eta^{2}cm \sigma^{2}\delta\tau\left[ 2\tau j + \tau-1 \right]\\    
    \end{gathered}
\end{equation*}

Then, summing over all periods $j=0$ to $j=K/\tau-1$ we get,
\begin{equation*}
    \begin{gathered}
     \sum_{j=0}^{K/\tau-1}\sum_{i=1}^{\tau} T1 \leq  \eta^{2}cm \sigma^{2}\delta \tau\left[ 2\tau\sum_{j=0}^{K/\tau-1}j  + (\tau-1)\sum_{j=0}^{K/\tau-1}1 \right] \\ 
     \implies \sum_{j=0}^{K/\tau-1}\sum_{i=1}^{\tau} T1 \leq  \eta^{2}cm \sigma^{2} \delta \tau\left[ 2\tau\frac{K/\tau (K/\tau - 1)}{2} + (\tau-1)K/\tau \right] \\
     \implies \sum_{j=0}^{K/\tau-1}\sum_{i=1}^{\tau} T1 \leq  \eta^{2}cm \sigma^{2}\delta \tau\left[ K\left(K/\tau-1\right) + (\tau-1)K/\tau \right] \\
     \implies \sum_{j=0}^{K/\tau-1}\sum_{i=1}^{\tau} T1 \leq  \eta^{2}cm \sigma^{2}\delta K(K-1)
    \end{gathered}
\end{equation*}

\subsubsection{Bounding T2} \label{t2}
\begin{equation*}
    \begin{gathered}
    T2 = 2\eta^{2} \mathbb{E} \left|\left|\sum_{r=0}^{j} Q_{r}\Phi_{r,j}^{T}(I-J)\right|\right|^{2}_{F}
    \end{gathered}
\end{equation*}
Since $||A||^{2}_{F} = Tr(A^{T}A)$, we have,
\begin{equation*}
    \begin{gathered}
    T2 = 2\eta^{2} \mathbb{E} Tr \left[ \left( \sum_{n=0}^{j} Q_{n}\Phi_{n,j}^{T}(I-J) \right)^{T} \left( \sum_{r=0}^{j} Q_{r}\Phi_{r,j}^{T}(I-J) \right) \right]
    \end{gathered}
\end{equation*}
Expanding the trace equation,
\begin{equation*}
    \begin{gathered}
    T2 = 2\eta^{2} \sum_{r=0}^{j} \sum_{n=0}^{j} \mathbb{E} Tr \left[ \left( \left[\Phi_{n,j}^{T}(I-J)\right]^{T}Q_{n}^{T} \right) \left( Q_{r}\Phi_{r,j}^{T}(I-J) \right) \right]
    \end{gathered}
\end{equation*}
We now split into two parts:
\begin{itemize}
    \item Diagonal Terms where $r=n$
    \item Off Diagonal Terms where $r\neq n$
\end{itemize}
Thus,
\begin{equation*}
    \begin{gathered}
    T2 = 2\eta^{2} \sum_{r=0}^{j} \sum_{n=0}^{j} \mathbb{E} Tr \left[ \left( \left[\Phi_{n,j}^{T}(I-J)\right]^{T}Q_{n}^{T} \right) \left( Q_{r}\Phi_{r,j}^{T}(I-J) \right) \right]
    \\
    = 2\eta^{2} \sum_{r=0}^{j} \mathbb{E} Tr \left[ \left[\Phi_{r,j}^{T}( I-J)\right]^{T}Q_{r}^{T}   Q_{r}\Phi_{r,j}^{T}(I-J) \right] + 2\eta^{2}  \sum_{n=0}^{j} \sum_{r=0, r\neq n}^{j} \mathbb{E} Tr \left[  \left[\Phi_{n,j}^{T}(I-J)\right]^{T}Q_{n}^{T}  Q_{r}\Phi_{r,j}^{T}(I-J)  \right]\\
    = 2\eta^{2} \sum_{r=0}^{j} \mathbb{E} || Q_{r}\Phi_{r,j}^{T}(I-J) ||^{2}_{F} + 2\eta^{2}  \sum_{n=0}^{j} \sum_{l=0, l\neq n}^{j} \mathbb{E} Tr \left[  \left[\Phi_{n,j}^{T}(I-J)\right]^{T}Q_{n}^{T}  Q_{l}\Phi_{l,j}^{T}(I-J)  \right]\\
    \textcolor{magenta}{(\because \text{By frobenius norm definition})}
    \end{gathered}
\end{equation*}
Now, by Lemma 7,
\begin{equation*}
    \begin{gathered}
     \left|Tr \left[  \left[\Phi_{n,j}^{T}(I-J)\right]^{T}Q_{n}^{T}  Q_{l}\Phi_{l,j}^{T}(I-J)  \right]\right|
     \leq  ||\left[\Phi_{n,j}^{T}(I-J)\right]^{T}Q_{n}^{T}||_{F}  ||Q_{l}\Phi_{l,j}^{T}(I-J)||_{F}\\
     \leq  ||\Phi_{n,j}^{T}(I-J)||_{op}||Q_{n}||_{F}  ||Q_{l}||_{op}||\Phi_{l,j}^{T}(I-J)||_{F}\\
     \textcolor{magenta}{(\because \text{Lemma 6 and }||A||_{F} = ||A^{T}||_{F})}\\
     \leq  ||\Phi_{n,j}^{T}(I-J)||_{F}||Q_{n}||_{F}  ||Q_{l}||_{op}||\Phi_{l,j}^{T}(I-J)||_{F}\\
     \textcolor{magenta}{(\because ||A||_{F} \geq ||A||_{op})}\\
     \leq  \sqrt{\delta}||Q_{n}||_{F}  ||Q_{l}||_{F}\sqrt{\delta}\\
     \textcolor{magenta}{(\because \text{Lemma 8 and 9})}\\
     \leq \frac{\delta}{2} [||Q_{n}||_{F}^{2} + ||Q_{l}||_{F}^{2}]\\
     \textcolor{magenta}{(\because \text{AM $\geq$ GM})}\\
    \end{gathered}
\end{equation*}
Thus, we have,
\begin{equation*}
    \begin{gathered}
    T2 \leq 2\eta^{2} \sum_{r=0}^{j} \mathbb{E} || Q_{r}\Phi_{r,j}^{T}(I-J) ||^{2}_{F} + \eta^{2}\delta  \sum_{n=0}^{j} \sum_{l=0, l\neq n}^{j}  \mathbb{E} \left[  ||Q_{n}||_{F}^{2} + ||Q_{l}||_{F}^{2}  \right]\\
    \leq 2\eta^{2} \sum_{r=0}^{j} \mathbb{E} || Q_{r}||^{2}_{op}||\Phi_{r,j}^{T}(I-J) ||^{2}_{F} + \eta^{2}\delta  \sum_{n=0}^{j} \sum_{l=0, l\neq n}^{j}  \mathbb{E} \left[  ||Q_{n}||_{F}^{2} + ||Q_{l}||_{F}^{2}  \right]\\
    \leq 2\eta^{2} \sum_{r=0}^{j} \mathbb{E} || Q_{r}||^{2}_{F}||\Phi_{r,j}^{T}(I-J) ||^{2}_{F} + \eta^{2}\delta  \sum_{n=0}^{j} \sum_{l=0, l\neq n}^{j}  \mathbb{E} \left[  ||Q_{n}||_{F}^{2} + ||Q_{l}||_{F}^{2}  \right]\\
    \leq 2\eta^{2}\delta \sum_{r=0}^{j} \mathbb{E} || Q_{r}||^{2}_{F} + \eta^{2}\delta  \sum_{n=0}^{j} \sum_{l=0, l\neq n}^{j}  \mathbb{E} \left[  ||Q_{n}||_{F}^{2} + ||Q_{l}||_{F}^{2}  \right]\\
    \end{gathered}
\end{equation*}
Now $||Q_{n}||$ and $||Q_{l}||$ are the same thing, so $\mathbb{E}||Q_{n}|| = \mathbb{E}||Q_{l}||$. Thus,
\begin{equation*}
    \begin{gathered}
    T2 \leq 2\eta^{2}\delta \sum_{r=0}^{j} \mathbb{E} || Q_{r}||^{2}_{F} + 2\eta^{2}\delta  \sum_{n=0}^{j} \sum_{l=0, l\neq n}^{j}  \mathbb{E} \left[  ||Q_{n}||_{F}^{2}  \right]\\
    \leq 2\eta^{2}\delta \sum_{r=0}^{j} \mathbb{E} || Q_{r}||^{2}_{F} + 2\eta^{2}\delta  \sum_{n=0}^{j}  \mathbb{E} \left[  ||Q_{n}||_{F}^{2}  \right]\sum_{l=0, l\neq n}^{j}1\\
    \leq 2\eta^{2}\delta \left[ \sum_{r=0}^{j} \mathbb{E} || Q_{r}||^{2}_{F} +  \sum_{n=0}^{j}  \mathbb{E}  ||Q_{n}||_{F}^{2}  \sum_{l=0, l\neq n}^{j}1 \right]\\
    \leq 2\eta^{2}\delta \left[ \sum_{r=0}^{j-1} \mathbb{E} || Q_{r}||^{2}_{F} +   \sum_{n=0}^{j-1}  \mathbb{E}  ||Q_{n}||_{F}^{2}  \sum_{l=0, l\neq n}^{j}1 +  \mathbb{E} || Q_{j}||^{2}_{F} +  \mathbb{E}  ||Q_{j}||_{F}^{2}  \sum_{l=0, l\neq n}^{j}1 \right]\\
    \leq 2\eta^{2}\delta \left[ \sum_{r=0}^{j-1} \mathbb{E} || Q_{r}||^{2}_{F} +  \sum_{n=0}^{j-1}  \mathbb{E}  ||Q_{n}||_{F}^{2}  \sum_{l=0}^{j}1 +  \mathbb{E} || Q_{j}||^{2}_{F} +  \mathbb{E}  ||Q_{j}||_{F}^{2}  \sum_{l=0}^{j}1 \right]\\
    \textcolor{magenta}{(\because \text{More terms on RHS})}\\
    \leq 2\eta^{2}\delta \left[ \sum_{r=0}^{j-1} \mathbb{E} || Q_{r}||^{2}_{F} +  \sum_{n=0}^{j-1}  \mathbb{E}  ||Q_{n}||_{F}^{2}  (j+1) +  \mathbb{E} || Q_{j}||^{2}_{F} +  \mathbb{E}  ||Q_{j}||_{F}^{2}  (j+1) \right]\\
    \leq 2\eta^{2}\delta  \sum_{r=0}^{j-1} (1+(j+1)) \mathbb{E} || Q_{r}||^{2}_{F} +  2\eta^{2}\delta(1+ (j+1)) \mathbb{E} || Q_{j}||^{2}_{F} \\
    \leq 2\eta^{2}\delta  \sum_{r=0}^{j-1} (j+2) \mathbb{E} || \sum_{s=1}^{\tau} \nabla F(X_{r\tau+s})||^{2}_{F} +  2\eta^{2}\delta(j+2) \mathbb{E} || \sum_{s=1}^{i-1} \nabla F(X_{j\tau+s}) ||^{2}_{F} \\
    \end{gathered}
\end{equation*}
Now, By Jensen's Inequality, 
\begin{equation*}
    \begin{gathered}
     \varphi(\mathbb{E}[X]) \leq \mathbb{E}[\varphi(X)] \implies \text{ Take }\varphi = ||\cdot ||^{2}_{F} \text{ and } X = \nabla F(X_{r\tau+s}) \text{, then }
     \\||E[X]||^{2}_{F} = \left| \left| \frac{1}{\tau} \sum_{s=1}^{\tau} \nabla F(X_{r\tau+s})  \right| \right|^{2}_{F} \leq  \frac{1}{\tau} \sum_{s=1}^{\tau} \left| \left| \nabla F(X_{r\tau+s})  \right| \right|^{2}_{F} \\
    \implies \left| \left| \sum_{s=1}^{\tau} \nabla F(X_{r\tau+s})  \right| \right|^{2}_{F} \leq  \tau \sum_{s=1}^{\tau} \left| \left| \nabla F(X_{r\tau+s})  \right| \right|^{2}_{F}\\
    \end{gathered}
\end{equation*}
For second term in a similar way, 
\begin{equation*}
    \begin{gathered}
     \left| \left| \sum_{s=1}^{\tau} \nabla F(X_{r\tau+s})  \right| \right|^{2}_{F} \leq  (i-1) \sum_{s=1}^{\tau} \left| \left| \nabla F(X_{r\tau+s})  \right| \right|^{2}_{F}\\
    \end{gathered}
\end{equation*}
Thus, we get,
\begin{equation*}
    \begin{gathered}
    T2 \leq 2\eta^{2}\delta\tau  \sum_{r=0}^{j-1} \left[ (j+2)  \sum_{s=1}^{\tau}  \mathbb{E} \left| \left| \nabla F(X_{r\tau+s})  \right| \right|^{2}_{F} \right] +  2\eta^{2}\delta(j+2)(i-1) \sum_{s=1}^{i-1} \mathbb{E}  \left| \left| \nabla F(X_{r\tau+s})  \right| \right|^{2}_{F}
    \end{gathered}
\end{equation*}

Next, we sum over all iterates in $j^{th}$ period.
\begin{equation*}
    \begin{gathered}
    \sum_{i=1}^{\tau}T2 \leq 2\eta^{2}\delta\tau  \sum_{r=0}^{j-1} \left[ (j+2)  \sum_{s=1}^{\tau}  \mathbb{E} \left| \left| \nabla F(X_{r\tau+s})  \right| \right|^{2}_{F} \right] \sum_{i=1}^{\tau}1 \\+  2\eta^{2}\delta(j+2)\sum_{i=1}^{\tau}(i-1) \sum_{s=1}^{i-1} \mathbb{E}  \left| \left| \nabla F(X_{r\tau+s})  \right| \right|^{2}_{F}\\
    \end{gathered}
\end{equation*}
Now,
\begin{equation*}
    \begin{gathered}
    \sum_{i=1}^{\tau}(1-i) \sum_{s=1}^{i-1} \mathbb{E} \left|\left|  \nabla F(X_{j\tau + s}) \right|\right|_{F}^{2} \leq \frac{\tau(\tau-1)}{2} \sum_{s=1}^{\tau-1} \mathbb{E} \left|\left|  \nabla F(X_{j\tau + s}) \right|\right|_{F}^{2}\\
    \end{gathered}
\end{equation*}
Thus, we get,
\begin{equation*}
    \begin{gathered}
    \sum_{i=1}^{\tau}T2 \leq 2\eta^{2}\delta\tau^{2}  \sum_{r=0}^{j-1} \left[ (j+2)  \sum_{s=1}^{\tau}  \mathbb{E} \left| \left| \nabla F(X_{r\tau+s})  \right| \right|^{2}_{F} \right] \\+  \eta^{2}\delta\tau(\tau-1)(j+2) \sum_{s=1}^{\tau-1} \mathbb{E} \left|\left|  \nabla F(X_{j\tau + s}) \right|\right|_{F}^{2}\\
    \end{gathered}
\end{equation*}
Now summing over all periods ($j=0$ to $j=K/\tau - 1$), we get
\begin{equation*}
    \begin{gathered}
    \sum_{j=0}^{K/\tau-1}\sum_{i=1}^{\tau}T2 \leq 2\eta^{2}\delta\tau^{2}  \sum_{j=0}^{K/\tau-1}\sum_{r=0}^{j-1} \left[ (j+2)  \sum_{s=1}^{\tau}  \mathbb{E} \left| \left| \nabla F(X_{r\tau+s})  \right| \right|^{2}_{F} \right] \\
    +  \eta^{2}\delta\tau(\tau-1)\sum_{j=0}^{K/\tau-1}(j+2) \sum_{s=1}^{\tau-1} \mathbb{E} \left|\left|  \nabla F(X_{j\tau + s}) \right|\right|_{F}^{2}\\
    \end{gathered}
\end{equation*}
Now, we have,
\begin{equation*}
    \begin{gathered}
    \sum_{j=0}^{K/\tau-1}\sum_{r=0}^{j-1} \left[( j+2)  \sum_{s=1}^{\tau}  \mathbb{E} \left| \left| \nabla F(X_{r\tau+s})  \right| \right|^{2}_{F} \right] \leq S_{series}\sum_{j=0}^{K/\tau-1}\sum_{s=1}^{\tau} \mathbb{E} \left| \left| \nabla F(X_{j\tau+s})  \right| \right|^{2}_{F}\\
    \end{gathered}
\end{equation*}
Here, $S_{series}$ is given by the following method:
\begin{equation*}
    \begin{gathered}
    \sum_{j=0}^{K/\tau-1}\sum_{r=0}^{j-1} \left[ (j+2)  \sum_{s=1}^{\tau}  \mathbb{E} \left| \left| \nabla F(X_{r\tau+s})  \right| \right|^{2}_{F} \right]  =  \sum_{j=0}^{K/\tau-1}\sum_{r=0}^{j-1} \left[ j  \underbrace{\sum_{s=1}^{\tau}  \mathbb{E} \left| \left| \nabla F(X_{r\tau+s})  \right| \right|^{2}_{F}}_{V_{r}} \right]  \\
    + 2\sum_{j=0}^{K/\tau-1}\sum_{r=0}^{j-1} \left[  \underbrace{\sum_{s=1}^{\tau}  \mathbb{E} \left| \left| \nabla F(X_{r\tau+s})  \right| \right|^{2}_{F}}_{V_{r}} \right]\\
    = \sum_{j=0}^{K/\tau-1}\sum_{r=0}^{j-1}  jV_{r}  + 2\sum_{j=0}^{K/\tau-1}\sum_{r=0}^{j-1}  V_{r}
    \end{gathered}
\end{equation*}
The first term is broken down as:
\begin{equation*}
    \begin{gathered}
    \sum_{j=0}^{K/\tau-1}\sum_{r=0}^{j-1}  jV_{r}  = \cancel{\sum_{r=0}^{-1}0.V_{r}} + \sum_{r=0}^{0}1.V_{r} + \sum_{r=0}^{1}2.V_{r}... + \sum_{r=0}^{K/\tau-2}(K/\tau - 1).V_{r}\\
    = \left[ 1.V_{0} \right] + \left[ 2.V_{0} + 2.V_{1} \right] + \left[ 3.V_{0} + 3.V_{1} + 3.V_{2} \right] ... \\
    = \left[ 1+2+3...+(K/\tau - 1) \right]V_{0} + \left[ 2+3...+(K/\tau - 1) \right]V_{1} + \left[ 3...+(K/\tau - 1) \right]V_{2} ... + \left[ K/\tau - 1 \right]V_{K/\tau - 2}\\
    \leq \left[ 1+2...+(K/\tau - 1) \right]V_{0} + \left[ 1+2...+(K/\tau - 1) \right]V_{1} + ... + \left[ 1+2...+(K/\tau - 1) \right]V_{K/\tau - 2}\\
    \leq \left[ 1+2...+(K/\tau - 1) \right] (V_{0} + V_{1} + V_{2} ... +V_{K/\tau - 2})\\
    \leq \frac{K/\tau(K/\tau - 1)}{2} (V_{0} + V_{1} + V_{2} ... +V_{K/\tau - 2})\\
    \leq \frac{K/\tau(K/\tau - 1)}{2} \sum_{r=0}^{K/\tau-2}V_{r}\\
    \leq \frac{K/\tau(K/\tau - 1)}{2} \sum_{r=0}^{K/\tau-1}V_{r}\\
    \leq \frac{K/\tau(K/\tau - 1)}{2} \sum_{r=0}^{K/\tau-1}\sum_{s=1}^{\tau}  \mathbb{E} \left| \left| \nabla F(X_{r\tau+s})  \right| \right|^{2}_{F}\\
    \end{gathered}
\end{equation*}
The second term can be written as:
\begin{equation*}
    \begin{gathered}
    \sum_{j=0}^{K/\tau-1}\sum_{r=0}^{j-1}  V_{r} = \cancel{\sum_{r=0}^{-1}  V_{r}} + \sum_{r=0}^{0}  V_{r} + \sum_{r=0}^{1}  V_{r} + ... \sum_{r=0}^{K/\tau-2}  V_{r}\\
    = \left[ V_{0} \right] + \left[ V_{0} + V_{1} \right] + \left[ V_{0} + V_{1} + V{2} \right] + ...\\
    = \left[ K/\tau-1 \right]V_{0} + \left[ K/\tau-2 \right]V_{1}  + ...\\
    \leq \left( K/\tau - 1 \right) (V_{0} + V_{1} + V_{2} ... +V_{K/\tau - 2})\\
    \leq \left( K/\tau - 1 \right) \sum_{r=0}^{K/\tau-2}V_{r}\\
    \leq \left( K/\tau - 1 \right) \sum_{r=0}^{K/\tau-1}V_{r}\\
    \leq \left( K/\tau - 1 \right) \sum_{r=0}^{K/\tau-1}\sum_{s=1}^{\tau}  \mathbb{E} \left| \left| \nabla F(X_{r\tau+s})  \right| \right|^{2}_{F}
    \end{gathered}
\end{equation*}
Then, we can write $S_{series} = (K/\tau - 1) \left[ 2 +\frac{ K}{2\tau}\right]$.
Thus, we get,
\begin{equation*}
    \begin{gathered}
    \sum_{j=0}^{K/\tau-1}\sum_{i=1}^{\tau}T2 \leq 2\eta^{2}\delta\tau^{2}  S_{series}\sum_{j=0}^{K/\tau-1}\sum_{s=1}^{\tau} \mathbb{E} \left| \left| \nabla F(X_{j\tau+s})  \right| \right|^{2}_{F} \\+  \eta^{2}\delta \tau(\tau-1)\sum_{j=0}^{K/\tau-1}(j+2) \sum_{s=1}^{\tau-1} \mathbb{E} \left|\left|  \nabla F(X_{j\tau + s}) \right|\right|_{F}^{2}\\
    \end{gathered}
\end{equation*}
Now, 
\begin{equation*}
    \begin{gathered}
    \sum_{j=0}^{K/\tau-1}(j+2) \sum_{s=1}^{\tau-1} \mathbb{E} \left|\left|  \nabla F(X_{j\tau + s}) \right|\right|_{F}^{2} \leq (2 + (K/\tau -1))\sum_{j=0}^{K/\tau-1} \sum_{s=1}^{\tau} \mathbb{E}\left|\left|  \nabla F(X_{j\tau + s}) \right|\right|_{F}^{2}\\
    \end{gathered}
\end{equation*}

Thus, we get,
\begin{equation*}
    \begin{gathered}
    \sum_{j=0}^{K/\tau-1}\sum_{i=1}^{\tau}T2 \leq 2\eta^{2}\delta\tau^{2}  S_{series}\sum_{j=0}^{K/\tau-1}\sum_{s=1}^{\tau} \mathbb{E} \left| \left| \nabla F(X_{j\tau+s})  \right| \right|^{2}_{F} \\+  \eta^{2}\delta \tau(\tau-1)(1 +  K/\tau)\sum_{j=0}^{K/\tau-1} \sum_{s=1}^{\tau}\mathbb{E}\left|\left|  \nabla F(X_{j\tau + s}) \right|\right|_{F}^{2}\\
    \end{gathered}
\end{equation*}

Replacing $j\tau+s$ with $K$, we get,
\begin{equation*}
    \begin{gathered}
    \sum_{j=0}^{K/\tau-1}\sum_{i=1}^{\tau}T2 \leq 2\eta^{2}\delta\tau^{2}  S_{series}\sum_{k=1}^{K} \mathbb{E} \left| \left| \nabla F(X_{k})  \right| \right|^{2}_{F} +  \eta^{2}\delta \tau(\tau-1)(1 + K/\tau)\sum_{k=1}^{K}\mathbb{E}\left|\left|  \nabla F(X_{k}) \right|\right|_{F}^{2}\\
    \implies \sum_{j=0}^{K/\tau-1}\sum_{i=1}^{\tau}T2 \leq \eta^{2}\delta\tau\left[ \underbrace{2\tau S_{series} + (\tau-1)(1 + K/\tau)}_{P}\right] \sum_{k=1}^{K}\mathbb{E}\left|\left|  \nabla F(X_{k}) \right|\right|_{F}^{2}\\
    \end{gathered}
\end{equation*}

\subsubsection{Final Result} \label{t2}
\begin{equation*}
    \begin{gathered}
    \mathbb{E} ||X_{k}(I-J)||^{2}_{F} = \leq \underbrace{\mathbb{E} \left|\left|X_{1}\Phi_{1,k-1}^{T}(I-J)\right|\right|^{2}_{F}}_{T0} + \underbrace{2\eta^{2} \mathbb{E} \left|\left|\sum_{r=0}^{j} (Y_{r}-Q_{r})\Phi_{r,j}^{T}(I-J)\right|\right|^{2}_{F}}_{\text{T1}} + \underbrace{2\eta^{2} \mathbb{E} \left|\left|\sum_{r=0}^{j} Q_{r}\Phi_{r,j}^{T}(I-J)\right|\right|^{2}_{F}}_{\text{T2}}
    \end{gathered}
\end{equation*}
Now taking summation over all periods, we get,
\begin{equation*}
    \begin{gathered}
    \frac{1}{Kcm}\sum_{k=1}^{K}\mathbb{E} ||X_{k}(I-J)||^{2}_{F}  \leq \frac{1}{Kcm}\sum_{k=1}^{K} (T0 + T1 + T2)\\\\
    \implies \frac{1}{Kcm}\sum_{k=1}^{K}\mathbb{E} ||X_{k}(I-J)||^{2}_{F} \leq \frac{1}{Kcm}\sum_{j=0}^{K/\tau-1}\sum_{i=1}^{\tau} (T0+ T1 + T2)\\
    \implies \frac{L^{2}}{Kcm}\sum_{k=1}^{K}\mathbb{E} ||X_{k}(I-J)||^{2}_{F} \leq  \frac{L^{2}K}{Kcm} \left( \delta\left|\left|X_{1}\right|\right|^{2}_{F} \right) + \frac{L^{2}}{Kcm} \left( \eta^{2}cm \sigma^{2}\delta K(K-1)  \right) \\
    + \frac{L^{2}}{Kcm} \left( P \sum_{k=1}^{K}\mathbb{E}\left|\left|  \nabla F(X_{k}) \right|\right|_{F}^{2}  \right)\\
    \leq \underbrace{\frac{\delta L^{2}}{cm}\left|\left|X_{1}\right|\right|^{2}_{F}}_{\text{Contribution of T0}} 
    +  \underbrace{\eta^{2}\sigma^{2}L^{2}\delta (K-1) }_{\text{Contribution of T1}} 
    +\underbrace{\frac{PL^{2}}{Kcm}\sum_{k=1}^{K}\mathbb{E}\left|\left|  \nabla F(X_{k}) \right|\right|_{F}^{2}}_{\text{Contribution of T2}}  \\
    \end{gathered}
\end{equation*}
Now,
\begin{equation*}
    \begin{gathered}
    ||\nabla F(X_{k}) ||^{2}_{F} = \sum_{i=1}^{m}||\nabla F_{i}(x_{k}^{(i)}) ||^{2} 
    = \sum_{i \in C_{K}}||\nabla F_{i}(x_{k}^{(i)}) ||^{2}\\
    \leq 3 \sum_{i \in C_{K}} \left[||\nabla F_{i}(x_{k}^{(i)}) - \nabla F_{i}(u_{k})||^{2} + ||\nabla F_{i}(u_{k}) - \nabla F(u_{k})||^{2} + ||\nabla F(u_{k})||^{2}\right]\\
    \textcolor{magenta}{(\because ||a+b+c||^{2} \leq 3(||a||^{2} + ||b||^{2} + ||c||^{2}))}\\
    \leq 3||X_{k}(I-J)||^{2} + 3cm\kappa^{2} + 3cm||\nabla F(u_{k})||^{2}\\
    \end{gathered}
\end{equation*}

Thus, we get,
\begin{equation*}
    \begin{gathered}
    \frac{L^{2}}{Kcm}\sum_{k=1}^{K}\mathbb{E} ||X_{k}(I-J)||^{2}_{F}  \leq  \frac{\delta L^{2}}{cm}\left|\left|X_{1}\right|\right|^{2}_{F}
    +  \eta^{2}\sigma^{2}L^{2}\delta (K-1) \\
    +\frac{3PL^{2}}{Kcm}\sum_{k=1}^{K}\mathbb{E}||X_{k}(I-J)||^{2} + \frac{3PL^{2}cm\kappa^{2}}{Kcm}\sum_{k=1}^{K}1 + \frac{3PL^{2}cm}{Kcm}\sum_{k=1}^{K}\mathbb{E}||\nabla F(u_{k})||^{2}  \\
    \implies \frac{L^{2}}{Kcm} \left( 1-3P\right)\sum_{k=1}^{K}\mathbb{E} ||X_{k}(I-J)||^{2}_{F}  \leq  \frac{\delta L^{2}}{cm}\left|\left|X_{1}\right|\right|^{2}_{F}
    +  \eta^{2}\sigma^{2}L^{2}\delta (K-1) \\
     + {3PL^{2}\kappa^{2}} + \frac{3PL^{2}}{K}\sum_{k=1}^{K}\mathbb{E}||\nabla F(u_{k})||^{2}  \\
     \implies \frac{L^{2}}{Kcm} \sum_{k=1}^{K}\mathbb{E} ||X_{k}(I-J)||^{2}_{F}  \leq  \frac{\delta L^{2}}{cm\left( 1-3P\right)}\left|\left|X_{1}\right|\right|^{2}_{F}
    +  \frac{\eta^{2}\sigma^{2}L^{2}\delta (K-1)}{\left( 1-3P\right)} \\
     + \frac{3PL^{2}\kappa^{2}}{\left( 1-3P\right)} + \frac{3PL^{2}}{\left( 1-3P\right)}\frac{1}{K}\sum_{k=1}^{K}\mathbb{E}||\nabla F(u_{k})||^{2}  \\
    \end{gathered}
\end{equation*}

Going all the way back, we now substitute in Equation \ref{eqn9}.
\begin{equation*}
    \begin{gathered}
      \mathbb{E}\left[ \frac{1}{K}\sum_{k=1}^{K}||\nabla F(u_{k})||^{2} \right] \leq \frac{2[ F(u_{1}) - F_{inf}]}{\eta_{eff}K} + \frac{\eta_{eff}L\sigma^{2}}{cm} + \frac{L^{2}}{Kcm}\sum_{k=1}^{K}\left|\left| X_{k}(I - J)\right|\right|^{2}_{F} \\
      \implies \mathbb{E}\left[ \frac{1}{K}\sum_{k=1}^{K}||\nabla F(u_{k})||^{2} \right] \leq \frac{2[ F(u_{1}) - F_{inf}]}{\eta_{eff}K} + \frac{\eta_{eff}L\sigma^{2}}{cm} + \frac{\delta L^{2}}{cm\left( 1-3P\right)}\left|\left|X_{1}\right|\right|^{2}_{F}
    +  \frac{\eta^{2}\sigma^{2}L^{2}\delta (K-1)}{\left( 1-3P\right)} \\
     + \frac{3PL^{2}\kappa^{2}}{\left( 1-3P\right)} + \frac{3PL^{2}}{\left( 1-3P\right)}\frac{1}{K}\sum_{k=1}^{K}\mathbb{E}||\nabla F(u_{k})||^{2}\\
     \implies \left[ 1 - \frac{3PL^{2}}{\left( 1-3P\right)} \right] \mathbb{E}\left[ \frac{1}{K}\sum_{k=1}^{K}||\nabla F(u_{k})||^{2} \right] \leq \frac{2[ F(u_{1}) - F_{inf}]}{\eta_{eff}K} + \frac{\eta_{eff}L\sigma^{2}}{cm} + \frac{\delta L^{2}}{cm\left( 1-3P\right)}\left|\left|X_{1}\right|\right|^{2}_{F}
    \\+  \frac{\eta^{2}\sigma^{2}L^{2}\delta (K-1)}{\left( 1-3P\right)} + \frac{3PL^{2}\kappa^{2}}{\left( 1-3P\right)}
    \end{gathered}
\end{equation*}\

\subsubsection{Case 1: $\delta=0$}
In such a scenario, we get,
\begin{equation*}
    \begin{gathered}
    \mathbb{E}\left[ \frac{1}{K}\sum_{k=1}^{K}||\nabla F(u_{k})||^{2} \right] \leq \frac{2[ F(u_{1}) - F_{inf}]}{\eta_{eff}K} + \frac{\eta_{eff}L\sigma^{2}}{cm}
    \end{gathered}
\end{equation*}

\subsubsection{Case 2: $1\geq\delta>0$}
Similar to the IID Scenario, let us look at the constant term and simplify it.
\begin{equation*}
    \begin{gathered}
      \left[ 1 - \frac{3PL^{2}}{\left( 1-3P\right)} \right] = \left[ \frac{(1-3P) -3PL^{2}}{\left( 1-3P\right)} \right]  = \left[ \frac{1-3P\left( 1 + L^{2} \right)}{ 1-3P} \right]
    \end{gathered}
\end{equation*}

Let the following condition be satisfied, so that our calculations are simplified.
\begin{equation*}
    \begin{gathered}
     \frac{ 1-3P}{1-3P\left( 1 + L^{2} \right)} \leq 2 \\
     \implies 1-3P\left( 1 + L^{2} \right) \geq \frac{1}{2} (1-3P)\\
     \implies (1-3P) - 3PL^{2} \geq \frac{1}{2} (1-3P)\\
     \implies \frac{1}{2}(1-3P) \geq 3PL^{2}\\
     \implies (1-3P) \geq 6PL^{2}\\
     \implies P \leq \frac{1}{6L^{2}+3}\\
    \end{gathered}
\end{equation*}

Eventually, we get,
\begin{equation*}
    \begin{gathered}
    \mathbb{E}\left[ \frac{1}{K}\sum_{k=1}^{K}||\nabla F(u_{k})||^{2} \right] \leq 2\left[ \frac{2[ F(u_{1}) - F_{inf}]}{\eta_{eff}K} + \frac{\eta_{eff}L\sigma^{2}}{cm} + \frac{\delta L^{2}}{cm\left( 1-3P\right)}\left|\left|X_{1}\right|\right|^{2}_{F}
    +  \frac{\eta^{2}\sigma^{2}L^{2}\delta (K-1)}{\left( 1-3P\right)} + \frac{3PL^{2}\kappa^{2}}{\left( 1-3P\right)} \right]\\
    \end{gathered}
\end{equation*}
Let $0\leq P \leq \frac{1}{6} \implies 1 \leq \frac{1}{1-3P} \leq 2$. Then we get,
\begin{equation*}
    \begin{gathered}
    \mathbb{E}\left[ \frac{1}{K}\sum_{k=1}^{K}||\nabla F(u_{k})||^{2} \right] \leq 2\left[ \frac{2[ F(u_{1}) - F_{inf}]}{\eta_{eff}K} + \frac{\eta_{eff}L\sigma^{2}}{cm} + \frac{2\delta L^{2}}{cm}\left|\left|X_{1}\right|\right|^{2}_{F}
    +  2\eta^{2}\sigma^{2}L^{2}\delta (K-1) + 6PL^{2}\kappa^{2} \right]\\
    \implies \mathbb{E}\left[ \frac{1}{K}\sum_{k=1}^{K}||\nabla F(u_{k})||^{2} \right] \leq 4\left[ \frac{2[ F(u_{1}) - F_{inf}]}{\eta_{eff}K} + \frac{\eta_{eff}L\sigma^{2}}{cm} + \frac{\delta L^{2}}{cm}\left|\left|X_{1}\right|\right|^{2}_{F}
    +  \eta^{2}\sigma^{2}L^{2}\delta (K-1) + 3PL^{2}\kappa^{2} \right]\\
    \implies \mathbb{E}\left[ \frac{1}{K}\sum_{k=1}^{K}||\nabla F(u_{k})||^{2} \right] \leq  \epsilon_{IID} +  12PL^{2}\kappa^{2}
    \end{gathered}
\end{equation*}
Also, since $P < \min \left( \frac{1}{6}, \frac{1}{6L^{2}+3} \right) \implies P < \frac{1}{B} \implies 12PL^{2} < \frac{12L^{2}}{B}$. So we can assume that this part of our equation is a constant.\\

Taking this criteria for P ensures that our previous assumption of $P<\frac{1}{3}$ also holds. This holds for Case 3 (Section \ref{theorem2:res-sec3}) as well.

If we further have the condition that $\delta \leq \frac{\tau}{K-1}$, then we can say
\begin{equation*}
    \begin{gathered}
    \mathbb{E}\left[ \frac{1}{K}\sum_{k=1}^{K}||\nabla F(u_{k})||^{2} \right] \leq 4\left[\frac{2[ F(u_{1}) - F_{inf}]}{\eta_{eff}K} + \frac{\eta_{eff}L\sigma^{2}}{cm} +  \frac{\delta L^{2}}{cm}\left|\left|X_{1}\right|\right|^{2}_{F} + \eta^{2}\sigma^{2}L^{2}\tau + 3PL^{2}K^{2}\right]\\
    \end{gathered}
\end{equation*}
This somewhat matches the result from \citep{Wang2021} (ignoring the model initialization error and constant factor). In fact, if we compare with their criterion:
\begin{equation*}
    \begin{gathered}
    \textcolor{blue}{\mathbb{E}\left[ \frac{1}{K}\sum_{k=1}^{K}||\nabla F(u_{k})||^{2} \right] \leq 2\left[\frac{2[ F(u_{1}) - F_{inf}]}{\eta_{eff}K} + \frac{\eta_{eff}L\sigma^{2}}{cm} + \eta^{2}\sigma^{2}L^{2}\left(\frac{1+\varsigma^{2}}{1-\varsigma^{2}}\tau - 1\right) + C_{2}\kappa^{2}\right]}
    \end{gathered}
\end{equation*}
Following a similar analysis as the IID case, we see that our error bounds improves on \citet{Wang2021} as long as $\tau > \frac{1-\varsigma^{2}}{2\varsigma^{2}}$, irrespective of the value of $\delta(K-1)$. Another thing to note is that $C_{2} = \frac{3\eta^{2}L^{2}\tau^{2}C_{1}}{1-\varsigma} \leq \frac{1}{4}$ can be treated as a simple constant. In our analysis too the constant for $\kappa^{2}$ is only dependent on the Lipschitz constant (when $P<\frac{1}{6} \implies 3PL^{2}<\frac{L^{2}}{2}$) or some other real number (when $P<\frac{1}{6L^{2}+3} \implies P<\frac{1}{6L^{2}} \implies 3PL^{2}<\frac{1}{2}$). So, these do not contribute to the error bound beyond a constant factor, and hence we do not consider them when calculating when our error bound is tighter than \citet{Wang2021}.

\subsubsection{Case 3: $\delta>1$}\label{theorem2:res-sec3}
We simply get,
\begin{equation*}
    \begin{gathered}
    \mathbb{E}\left[ \frac{1}{K}\sum_{k=1}^{K}||\nabla F(u_{k})||^{2} \right] \leq 4\left[ \frac{2[ F(u_{1}) - F_{inf}]}{\eta_{eff}K} + \frac{\eta_{eff}L\sigma^{2}}{cm} + \frac{\delta L^{2}}{cm}\left|\left|X_{1}\right|\right|^{2}_{F}
    +  \eta^{2}\sigma^{2}L^{2}\delta (K-1) + 3PL^{2}\kappa^{2} \right]\\
    \implies \mathbb{E}\left[ \frac{1}{K}\sum_{k=1}^{K}||\nabla F(u_{k})||^{2} \right] \leq  \epsilon_{IID} +  12PL^{2}\kappa^{2}
    \end{gathered}
\end{equation*}
No optimization is possible.

\subsection{Lower Bound on the Fraction of Client's Selected $c$}

We know that,
\begin{equation*}
    \begin{gathered}
    ||\nabla F(X_{k}) ||^{2}_{F} 
    \leq 3||X_{k}(I-J)||^{2} + 3cm\kappa^{2} + 3cm||\nabla F(u_{k})||^{2}\\
    \implies ||\nabla F(X_{k}) ||^{2}_{F} \leq 3||X_{k}(I-J)||^{2} + 3m\kappa^{2} + 3m||\nabla F(u_{k})||^{2}\\
    \end{gathered}
\end{equation*}

If we use the looser upper bound we get,
\begin{equation*}
    \begin{gathered}
    \frac{L^{2}}{Kcm} \sum_{k=1}^{K}\mathbb{E} ||X_{k}(I-J)||^{2}_{F}  \leq  \frac{\delta L^{2}}{cm\left( 1-3P\right)}\left|\left|X_{1}\right|\right|^{2}_{F}
    +  \frac{\eta^{2}\sigma^{2}L^{2}\delta (K-1)}{c\left( 1-3P\right)} \\
     + \frac{3PL^{2}\kappa^{2}}{c\left( 1-3P\right)} + \frac{3PL^{2}}{\left( 1-3P\right)}\frac{1}{K}\sum_{k=1}^{K}\mathbb{E}||\nabla F(u_{k})||^{2}  \\
    \end{gathered}
\end{equation*}

Going all the way back, we now substitute in Equation \ref{eqn9}.
\begin{equation*}
    \begin{gathered}
      \left[ 1 - \frac{3PL^{2}}{c\left( 1-3P\right)} \right] \mathbb{E}\left[ \frac{1}{K}\sum_{k=1}^{K}||\nabla F(u_{k})||^{2} \right] \leq \frac{2[ F(u_{1}) - F_{inf}]}{\eta_{eff}K} + \frac{\eta_{eff}L\sigma^{2}}{cm} + \frac{\delta L^{2}}{cm\left( 1-3P\right)}\left|\left|X_{1}\right|\right|^{2}_{F}
    \\+  \frac{\eta^{2}\sigma^{2}L^{2}\delta (K-1)}{\left( 1-3P\right)} + \frac{3PL^{2}\kappa^{2}}{c\left( 1-3P\right)} \\
    \end{gathered}
\end{equation*}

We can simplify the constant term as
\begin{equation*}
    \begin{gathered}
      \left[ 1 - \frac{3PL^{2}}{c\left( 1-3P\right)} \right] = \left[ \frac{c(1-3P) -3PL^{2}}{c\left( 1-3P\right)} \right]  = \left[ \frac{1-3P\left( 1 + \frac{L^{2}}{c} \right)}{ 1-3P} \right] \\
    \end{gathered}
\end{equation*}

If the following condition is satisfied, we get a lower bound on the fraction of client's selected.
\begin{equation*}
    \begin{gathered}
     \frac{ 1-3P}{1-3P\left( 1 + \frac{L^{2}}{c} \right)} \leq 2 \\
     \implies 1-3P\left( 1 + \frac{L^{2}}{c} \right) \geq \frac{1}{2} (1-3P)\\
     \implies (1-3P) - \frac{3PL^{2}}{c} \geq \frac{1}{2} (1-3P)\\
     \implies \frac{1}{2}(1-3P) \geq \frac{3PL^{2}}{c}\\
     \implies (1-3P) \geq \frac{6PL^{2}}{c}\\
     \implies c \geq \frac{6PL^{2}}{1-3P}\\
     \implies c \geq 6PL^{2}\\
     \textcolor{magenta}{\left(\because \frac{1}{1-3P} \geq 1\right)}
    \end{gathered}
\end{equation*}
This too leads to the same convergence error bound as was calculated previously. Another thing to note is that when $\delta=0$, then, similar to the IID case, we get no bounds on the value of c.



\section{Writing requisite condition for $K$ with Auxillary Variables} \label{supp:k-cond}

Based on Theorem 1 in the IID Case and and Theorem 2 in the non-IID case, we get,
\begin{equation*}
    \begin{gathered}
      P = \eta^{2}\delta\tau\left[ 2(K - \tau) \left[ 2 +\frac{K}{2\tau}\right] + (\tau-1)(1 + K/\tau)\right] \leq \min \left(\frac{1}{6} , \frac{1}{6L^{2}+3},  \frac{c}{6L^{2}} \right)\\
      \implies P = \eta^{2}\delta\tau\left[ 2(K - \tau) \left[ 2 +\frac{K}{2\tau}\right] + (\tau-1)(1 + K/\tau)\right] \leq \min \left(\frac{1}{6} , \frac{1}{6L^{2}+3},  \frac{1}{6L^{2}/c} \right)
    \end{gathered}
\end{equation*}

Let us generalize this as:
\begin{equation*}
    \begin{gathered}
      P = \eta^{2}\delta\tau\left[ 2(K - \tau) \left[ 2 +\frac{K}{2\tau}\right] + (\tau-1)(1 + K/\tau)\right] \leq \frac{1}{B}\\
      \implies \eta^{2}\delta\tau\left[ 2 K  \left[ 2 + \frac{K}{2\tau}\right] - 2\tau \left[ 2 + \frac{K}{2\tau}\right] + \tau + K - 1 - \frac{K}{\tau} \right] \leq \frac{1}{B}\\
      \implies \eta^{2}\delta\tau\left[ 4 K   + \frac{K^{2}}{\tau} - 2\tau \left[ \cancel{2} + \frac{K}{2\tau} - \cancel{\frac{1}{2}}\right]  + K - 1 - \frac{K}{\tau} \right] \leq \frac{1}{B}\\
      \implies \eta^{2}\delta\tau\left[ 4 K   + \frac{K^{2}}{\tau} - \frac{3}{2}\tau - K  + K - 1 - \frac{K}{\tau}\right] \leq \frac{1}{B}\\
      \implies \eta^{2}\delta\tau\left[ 4 K   + \frac{K^{2}}{\tau} - \frac{3}{2}\tau - 1 - \frac{K}{\tau}\right] \leq \frac{1}{B}\\
      \implies \eta^{2}\delta\tau\left[ K\left( 4   + \frac{K}{\tau} - \frac{1}{\tau}\right) - \frac{3}{2}\tau - 1 \right] \leq \frac{1}{B}\\
      \implies B\eta^{2}\delta\tau\left[ K\left( 4   + \frac{K}{\tau} - \frac{1}{\tau}\right) - \frac{3}{2}\tau - 1 \right] \leq 1\\
      \implies B\eta^{2}\delta\tau K\left( 4   + \frac{K}{\tau} - \frac{1}{\tau}\right) - B\eta^{2}\delta\tau\left[ \frac{3}{2}\tau + 1\right]  \leq 1\\
      \implies B\eta^{2}\delta\tau K\left( 4   + \frac{K}{\tau} - \frac{1}{\tau}\right)  \leq B\eta^{2}\delta\tau\left[\frac{3}{2}\tau + 1\right] + 1\\
    \end{gathered}
\end{equation*}
Putting $\eta = \frac{m+v}{Lcm}\sqrt{\frac{cm}{K}}$, we get,
\begin{equation*}
    \begin{gathered}
      B\frac{(m+v)^{2}}{L^{2}cmK}\delta\tau K\left( 4   + \frac{K}{\tau} - \frac{1}{\tau}\right)  \leq B\frac{(m+v)^{2}}{L^{2}cmK}\delta\tau\left[\frac{3}{2}\tau + \right] + 1\\
      \implies B\frac{(m+v)^{2}}{\cancel{L^{2}cmK}}\delta\tau K\left( 4   + \frac{K}{\tau} - \frac{1}{\tau}\right)  \leq B\frac{(m+v)^{2}\delta\tau\left[\frac{3}{2}\tau + 1\right] + L^{2}cmK}{\cancel{L^{2}cmK}}\\
      \implies B(m+v)^{2}\delta\tau K\left( 4   + \frac{K}{\tau} - \frac{1}{\tau}\right)  \leq B(m+v)^{2}\delta\tau\left[\frac{3}{2}\tau + 1\right] + L^{2}cmK\\
    \end{gathered}
\end{equation*}

\subsection{$0 \leq \delta \leq 1$}
Assuming $1 \geq \delta \geq 0$, we can say $B(m+v)^{2}\delta\tau\left[\frac{3}{2}\tau + 1\right] \leq B(m+v)^{2}\tau\left[\frac{3}{2}\tau + 1\right]$. Thus,
\begin{equation*}
    \begin{gathered}
    B(m+v)^{2}\delta\tau K\left( 4   + \frac{K}{\tau} - \frac{1}{\tau}\right)  \leq B(m+v)^{2}\tau\left[\frac{3}{2}\tau + 1\right] + L^{2}cmK\\
      \cancel{B(m+v)^{2}\delta}\tau K\left( 4   + \frac{K}{\tau} - \frac{1}{\tau} - \frac{L^{2}cm}{B(m+v)^{2}\delta\tau} \right)  \leq \cancel{B(m+v)^{2}}\tau\left[\frac{3}{2}\tau + 1\right]\\
      \implies K\delta\tau\left( 4   + \frac{K}{\tau} - \frac{1}{\tau} - \frac{L^{2}cm}{B(m+v)^{2}\delta\tau} \right)  \leq \tau\left[\frac{3}{2}\tau + 1\right]\\
    \end{gathered}
\end{equation*}
Now, we know that $\tau<K$, and thus we can say,
\begin{equation*}
    \begin{gathered}
      \cancel{K}\delta\tau\left( 4   + \frac{K}{\tau} - \frac{1}{\tau} - \frac{L^{2}cm}{B(m+v)^{2}\delta\tau}\right)  \leq \cancel{K}\left[\frac{3}{2}K + 1\right]\\
       \implies 4\delta\tau   + K\delta - \delta - \frac{L^{2}cm}{B(m+v)^{2}} \leq \frac{3}{2} K + 1\\
       \implies 4\delta\tau - (1+\delta) \leq  \frac{K(3-2\delta)}{2} + \frac{L^{2}cm}{B(m+v)^{2}}\\
        \implies 4\delta\tau - \delta \leq  \frac{K(3-2\delta)}{2} + \frac{L^{2}cm}{B(m+v)^{2}}\\
       \implies 4\delta\tau-\delta   \leq K\left(\frac{3-2\delta}{2} + \frac{L^{2}cm}{B(m+v)^{2}} \right)\\
       \implies K\geq \frac{\delta(4\tau-1)}{\frac{3-2\delta}{2} + \frac{L^{2}cm}{B(m+v)^{2}}} \\
    \end{gathered}
\end{equation*}
If we assume $\frac{L^{2}cm}{B(m+v)^{2}} \leq \frac{1}{2} \implies \frac{3}{2} - \delta +\frac{L^{2}cm}{B(m+v)^{2}} \leq \frac{1}{2} + \frac{3}{2} - \delta  \implies \frac{3}{2} - \delta +\frac{L^{2}cm}{B(m+v)^{2}} \leq 2 - \delta \leq 2 \implies \frac{\delta(4\tau-1)}{\frac{3}{2} - \delta + \frac{L^{2}cm}{B(m+v)^{2}\delta}} \geq \delta(2\tau-\frac{1}{2})$, then we get,
\begin{equation*}
    \begin{gathered}
      K \geq \delta(2\tau-\frac{1}{2}) \\
       \implies K \geq \mathcal{O}(\delta\tau)\\
    \end{gathered}
\end{equation*}
This holds if we assume $\frac{L^{2}cm}{B(m+v)^{2}} \leq \frac{1}{2}$, which is not an unreasonable assumption to make, as it implies $\frac{B(m+v)^{2}}{L^{2}cm} \geq 2$. This can be easily achieved since $B>1$ and, thus, the fraction is $\mathcal{O}(m)$ which can easily exceed the value 2. Also, the criterion for $K$ is also easily achieved, especially considering $K>\tau$ and $\delta < 1$.

\subsection{$\delta \geq 1$}
Assuming $\delta \geq 1$, we can say,
\begin{equation*}
    \begin{gathered}
    B(m+v)^{2}\delta\tau K\left( 4   + \frac{K}{\tau} - \frac{1}{\tau}\right)  \leq B(m+v)^{2}\tau\left[\frac{3}{2}\tau + 1\right] + L^{2}cmK\\
      \cancel{B(m+v)^{2}\delta}\tau K\left( 4   + \frac{K}{\tau} - \frac{1}{\tau} - \frac{L^{2}cm}{B(m+v)^{2}\delta\tau} \right)  \leq \cancel{B(m+v)^{2}\delta}\tau\left[\frac{3}{2}\tau + 1\right]\\
      \implies K\tau\left( 4   + \frac{K}{\tau} - \frac{1}{\tau} - \frac{L^{2}cm}{B(m+v)^{2}\delta\tau} \right)  \leq \tau\left[\frac{3}{2}\tau + 1\right]\\
    \end{gathered}
\end{equation*}
Now, we know that $\tau<K$, and thus we can say,
\begin{equation*}
    \begin{gathered}
      \cancel{K}\tau\left( 4   + \frac{K}{\tau} - \frac{1}{\tau} - \frac{L^{2}cm}{B(m+v)^{2}\delta\tau}\right)  \leq \cancel{K}\left[\frac{3}{2}K + 1\right]\\
       \implies 4\tau   + K - 1 - \frac{L^{2}cm}{B(m+v)^{2}\delta} \leq \frac{3}{2} K + 1\\
       \implies 4\tau - 2 \leq  \frac{K}{2} + \frac{L^{2}cm}{B(m+v)^{2}\delta}\\
       \implies 4\tau-2   \leq K\left(\frac{1}{2} + \frac{L^{2}cm}{B(m+v)^{2}\delta} \right)\\
       \implies K\geq \frac{2(2\tau-1)}{\frac{1}{2} + \frac{L^{2}cm}{B(m+v)^{2}\delta}} \\
    \end{gathered}
\end{equation*}
If we assume $\frac{L^{2}cm}{B(m+v)^{2}\delta} \leq \frac{1}{2} \implies \frac{2(2\tau-1)}{\frac{1}{2} + \frac{L^{2}cm}{B(m+v)^{2}\delta}} \geq 2(2\tau-1)$, then we get,
\begin{equation*}
    \begin{gathered}
      K \geq 2(2\tau-1) \\
       \implies K \geq \mathcal{O}(\tau)\\
    \end{gathered}
\end{equation*}
This holds if we assume $\frac{L^{2}cm}{B(m+v)^{2}\delta} \leq \frac{1}{2}$, which is not an unreasonable assumption to make, as it implies $\frac{B(m+v)^{2}\delta}{L^{2}cm} \geq 2$. This can be easily achieved since $\delta > 1$, and $\mathcal{O}(1)$ at best, implying the fraction is at least $\mathcal{O}(m)$ and at most $\mathcal{O}(m^{2})$, which can easily exceed the value 2. Thus, irrespective of value of $\delta$, we can say $K \geq \mathcal{O}(\tau)$.


\section{Corollary 1: Proof}\label{Corollary1}

We have the following condition:
\begin{equation*}
    \begin{gathered}
        \eta_{eff}L \leq 1
    \end{gathered}
\end{equation*}

Putting $\eta = \frac{m+v}{Lcm}\sqrt{\frac{cm}{K^{2}}} \implies \eta_{eff} = \frac{1}{L}\sqrt{\frac{cm}{K^{2}}}$, we get,
\begin{equation*}
    \begin{gathered}
         \sqrt{\frac{cm}{K^{2}}} \leq  1\\
    \end{gathered}
\end{equation*}
Then,
\begin{equation*}
    \begin{gathered}
          \frac{cm}{K^{2}} \leq 1\\
          \implies K \geq \sqrt{cm}
    \end{gathered}
\end{equation*}

Also, from $P$'s conditions (see previous), we get
\begin{equation*}
    \begin{gathered}
      K \geq \mathcal{O}(\tau)
    \end{gathered}
\end{equation*}

Putting the value of $\eta$ in the convergence error, we get,
\begin{equation*}
    \begin{gathered}
        \mathbb{E}\left[ \frac{1}{K}\sum_{k=1}^{K}||\nabla F(u_{k})||^{2} \right] \leq 4\left[\frac{2[ F(u_{1}) - F_{inf}]}{\eta_{eff}K} + \frac{\eta_{eff}L\sigma^{2}}{cm} + \frac{\delta L^{2}}{cm}\left|\left|X_{1}\right|\right|^{2}_{F} 
        +  \eta^{2}\sigma^{2}L^{2}\delta(K-1)\right]\\
        \implies \mathbb{E}\left[ \frac{1}{K}\sum_{k=1}^{K}||\nabla F(u_{k})||^{2} \right] \leq 4\left[
        \begin{aligned}
            \frac{2[ F(u_{1}) - F_{inf}]}{\eta_{eff}K} + \frac{\eta_{eff}L\sigma^{2}}{cm} + \frac{\delta L^{2}}{cm}\left|\left|X_{1}\right|\right|^{2}_{F} 
            +  \eta^{2}_{eff}L^{2}\sigma^{2}\delta\left[\frac{1}{c^{2}}\left( 1+ \frac{v}{m} \right)^{2}(K-1)\right] 
        \end{aligned}
        \right]\\
        \implies \mathbb{E}\left[ \frac{1}{K}\sum_{k=1}^{K}||\nabla F(u_{k})||^{2} \right] \leq 4\left[
        \begin{aligned}
            \frac{2L[ F(u_{1}) - F_{inf}]}{\sqrt{cm}} + \frac{\sigma^{2}}{\sqrt{cmK^{2}}}+ \frac{\delta L^{2}}{cm}\left|\left|X_{1}\right|\right|^{2}_{F} 
            +  \frac{cm}{K^{2}}\sigma^{2}\delta\left[\frac{1}{c^{2}}\left( 1+ \frac{v}{m} \right)^{2} (K-1)\right] 
        \end{aligned}
        \right]\\
        \implies \mathbb{E}\left[ \frac{1}{K}\sum_{k=1}^{K}||\nabla F(u_{k})||^{2} \right] \leq 4\left[
        \begin{aligned}
            \frac{2L[ F(u_{1}) - F_{inf}]}{\sqrt{cm}} + \frac{\sigma^{2}}{\sqrt{cmK^{2}}}+ \frac{\delta L^{2}}{cm}\left|\left|X_{1}\right|\right|^{2}_{F} 
            +  \frac{m}{K^{2}c}\sigma^{2}\delta\left[\left( 1+ \frac{v}{m} \right)^{2} 
            (K-1)\right] 
        \end{aligned}
        \right]
    \end{gathered}
\end{equation*}
Now, for the last term,
\begin{equation*}
    \begin{gathered}
         \left( 1+ \frac{v}{m} \right)^{2} (K - 1)\\\
         = \left( 1+ \left(\frac{v}{m}\right)^{2} + 2\frac{v}{m} \right) (K - 1)\\
         = \left( 1+ \left(\frac{v}{m}\right)^{2} + 2\frac{v}{m} \right) \mathcal{O}\left( K \right)\\
         =  \mathcal{O}\left( K \right)+ \left(\frac{v}{m}\right)^{2}\mathcal{O}\left( K \right) + 2\frac{v}{m}  \mathcal{O}\left(K \right)\\
         =  \mathcal{O}\left( K \right)+ \mathcal{O}\left( \frac{v^{2}K}{m^{2}} \right) + \mathcal{O}\left(\frac{vK}{m} \right)\\
    \end{gathered}
\end{equation*}
Including the factors before,
\begin{equation*}
    \begin{gathered}
         \frac{m}{K^{2}c}\sigma^{2}\delta\mathcal{O}\left( K \right)+ \frac{m}{K^{2}c}\sigma^{2}\delta\mathcal{O}\left( \frac{v^{2}K}{m^{2}} \right) + \frac{m}{K^{2}c}\sigma^{2}\delta\mathcal{O}\left(\frac{vK}{m} \right)\\
         = \mathcal{O}\left( \frac{\delta m}{Kc} \right)+ \mathcal{O}\left( \frac{v^{2}m\delta}{m^{2}cK} \right) + \mathcal{O}\left(\frac{v\delta}{cK} \right)\\
         = \mathcal{O}\left( \frac{\delta m}{cK} \right)
    \end{gathered}
\end{equation*}
Thus, we get the following final complexity,
\begin{equation*}
    \begin{gathered}
        \mathbb{E}\left[ \frac{1}{K}\sum_{k=1}^{K}||\nabla F(u_{k})||^{2} \right] \leq \left[\mathcal{O}\left( \frac{1}{\sqrt{cm}} \right) + \frac{\delta L^{2}}{cm}\left|\left|X_{1}\right|\right|^{2}_{F} 
        +  \mathcal{O}\left( \frac{m\delta}{cK} \right) \right]
    \end{gathered}
\end{equation*}
Assuming that we initialize all initial models as zero, then we get $||X_{1}||^{2}_{F} = 0$
\begin{equation*}
    \begin{gathered}
    \mathbb{E}\left[ \frac{1}{K}\sum_{k=1}^{K}||\nabla F(u_{k})||^{2} \right] \leq \mathcal{O}\left( \frac{1}{\sqrt{cm}} \right)
    +  \mathcal{O}\left( \frac{m\delta}{cK} \right)
    \end{gathered}
\end{equation*}

To get a convergence rate better than $\mathcal{O}\left( \frac{1}{\sqrt{cm}} \right) $, we need $\frac{1}{\sqrt{cm}} > \frac{m\delta}{cK} \implies K > \mathcal{O}\left(\delta m\sqrt{\frac{m}{c}}\right) $. We are unable to achieve a convergence rate of $\mathcal{O}\left( \frac{1}{\sqrt{cmK}} \right) $ under any situation, however. Thus, eventually the criteria becomes $K > \mathcal{O}\left(\max(\delta m\sqrt{\frac{m}{c}}, \tau)\right) $.

\section{PSASGD: Proof}

We can infer from Lemma 8 that $\delta = 0$ for this scenario. We can also rewrite Lemma 8 to prove this in another way. We know the following inequality holds:
\begin{equation*}
    \begin{gathered}
    ||A||_{F} \leq \sqrt{m}||A||_{op}
    \end{gathered}
\end{equation*}
Thus, we have, 
\begin{equation}
    ||J^{k}-J||_{F} \leq \sqrt{m+v}||J^{k}-J||_{op}
\end{equation}
Now, following an analysis similar to Wang and Joshi, we find that 
\begin{equation}
    ||J^{k}-J||_{op} = \varsigma^{k}
\end{equation}
But we know that for matrix $J$, $\varsigma = 0$. Also, we know that frobenius norm is greater than zero. Thus, we can conclude that $||J^{k}-J||_{F} = \delta$ where $\delta = 0$. In reality, this would imply simply substituting the value of $\delta = 0$ at all points. This would lead to the scenario that:

\begin{equation*}
    \begin{gathered}
      \mathbb{E}\left[ \frac{1}{K}\sum_{k=1}^{K}||\nabla F(u_{k})||^{2} \right] \leq 4\left[\frac{2[ F(u_{1}) - F_{inf}]}{\eta_{eff} K} + \frac{\eta_{eff} L\sigma^{2}}{cm}\right]
    \end{gathered}
\end{equation*}
The condition for learning rate then becomes,
\begin{equation*}
    \begin{gathered}
         \eta_{eff} L \leq 1
    \end{gathered}
\end{equation*}
If we set the learning rate as $\eta = \frac{1}{Lc}\sqrt{\frac{cm}{K}} \implies \eta_{eff} = \frac{1}{L}\sqrt{\frac{cm}{K}} $, then we get,
\begin{equation*}
    \begin{gathered}
      \mathbb{E}\left[ \frac{1}{K}\sum_{k=1}^{K}||\nabla F(u_{k})||^{2} \right] \leq 4\left[\frac{2L[ F(u_{1}) - F_{inf}]}{\sqrt{cmK}} + \frac{\sigma^{2}}{\sqrt{cmK}}\right]
    = \mathcal{O}\left( \frac{1}{\sqrt{cmK}} \right) 
    \end{gathered}
\end{equation*}
Also, we have the condition that
\begin{equation*}
    \begin{gathered}
         \sqrt{\frac{cm}{K}} \leq 1
         \implies \frac{cm}{K}\leq 1 
         \implies K \geq cm
    \end{gathered}
\end{equation*}

\section{IID Scenario with $\eta = \frac{1}{Lc}\sqrt{\frac{cm}{K}}$: Proof for D-PSGD with Dynamic Matrices (Replace $\tau = 1$)}

We have, assuming all models are initialized as zero,
\begin{equation*}
    \begin{gathered}
      \mathbb{E}\left[ \frac{1}{K}\sum_{k=1}^{K}||\nabla F(u_{k})||^{2} \right] \leq 4\left[\frac{2[ F(u_{1}) - F_{inf}]}{\eta K} + \frac{\eta L\sigma^{2}}{cm} 
    +   \eta^{2}\sigma^{2}L^{2}\delta (K-1)\right]\\
    \end{gathered}
\end{equation*}
The condition for learning rate then becomes,
\begin{equation*}
    \begin{gathered}
         \eta_{eff} L \leq 1
    \end{gathered}
\end{equation*}
In this scenario, following the analysis of \citet{Wang2021}, we have $\delta \leq 1$. If we have $\tau = 1 > \frac{1-\varsigma^{2}}{2\varsigma^{2}} \implies \varsigma > \frac{1}{\sqrt{3}}$ and set the learning rate as $\eta = \frac{1}{Lc}\sqrt{\frac{cm}{K}}$, then we get,
\begin{equation*}
    \begin{gathered}
      \mathbb{E}\left[ \frac{1}{K}\sum_{k=1}^{K}||\nabla F(u_{k})||^{2} \right] \leq 4\left[\frac{2L[ F(u_{1}) - F_{inf}]}{\sqrt{cmK}} + \frac{\sigma^{2}}{\sqrt{cmK}} 
    +  \frac{\sigma^{2}\delta (K-1) cm}{Kc^{2}}\right]\\
    \implies \mathbb{E}\left[ \frac{1}{K}\sum_{k=1}^{K}||\nabla F(u_{k})||^{2} \right] \leq 4\left[\frac{2L[ F(u_{1}) - F_{inf}]}{\sqrt{cmK}} + \frac{\sigma^{2}}{\sqrt{cmK}} 
    +  \frac{\sigma^{2}\left(\frac{1+\varsigma^{2}}{1-\varsigma^{2}} - 1\right)m}{Kc}\right]\\
    \implies \mathbb{E}\left[ \frac{1}{K}\sum_{k=1}^{K}||\nabla F(u_{k})||^{2} \right] \leq \mathcal{O}\left( \frac{1}{\sqrt{cmK}} \right) + \mathcal{O}\left( \frac{m}{Kc} \right)
    \end{gathered}
\end{equation*}
This matches the convergence bounds provided by \citet{Wang2021}. Furthermore, we get a convergence rate matching $\mathcal{O}\left( \frac{1}{\sqrt{cmK}} \right)$ when $\frac{1}{\sqrt{cmK}} > \frac{m}{Kc} \implies K > \frac{m^{3}}{c}$. In general, if $\tau > 1$ for a more general form of D-PSGD, this criterion becomes $K > \frac{m^{3}\tau^{2}}{c}$.\\

When applied to the learning rate, we get, 
\begin{equation*}
    \begin{gathered}
         \sqrt{\frac{cm}{K}} \leq 1\\
         \implies K  \geq  cm = \mathcal{O}(cm)
    \end{gathered}
\end{equation*}

The final upper bound we eventually get is $K > \frac{m^{3}\tau^{2}}{c}$, which matches the criterion of \citet{Wang2021}. 






\section{Experimental Details}

\subsection{Local-SGD algorithm}

The code for experiments have been provided in the supplementary materials. We describe the algorithm for centralized local-SGD that has been used for all experiments below:
\begin{algorithm}[H]
    \caption{Centralized Local-SGD}\label{alg:local-sgd-cent}
    \begin{algorithmic}
        \Require $K$ total epochs, $\tau$ communication period, $\eta$ learning rate, $S$ set of selected clients of size $cm$ where $m$ is total number of clients and $c$ is fraction of clients selected
        \Procedure{LocalSGD}{$K,\tau,\eta,S$}
            \State $x = 0$\Comment{Keeps track of how many epochs finished}
            \State $y = \max(\text{Batches of clients})$\Comment{Stores size of largest number of batches at any client}
            \State $batch\_count = 0$
            \For{$x<K$}
                \For{Client $i$ in selected set $S$}
                    \For{Batch $b_{i} \in \mathcal{D}_{i}$ of client $i$}\Comment{Ignore if $b_{i}$ doesn't exist}
                        \State Train on $b_{i}$
                    \EndFor
                \EndFor
                \State $batch\_count++$
                \State $y--$
                \If{$batch\_count\%\tau==0$}
                    \State All selected clients send their weights to central server
                    \State Aggregate the models using mixing matrix
                    \State Update global model
                    \State Communicate global model to all clients via ALLREDUCE
                \EndIf
                \If{$y==0$} \Comment{We have finished all batches $\implies$ Next epoch starts}
                    \State $x++$
                \EndIf
            \EndFor
        \EndProcedure
    \end{algorithmic}
\end{algorithm}

One thing to note here is the departure from regular federated learning literature. In regular federated learning literature, it is common to do the training as follows:
\begin{algorithmic}
    \For{rounds} 
        \For{epochs}
            \State \vdots
        \EndFor
    \EndFor
\end{algorithmic}
Here, we are able to aggregate the models after several \textbf{epochs}, which implies we run through the batches of each client several times before aggregation. One ``round'' encompasses several epochs. In our implementation of local-SGD, we deviate from this by allowing aggregation in between epochs as well - here, a ``round'' refers to one communication period, and hence one epoch may cover several rounds.\\

The algorithm for decentralized local-SGD is quite similar:
\begin{algorithm}[H]
    \caption{Decentralized Local-SGD}\label{alg:local-sgd-dec}
    \begin{algorithmic}
        \Require $K$ total epochs, $\tau$ communication period, $\eta$ learning rate, $S$ set of selected clients of size $cm$ where $m$ is total number of clients and $c$ is fraction of clients selected
        \Procedure{D-LocalSGD}{$K,\tau,\eta,S$}
            \State $x = 0$\Comment{Keeps track of how many epochs finished}
            \State $y = \max(\text{Batches of clients})$\Comment{Stores size of largest number of batches at any client}
            \State $batch\_count = 0$
            \For{$x<K$}
                \For{Client $i$ in selected set $S$}
                    \For{Batch $b_{i} \in \mathcal{D}_{i}$ of client $i$}\Comment{Ignore if $b_{i}$ doesn't exist}
                        \State Train on $b_{i}$
                    \EndFor
                \EndFor
                \State $batch\_count++$
                \State $y--$
                \If{$batch\_count\%\tau==0$}
                    \State All selected clients send their weights to each other via ALLREDUCE
                    \State Aggregate the models
                    \State Communicate aggregated model to all remaining clients
                \EndIf
                \If{$y==0$} \Comment{We have finished all batches $\implies$ Next epoch starts}
                    \State $x++$
                \EndIf
            \EndFor
        \EndProcedure
    \end{algorithmic}
\end{algorithm}

\subsection{Experimental Setup}

We train our models on TU104GL with Cuda Version 12.2 and python version 3.8.20. We have a single GPU setup, and hence have to resort to simulation of a federated learning setup. We use a VGG16 model from the torchvision package with default weights as ``VGG16\_Weights.DEFAULT'' except for the weight initialization experiments. We seed at every possible instance to ensure experiments are entirely replicable (Seed value of 42). Further, we train on a fraction of the entire dataset due to resource constraints. This fraction was fixed as 60\% of the entire dataset for all experiments. The communication period ($\tau$) refers to how many batches each client has to train on before aggregation. We now describe each experiment in detail in the following sections.

\subsubsection{Comparison with $\tau$}

For this experiment, we use the following hyperparameters:
\begin{itemize}
    \item Learning Rate: 0.0004/0.004
    \item Batch Size: 128
    \item Epochs: 20
    \item Clients: 8
    \item Clients Selected: 7
\end{itemize}
We vary the value of communication period as $\tau=24, 104, 154, 204$.

\subsubsection{Comparison with Client Selection}


For this experiment, we use the following hyperparameters:
\begin{itemize}
    \item Learning Rate: 0.0004/0.004
    \item Batch Size: 128
    \item Epochs: 20
    \item Clients: 8
    \item Communication Period: 60
\end{itemize}
We vary the number of selected clients as $cm=1,3,5,7$.

\subsubsection{Comparison with Model Initialization}

For this experiment, we use the following hyperparameters:
\begin{itemize}
    \item Learning Rate: 0.0004/0.004
    \item Batch Size: 128
    \item Epochs: 20
    \item Clients: 8
    \item Communication Period: 60
    \item Clients Selected: 5
\end{itemize}

\begin{figure}[t]
    \centering
    \begin{subfigure}[t]{.3\textwidth}
        \includegraphics[width=\textwidth]{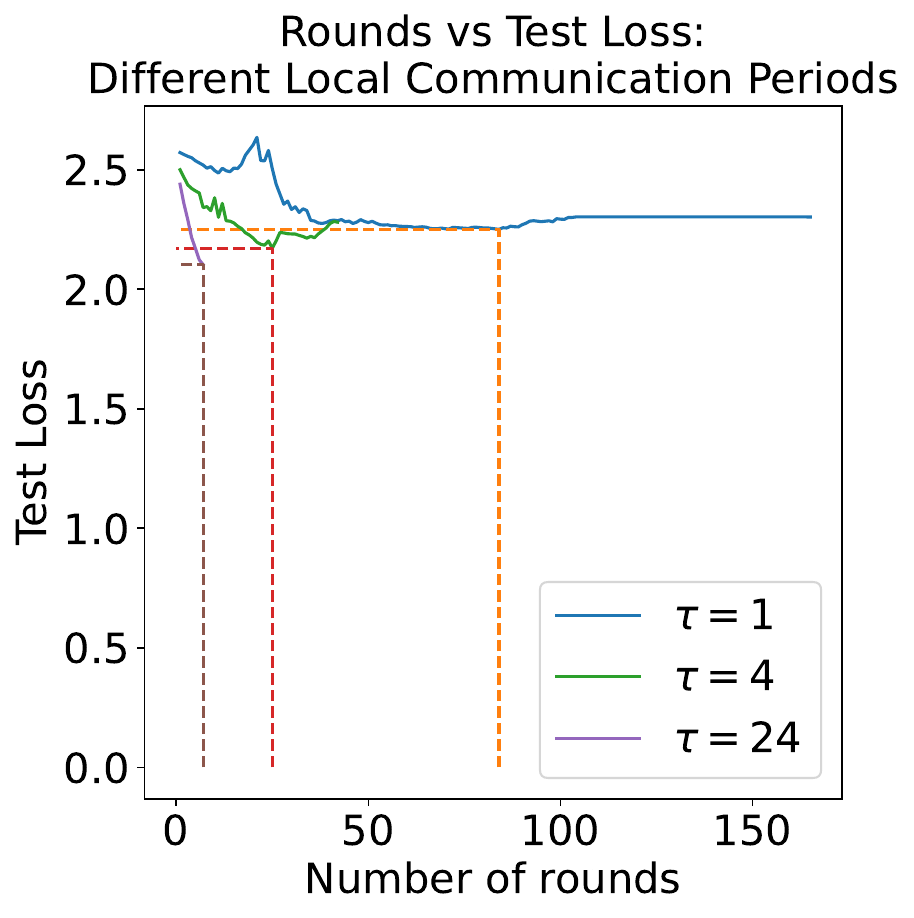}
        \caption{IID Comparison of Communication Period}\label{fig:tau_iid_fsync}
    \end{subfigure}\hfill
    \begin{subfigure}[t]{.3\textwidth}
        \includegraphics[width=\textwidth]{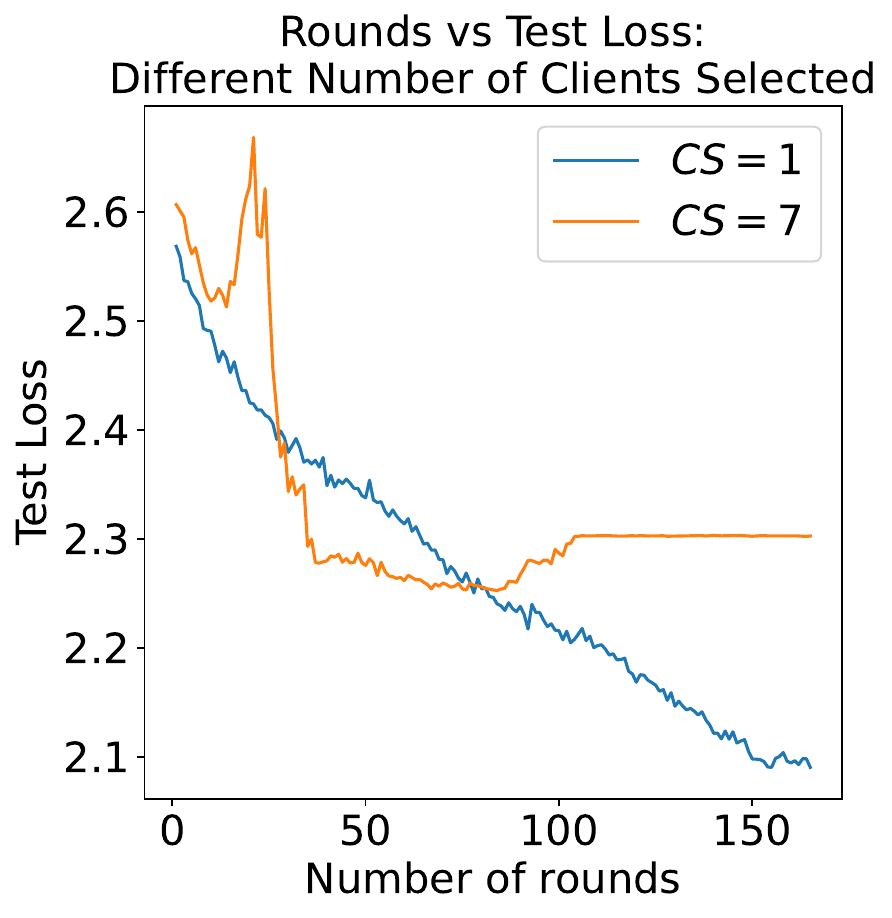}
        \caption{IID Comparison of Clients Selected}\label{fig:cs_iid_fsync}
    \end{subfigure}\hfill
    \begin{subfigure}[t]{.3\textwidth}
        \includegraphics[width=1.05\textwidth]{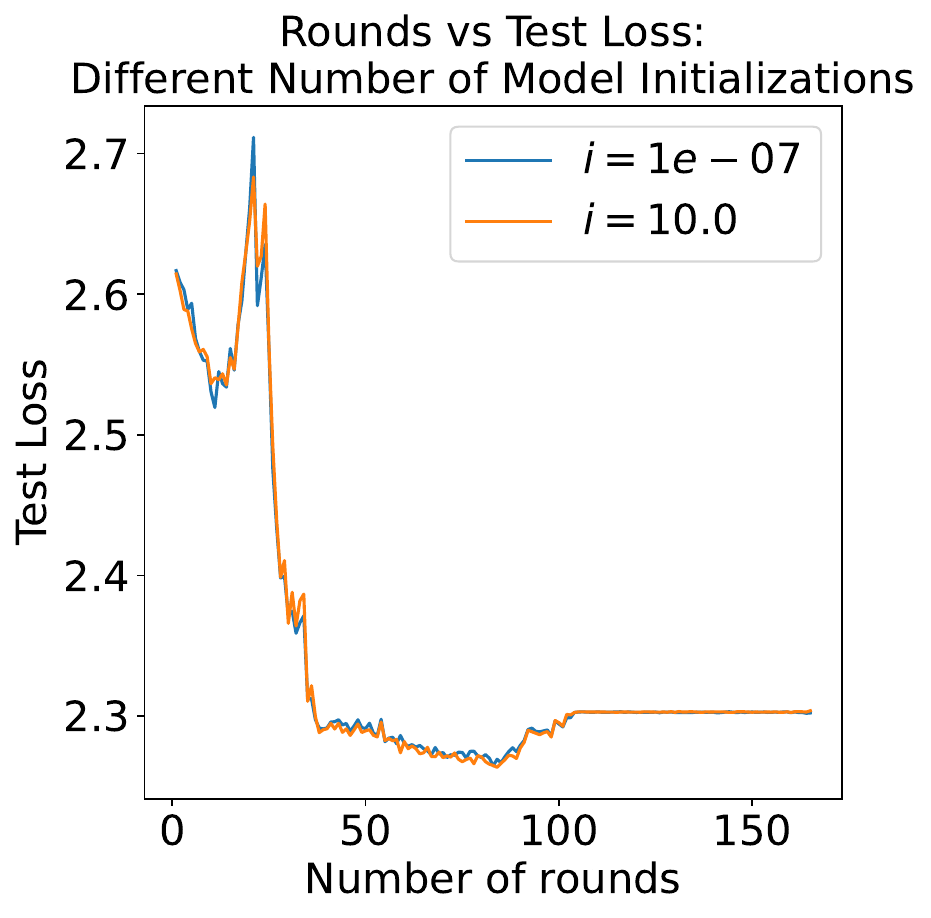}
        \caption{IID Comparison of Different Weight Initializations}\label{fig:init_iid_fsync}
    \end{subfigure}\hfill
    \vspace{7pt}
    \begin{subfigure}[t]{.3\textwidth}
        \includegraphics[width=\textwidth]{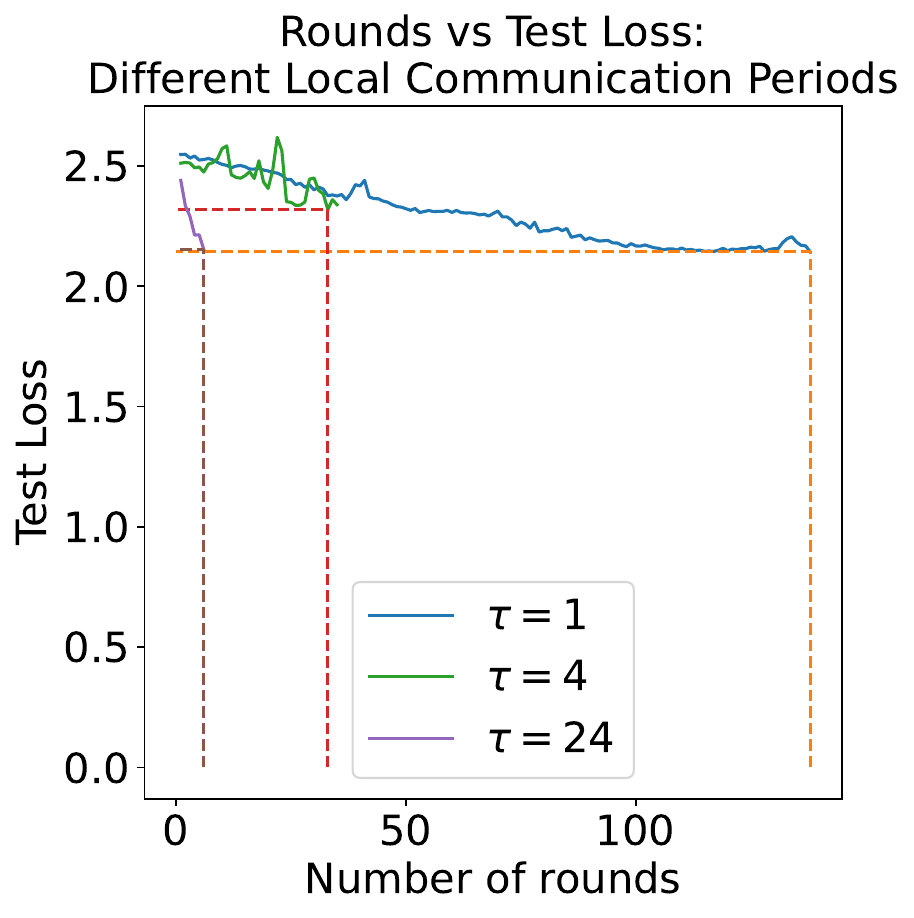}
        \caption{Non-IID Comparison of Communication Period}\label{fig:tau_niid_dir_fsync}
    \end{subfigure}\hfill
    \begin{subfigure}[t]{.3\textwidth}
        \includegraphics[width=\textwidth]{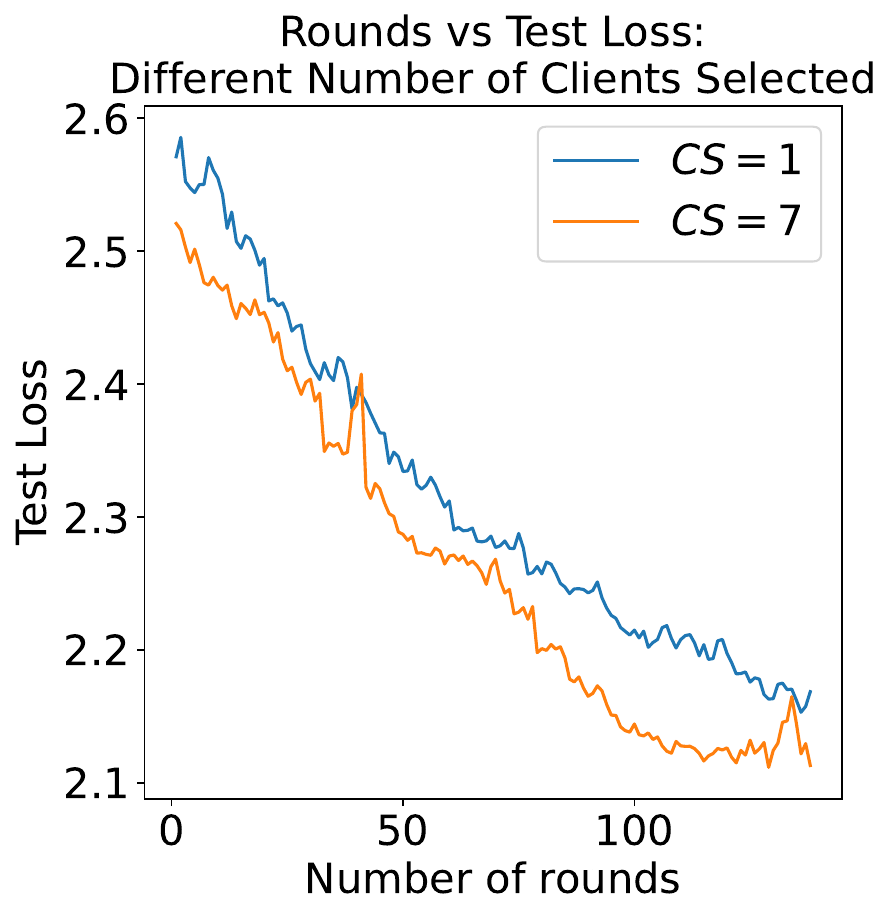}
        \caption{Non-IID Comparison of Clients Selected}\label{fig:cs_niid_dir_fsync}
    \end{subfigure}\hfill
    \begin{subfigure}[t]{.3\textwidth}
        \includegraphics[width=1.05\textwidth]{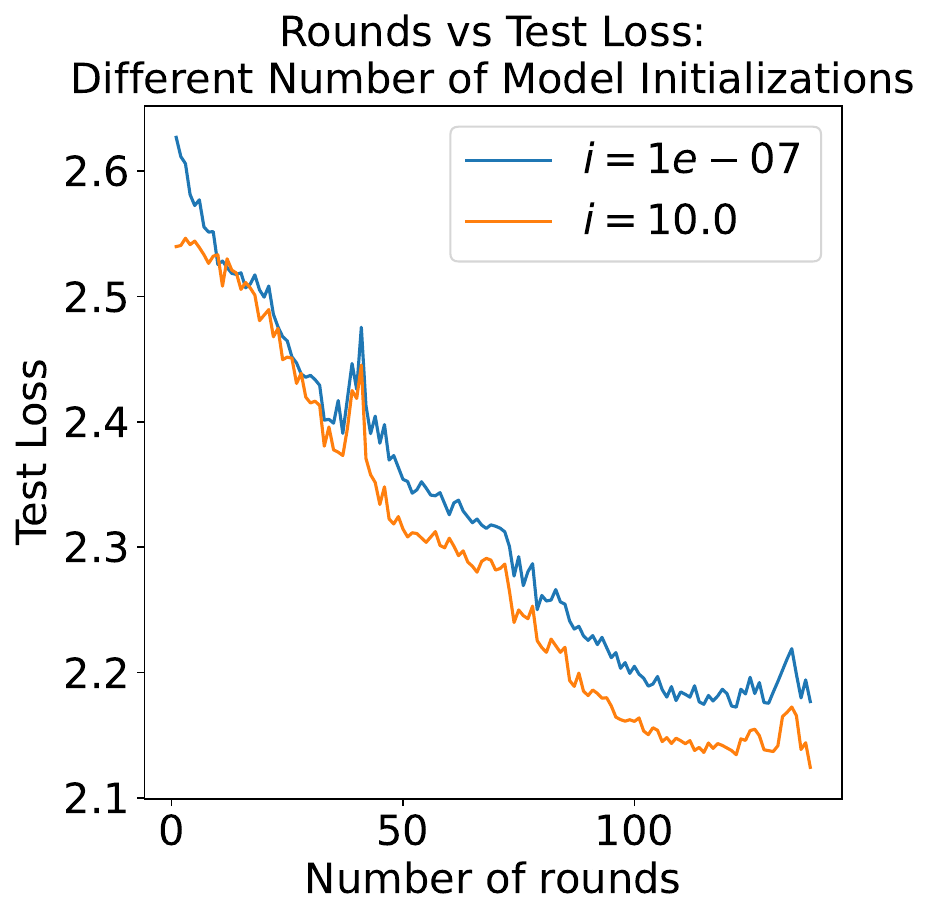}
        \caption{Non-IID Comparison of Different Weight Initializations}\label{fig:init_niid_dir_fsync}
    \end{subfigure}\hfill\vspace{12pt}
    \caption{Model convergence under different situations in both IID and non-IID scenario. The non-IID scenario is using a dirichlet distribution with $\alpha=0.6$. 7 clients were generally selected out of 8 (except Figure \ref{fig:cs_iid_fsync}, \ref{fig:cs_niid_dir_fsync}), and the communication period $\tau = 1$ (except Figure \ref{fig:tau_iid_fsync}, \ref{fig:tau_niid_dir_fsync}).} \label{fig:fsync}\vspace{20pt}
\end{figure}
We analyzed the default weights of VGG16 and found the following statistics:
\begin{itemize}
    \item Minimum parameter size = 8.573436231784637e-12
    \item Maximum parameter size = 1.4974256753921509
    \item Number of parameters below 1 = 134301501, Number of parameters above 1 = 13
    \item Number of parameters below 0.1 = 134281906, Number of parameters above 0.1 = 19608
    \item Number of parameters below 0.01 = 115544139, Number of parameters above 0.01 = 18757371
    \item Number of parameters below 0.001 = 19245245, Number of parameters above 0.001 = 115056268
    \item Number of parameters below 0.0001 = 1939822, Number of parameters above 0.0001 = 132361692
    \item Number of parameters below 1e-05 = 193665, param above 1e-05 = 134107849
    \item Total number of parameters = 134301514
\end{itemize}
We found that large deviation from this distribution resulted in poor convergence. To simulate the size of weights, we multiply each parameter of VGG16 with some initial value $i$. We check for convergence with value of $i = 0.7, 0.9, 1, 1.1, 1.3$.


\begin{figure}[t]
    \centering
    \begin{subfigure}[t]{.3\textwidth}
        \includegraphics[width=\textwidth]{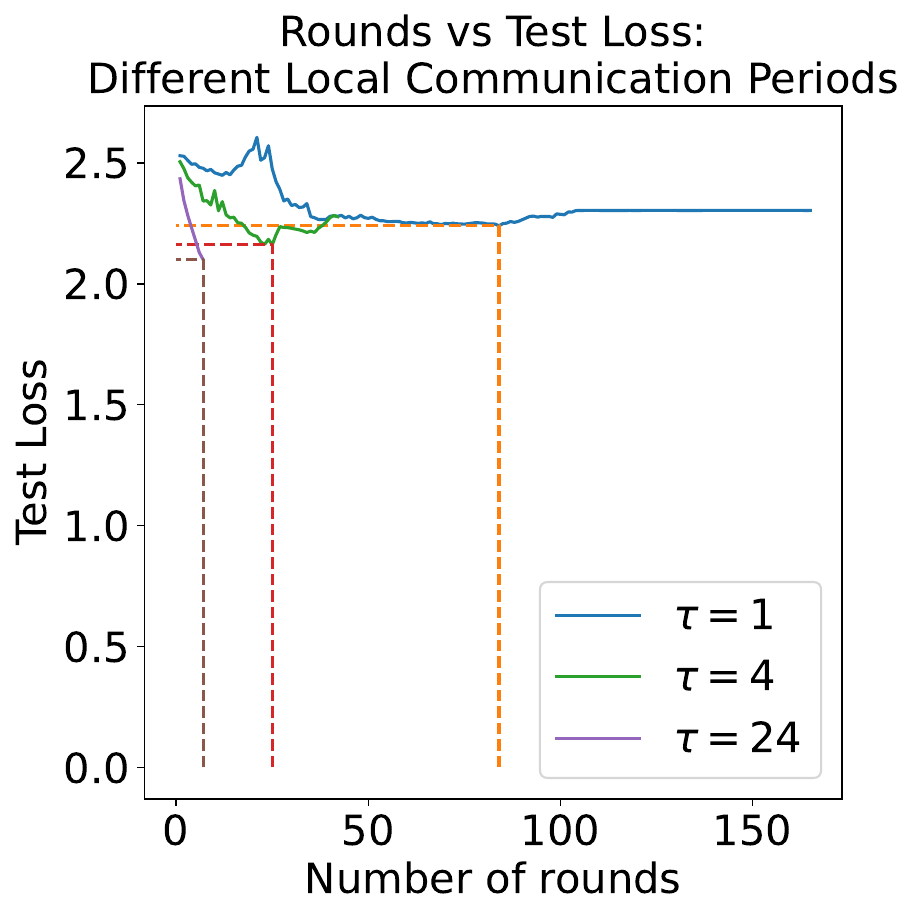}
        \caption{IID Comparison of Communication Period}\label{fig:tau_iid_dpsgd}
    \end{subfigure}\hfill
    \begin{subfigure}[t]{.3\textwidth}
        \includegraphics[width=\textwidth]{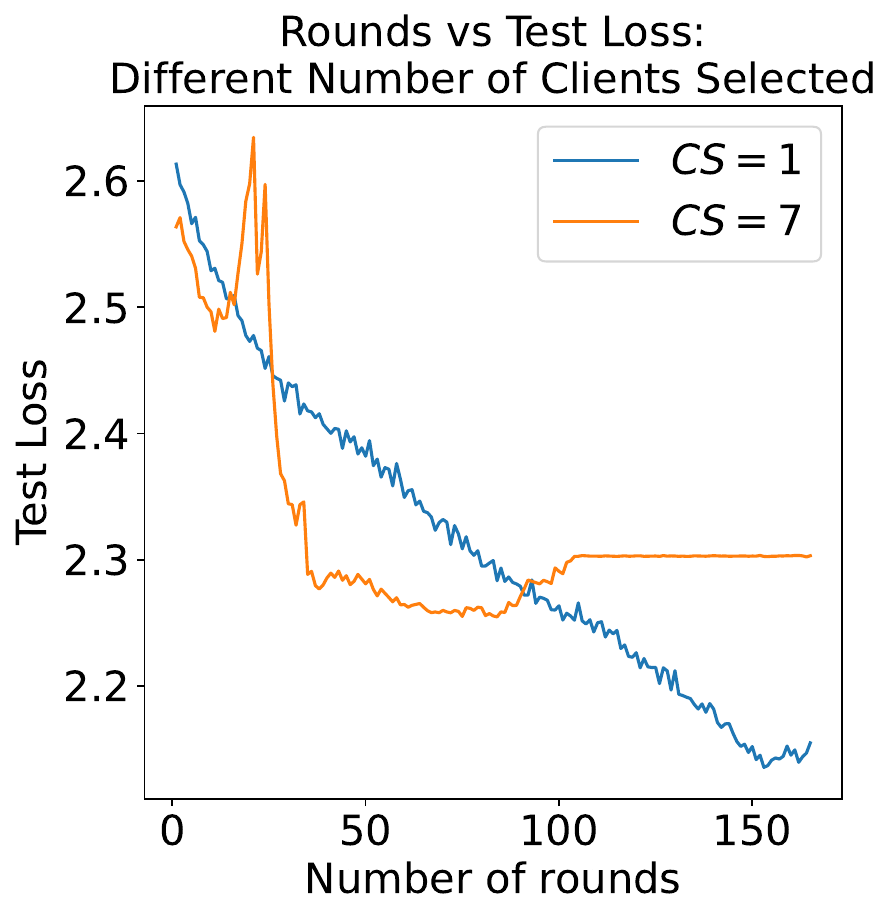}
        \caption{IID Comparison of Clients Selected}\label{fig:cs_iid_dpsgd}
    \end{subfigure}\hfill
    \begin{subfigure}[t]{.3\textwidth}
        \includegraphics[width=1.05\textwidth]{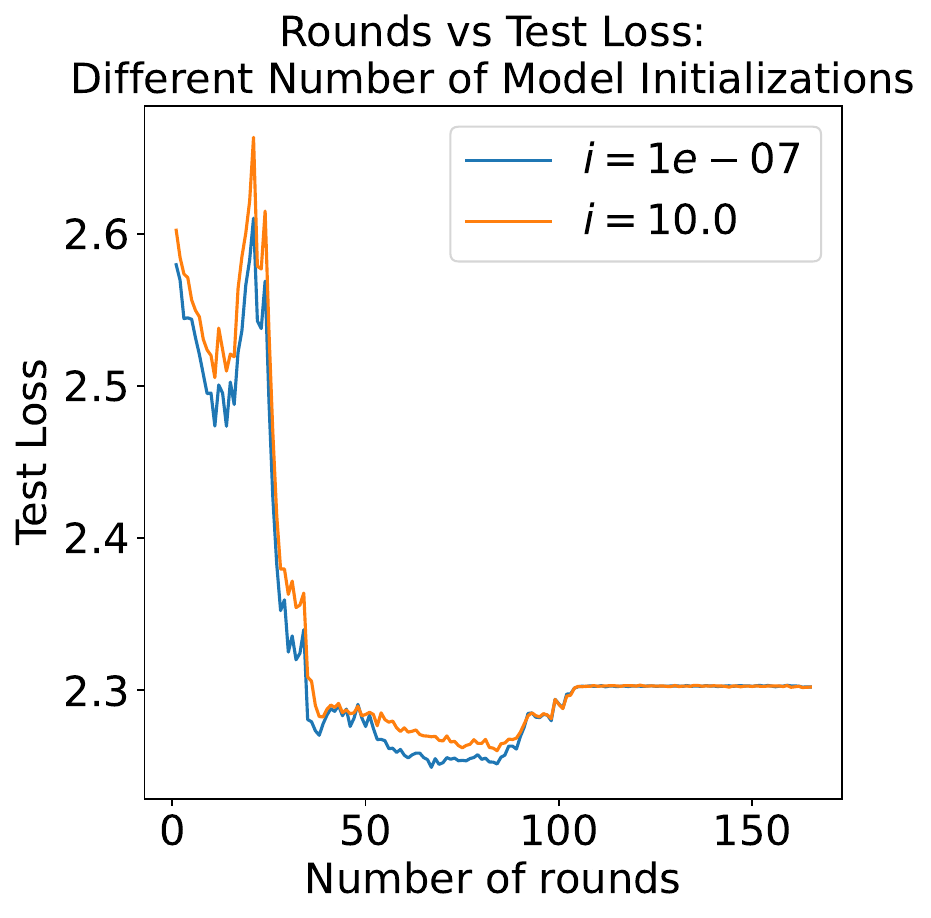}
        \caption{IID Comparison of Different Weight Initializations}\label{fig:init_iid_dpsgd}
    \end{subfigure}\hfill
    \vspace{7pt}
    \begin{subfigure}[t]{.3\textwidth}
        \includegraphics[width=\textwidth]{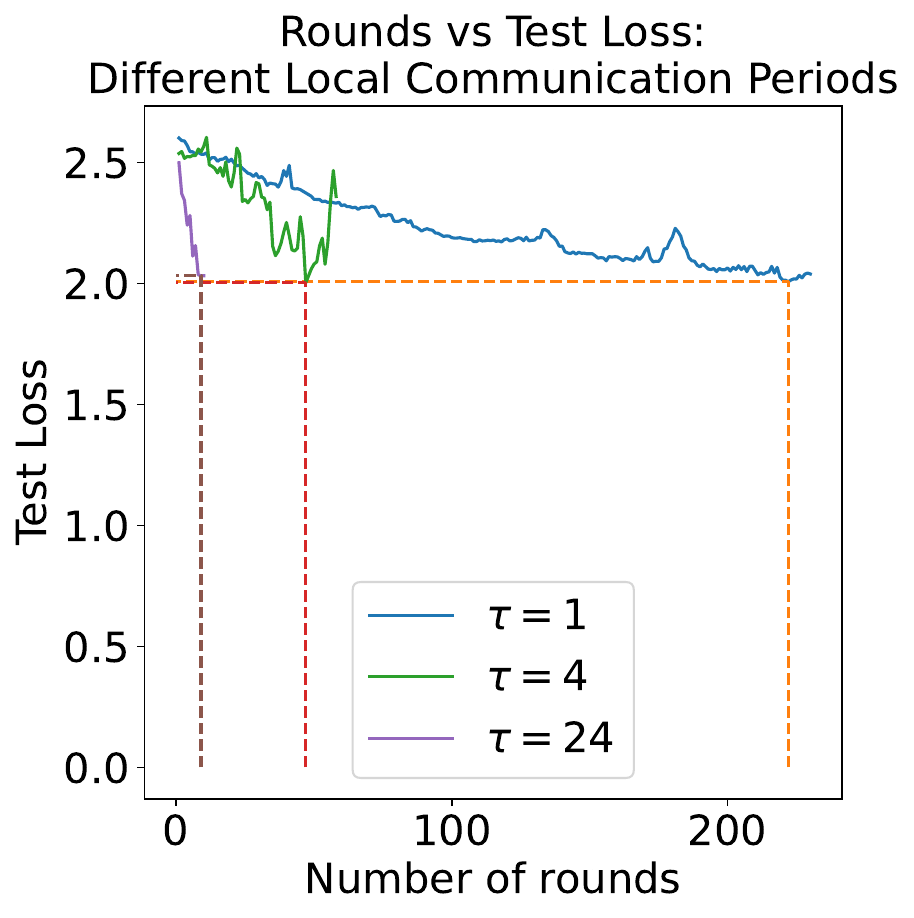}
        \caption{Non-IID Comparison of Communication Period}\label{fig:tau_niid_dir_dpsgd}
    \end{subfigure}\hfill
    \begin{subfigure}[t]{.3\textwidth}
        \includegraphics[width=\textwidth]{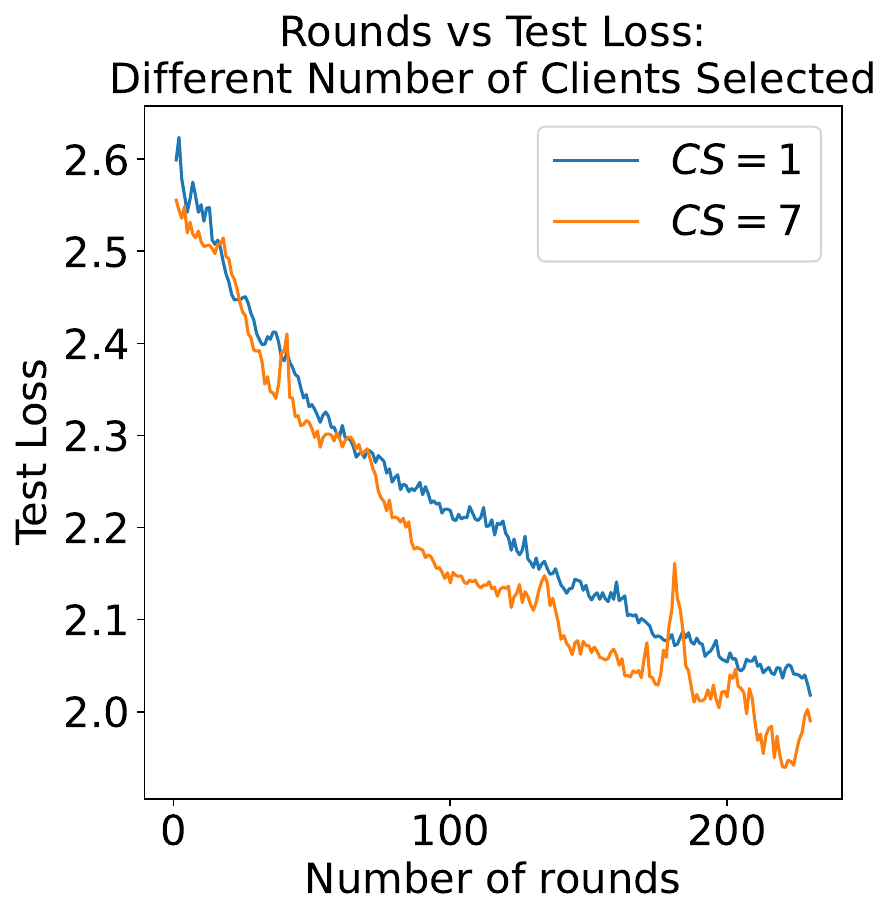}
        \caption{Non-IID Comparison of Clients Selected}\label{fig:cs_niid_dir_dpsgd}
    \end{subfigure}\hfill
    \begin{subfigure}[t]{.3\textwidth}
        \includegraphics[width=1.05\textwidth]{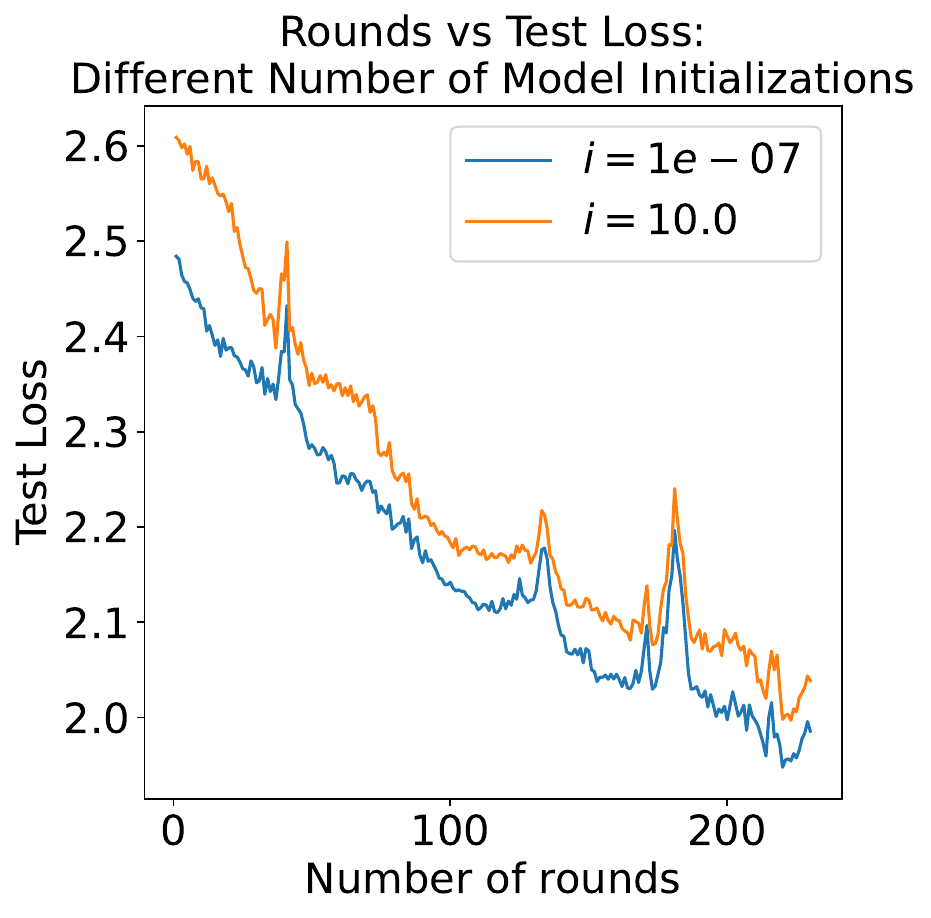}
        \caption{Non-IID Comparison of Different Weight Initializations}\label{fig:init_niid_dir_dpsgd}
    \end{subfigure}\hfill\vspace{12pt}
    \caption{Model convergence under different situations in both IID and non-IID scenario. The non-IID scenario is using a dirichlet distribution with $\alpha=0.6$. 7 clients were generally selected out of 8 (except Figure \ref{fig:cs_iid_dpsgd}, \ref{fig:cs_niid_dir_dpsgd}), and the communication period $\tau = 1$ (except Figure \ref{fig:tau_iid_dpsgd}, \ref{fig:tau_niid_dir_dpsgd}).} \label{fig:dpsgd}\vspace{20pt}
\end{figure}


\subsection{Loss Graphs for Fully Synchronous SGD} \label{exp:fsync}

We run studies on fully synchronous SGD and showcase the results in Figure \ref{fig:fsync}. We deviate from our previous experiments and run the above only for 3 epochs in the IID case and the Non-IID case. This was because we found that running it for more than 3 epochs led to a scenario where D-PSGD started overfitting for the data, leading to large loss value. We showcase for the IID and the Dirichlet non-IID case, with $\alpha = 0.6$.

\subsection{Graphs for D-PSGD} \label{exp:dpsgd}

We conduct experiments for D-PSGD and show our graphs here in Figure \ref{fig:dpsgd}. Unlike previous experiments, we run the ones above only for 3 epochs in the IID case and 5 epochs in the Non-IID case as running it for larger epoch values led to  overfitting for the data similarly to fully synchronous SGD. We showcase for the IID and the Dirichlet non-IID case, with $\alpha = 0.6$.



\end{document}